\documentclass[11pt]{article}

\usepackage{amsthm}
\usepackage{jmlr2e}
\usepackage{mlmath}
\usepackage[shortlabels]{enumitem}
\usepackage{tabularx}
\usepackage{multirow}
\usepackage{booktabs}
\usepackage[caption=false]{subfig}

\usepackage{xcolor}

\newtheorem{innercustomgeneric}{\customgenericname}
\providecommand{\customgenericname}{}
\newcommand{\newcustomtheorem}[2]{%
  \newenvironment{#1}[1]
  {%
  \renewcommand\customgenericname{#2}%
  \renewcommand\theinnercustomgeneric{##1}%
  \innercustomgeneric
  }
  {\endinnercustomgeneric}
}

\newcustomtheorem{customthm}{Theorem}
\newcustomtheorem{customlemma}{Lemma}
\newtheorem{assump}{Assumption}

\begin{document}

\title{Phase diagram for two-layer ReLU neural networks at infinite-width limit}

\author{%
    \name Tao\ Luo\footnotemark[1] \email luotao41@sjtu.edu.cn \\
    \name Zhi-Qin John Xu\footnotemark[1] \email xuzhiqin@sjtu.edu.cn \\
    \name Zheng Ma \email zhengma@sjtu.edu.cn \\
    \name Yaoyu Zhang\footnotemark[2] \email zhyy.sjtu@sjtu.edu.cn \\
    \addr School of Mathematical Sciences, Institute of Natural Sciences, MOE-LSC and Qing Yuan Research Institute, \\
    Shanghai Jiao Tong University, Shanghai, 200240, P.R. China}

\editor{}
\date{\today}

\footnotetext[1]{The first two authors contributed equally.}
\footnotetext[2]{Corresponding author.}

% \jmlrheading{1}{1993}{1-15}{8/93; Revised 10/93}{12/93}{14-115}{Jane Q. Public and A. U. Thor}
\maketitle


\begin{abstract}

    %In this work, we draw a phase diagram for the two-layer ReLU NN at the infinite width limit for a complete characterization of its dynamical regimes of distinctive features and their dependence on hyperparameters. Our phase diagram is obtained through a combination of experimental and theoretical approaches. Based on whether the first layer weight stays at the vicinity of its initialization, we identified the linear regime and its range in the phase diagram, in which the NN training dynamics is approximately linear.  In addition, 
    %NNs at the infinite width limit exhibits distinctive training behaviors depending on the choice of hyperparameters, e.g., NTK vs. mean-field. In analogy to statistical mechanics, these distinctive behaviors are like different states of a matter observed at different conditions, which can be captured by a unified phase diagram. 

    %Without a unified framework, it remains unclear how deep neural networks~(DNNs) with different hyperparameters present various behaviors during the gradient based training,  except for particular examples, e.g., some random feature~(RF) like models and the mean-field model, studied in literature.

    %This work initiates a step towards understanding NNs by studying characteristics of each regime in a phase diagram.

    How neural network behaves during the training over different choices of hyperparameters is an important question in the study of neural networks. In this work, inspired by the phase diagram in statistical mechanics, we draw the phase diagram for the two-layer ReLU neural network at the infinite-width limit for a complete characterization of its dynamical regimes and their dependence on hyperparameters related to initialization. Through both experimental and theoretical approaches, we identify three regimes in the phase diagram, i.e., \emph{linear} regime, \emph{critical} regime and \emph{condensed} regime, based on the relative change of input weights as the width approaches infinity, which tends to $0$, $O(1)$ and $+\infty$, respectively. In the linear regime, NN training dynamics is approximately linear similar to a random feature model with an exponential loss decay. In the condensed regime, we demonstrate through experiments that active neurons are condensed at several discrete orientations. The critical regime serves as the boundary between above two regimes, which exhibits an intermediate nonlinear behavior with the mean-field model as a typical example. Overall, our phase diagram for the two-layer ReLU NN serves as a map for the future studies and is a first step towards a more systematical investigation of the training behavior and the implicit regularization of NNs of different structures.

\end{abstract}

\section{Introduction}

It has been widely observed that, given training data, neural networks~(NNs) may exhibit distinctive dynamical behaviors during the training, depending on the choices of hyperparameters. As an example, we consider a two-layer NN with $m$  hidden neurons
\begin{equation}\label{eq: 2LNN}
    f^{\alpha}_{\vtheta}(\vx) = \frac{1}{\alpha}\sum_{k=1}^{m}a_k\sigma(\vw_k^{\T}\vx),
\end{equation}
where $\vx\in\sR^{d}$, $\alpha$ is the scaling factor, $\vtheta=\mathrm{vec}(\vtheta_a,\vtheta_{\vw})$ with $\vtheta_a=\mathrm{vec}(\{a_k\}_{k=1}^{m})$, $\vtheta_{\vw}=\mathrm{vec}(\{\vw_k\}_{k=1}^{m})$ is the set of parameters initialized by $a_k^0\sim N(0, \beta_1^2)$, $\vw_k^0\sim N(0, \beta_2^2 \mI_d)$. The bias term $b_k$ can be incorporated by expanding $\vx$ and $\vw_k$ to $(\vx^\T,1)^\T$ and $(\vw_k^\T,b_k)^\T$. At the infinite-width limit $m\to\infty$, given $\beta_{1},\beta_{2}\sim O(1)$, for $\alpha\sim\sqrt{m}$, the gradient flow of NN can be approximated by a linear dynamics of neural tangent kernel (NTK) ~\citep{jacot_neural_2018,arora2019exact,zhang_type_2019}, whereas for $\alpha\sim m$, gradient flow of NN exhibits highly nonlinear mean-field dynamics~\citep{mei_mean_2018,rotskoff_parameters_2018,chizat_global_2018,sirignano_mean_2020}. The current situation of NN study is similar to an early era of statistical mechanics, when we observe different states of a matter at several discrete conditions without the guidance of a unified phase diagram.

In this work, we present the first phase diagram for the two-layer neural networks with rectified linear units~(ReLU NN). To this end, two difficulties need to be overcome. The first difficulty is that one can not identify sharply distinctive regimes/states required for a phase diagram with finite neurons. This situation is similar to the analysis in statistical mechanics, e.g., Ising model, where phase transition can not happen with finite particles. Therefore, in analogy to the thermodynamic limit, we take the infinite-width limit $m\to\infty$ as our starting point and successfully identify three dynamical regimes of NNs, i.e., \emph{linear} regime, \emph{critical} regime, and \emph{condensed} regime. In the linear regime, $\vtheta_{\vw}$ almost does not change and NN training dynamics can be linearized around the initialization similar to an NTK or a random feature model. In the condensed regime, the relative change of $\vtheta_{\vw}$ tends to infinity and is condensed at several discrete directions in the feature space. In the critical regime, which serves as the boundary between above two regimes, relative change of $\vtheta_{\vw}$ is $O(1)$ with the mean-field model as an example. The second difficulty is the identification of phase diagram coordinates. For the vanilla gradient flow training dynamics of NN in Eq.~\eqref{eq: 2LNN}, there are three hyperparameters $\alpha$, $\beta_1$ and $\beta_2$, which in general are functions of $m$. However, through appropriate rescaling and normalization of the gradient flow dynamics, which accounts for the dynamical similarity up to a time scaling, we arrive at two independent coordinates
\begin{equation}
    \gamma=\lim\limits_{m\to\infty}-\frac{\log\beta_1\beta_2/\alpha}{\log m}, \quad \gamma'=\lim\limits_{m\to\infty}-\frac{\log\beta_1/\beta_2}{\log m}.
\end{equation}
The resulting phase diagram is shown in Fig.~\ref{fig:phase-diagram}. Examples studied in previous literature are also marked, for example, Ref.~\citet{e2020comparative} studied NNs with settings represented by the red dashed line.

This phase diagram is obtained through experimental and theoretical approaches. We first present an intuitive scaling analysis to provide a rationale for the boundary that separates the linear regime and the condensed regime. Then, we experimentally demonstrate the transition across this boundary in the phase diagram for an $1$-d dataset. Finally, we establish a rigorous theory for general datasets.

% Experimentally, for a specific $1$-d fitting problem, we demonstrate that relative change of $\vw$ tends to $0$, $O(1)$ and $\infty$ at linear, critical and condensed regime, respectively, as $m\to\infty$. Moreover, we show that, in the condensed regime, the distribution of $\vtheta$ in feature space after training tends to be more and more condensed as $m\to\infty$ opposite to the linear and critical regimes. Theoretically, we prove that, in the linear regime of the phase diagram, the relative change of $\vw$ tends to $0$ during the training as $m\to\infty$.

Our work is a first step towards a systematical effort in drawing the phase diagrams for NNs of different structures. With the guidance of these phase diagrams, detailed experimental and theoretical works can be done to further characterize the dynamical behavior and the corresponding implicit regularization effect at each of the identified regime.

\begin{figure}
    \centering
    \includegraphics[width=\textwidth]{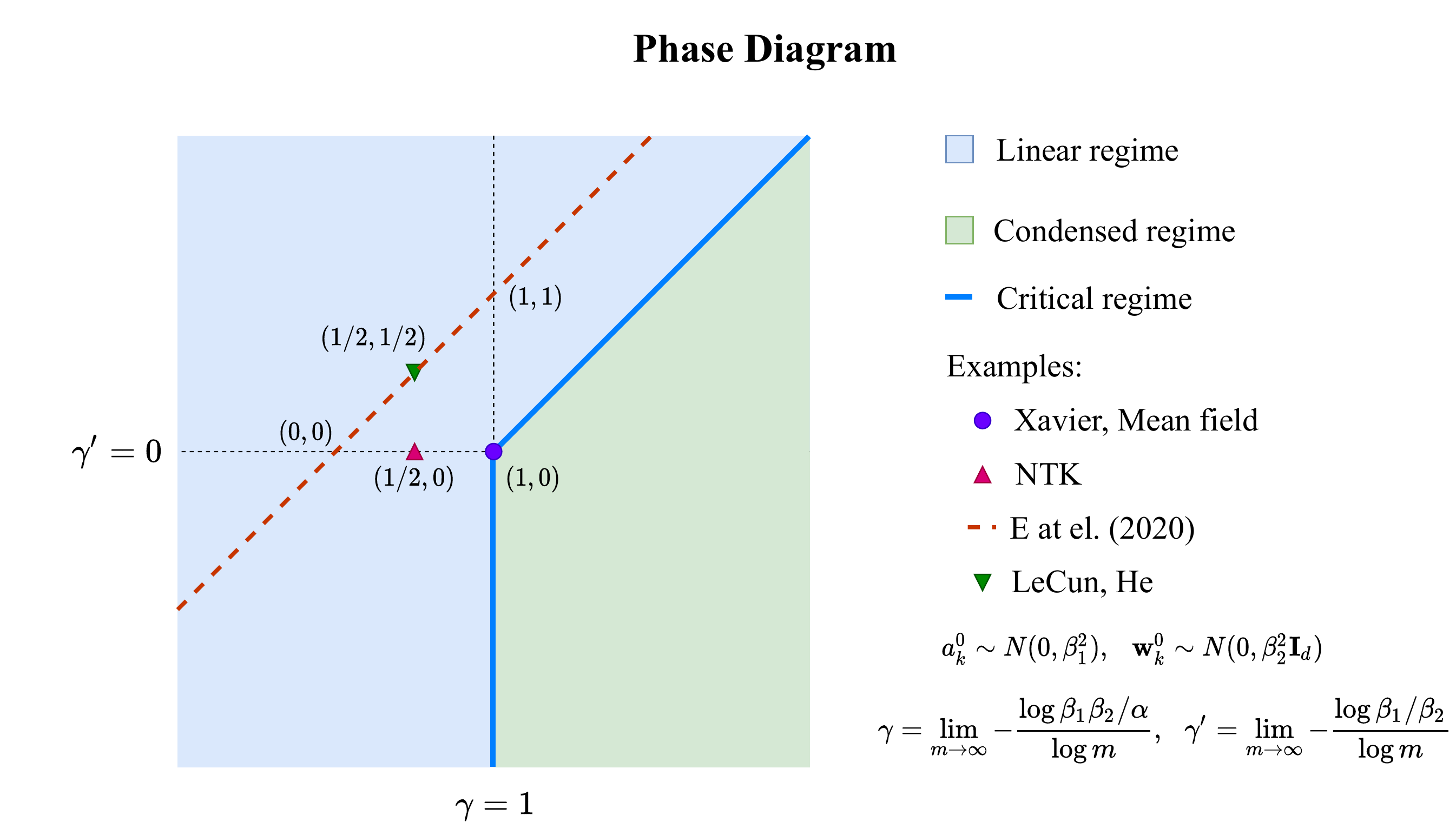}
    \caption{Phase diagram of two-layer ReLU NNs at infinite-width limit. The marked examples are studied in existing literature (see Table \ref{tab..InitializationMethods} for details.) }
    \label{fig:phase-diagram}
\end{figure}

\section{Related works}
The study of regimes in the literature usually revolves around the choice of scaling factor $\alpha$ in specific power-law relations to the width $m$. For example, the NTK scaling $\alpha\sim\sqrt{m}$ ~\citep{jacot_neural_2018,arora2019exact,zhang_type_2019} and the mean-field scaling $\alpha\sim m$~\citep{mei_mean_2018,rotskoff_parameters_2018,chizat_global_2018,sirignano_mean_2020} has been studied extensively. In~\citet{chizat2019lazy}, the authors identify the lazy training behavior for $\lim_{m\to \infty}m/\alpha =\infty$, by which NN parameters stay close to initialization during the training. In~\citet{williams_gradient_2019}, for two-layer ReLU network, lazy and active regimes and their corresponding regularization effect are studied for $1$-d problems. Their analysis uses different quantities for regime separation, which cannot serve as coordinates for a phase diagram. All above works do not account for the effect of specific power-law scaling of initialization over different layers used in practice.

%Different research groups made attempts and proposed different criterions. In~\cite{chizat2019lazy}, they classify regimes into lazy (or kernel) and active regimes by the distance between the initial parameters the final ones by the choice of scaling. In \cite{woodworth2020kernel} they obtain the kernel regime by increasing the scale of parameter initialization; they also empirically show that two-layer ReLU NNs use linear spline to interpolate $1$-d functions when the initialization tends to zero, which is called rich regime.

%In \cite{jacot2019freeze} they  classify regimes by the NTK for fully-connected DNNs at infinite-width limit as the depth grows:  a freeze regime where the (scaled) NTK converges to a constant , and a chaos regime where it converges to a Kronecker delta. 

%Natty Sebrabo, kernel and rich regime

In~\citet{e2020comparative}, for two-layer NNs with $\alpha=1$, $\beta_2\sim O(1)$, the authors study the effect of $\beta$ ($\sim \beta_1$) in relation to $m$. Specifically, they prove that NN training dynamics can be linearized for $\beta=o(m^{-1/6})$ as $m\to\infty$, which constitutes a line in Fig.~\ref{fig:phase-diagram}. In~\citet{ma2020quenching}, they further study such cases in the under-parameterized and mildly over-parameterized settings and experimentally identified the quenching-activation behavior for finite $m$, which phenomenologically is closely related to the condensed regime we identified at $m\to\infty$.

Another work related to the condensed regime is~\citet{maennel2018gradient}. The authors study the two-layer ReLU NNs and prove that, as the initialization of parameters goes to zero, a quantization effect emerges, that is, the weight vectors tend to concentrate at a small number of orientations determined by the input data at an early stage of training. However, the limit of $m\to \infty$ is not considered in their work.

%The perspective of studying NNs from macroscopic level of the NN output is also closely related to studying different parameter regimes. In~\cite{xu_training_2019,xu_frequency_2019,rahaman_spectral_2019,luo_theory_2019,e_machine_2019} they demonstrate a general phenomenon of frequency principle---NNs often learn low frequencies first. Further more,  in~\cite{zhang2019explicitizing} they explicitly show the convergence of different frequency and with different parameter initializations, the NN interpolates training data in a way varying from a linear spline  to a cubic spline interpolation. Therefore, studying DNNs in the macroscopic level is closely related to studying different parameter regimes.

\section{Rescaling and the normalized model}\label{sec..Rescaling}
Identification of the coordinates is important for drawing the phase diagram. Unlike in some thermodynamic systems where temperature and pressure are natural choices, for NNs, it is not obvious which quantities of hyperparameters are keys to the regime separation. However, there are some guiding principles for finding the coordinates of a phase diagram at $m\to\infty$:
\begin{enumerate}[label=(\roman*)]
    \item They should be effectively independent.
    \item Given a specific coordinate in the phase diagram, the learning dynamics of all the corresponding NNs statistically should be similar up to a time scaling.
    \item They should well differentiate dynamical differences except for the time scaling.
\end{enumerate}

Guided by above principles, in this section, we perform the following rescaling procedure for a fair comparison between different choices of hyperparameters and obtain a normalized model with two independent quantities irrespective of the time scaling of the gradient flow dynamics.
%which we name it as, in contrast to the original model~\eqref{eq: 2LNN}, the \emph{normalized} model. 
We start with the original model~\eqref{eq: 2LNN}
\begin{equation}
    f^{\alpha}_{\vtheta}(\vx) = \frac{1}{\alpha}\sum_{k=1}^{m}a_k\sigma(\vw_k^{\T}\vx),
\end{equation}
defined on a given sample set $S=\{(\vx_i,y_i)\}_{i=1}^n$ where $\vx_i\in\sR^d$, $i\in[n]$, network width $m$ and a scaling parameter $1/\alpha$ and $\sigma=\ReLU$. The parameters are initialized by
\begin{equation}
    a_k^0\sim N(0, \beta_1^2), \quad \vw_k^0\sim N(0, \beta_2^2 \mI_d),
\end{equation}
where $a_k$ and $\vw_k$ are separated into to different scales $\beta_1$ and $\beta_2$. The empirical risk is
\begin{equation}
    \RS(\vtheta)=\frac{1}{2n}\sum_{i=1}^n {(f^{\alpha}_{\vtheta}(\vx_i)-y_i)}^2.
\end{equation}
Then the training dynamics based on gradient descent~(GD) at the continuous limit obeys the following gradient flow of $\vtheta$,
\begin{equation}
    \frac{\D \vtheta}{\D t}=-\nabla_{\vtheta}\RS(\vtheta).
\end{equation}
More precisely, $\vtheta=\mathrm{vec}(\{\vq_k\}_{k=1}^m)$ with $\vq_k=(a_k,\vw_k^{\T})^\T$, $k\in[m]$ solves
\begin{align*}
    \frac{\D a_k}{\D t}
     & = -\frac{1}{n}\sum_{i=1}^n\frac{1}{\alpha}\sigma(\vw_k^{\T}\vx_i) \left(\frac{1}{\alpha}\sum_{k=1}^{m}a_k\sigma(\vw_k^{\T}\vx_i)-y_i\right)            \\
    \frac{\D \vw_k}{\D t}
     & = -\frac{1}{n}\sum_{i=1}^n\frac{1}{\alpha} a_k\sigma'(\vw_k^{\T}\vx_i)\vx_i \left(\frac{1}{\alpha}\sum_{k=1}^{m}a_k\sigma(\vw_k^{\T}\vx_i)-y_i\right).
\end{align*}
Let
\begin{equation}
    \bar{a}_k=\beta_1^{-1}a_k, \quad \bar{\vw}_k=\beta_2^{-1}\vw_k,\quad \bar{t}=\frac{1}{\beta_1\beta_2}t,
\end{equation}
then
\begin{align*}
    \frac{\D \bar{a}_k}{\D \bar{t}}
     & = -\frac{\beta_2}{\beta_1}\frac{1}{n}\sum_{i=1}^n \frac{\beta_1\beta_2}{\alpha}\sigma(\bar{\vw}_k^\T\vx_i) \left(\frac{\beta_1\beta_2}{\alpha}\sum_{k=1}^{m}\bar{a}_k\sigma(\bar{\vw}_k^\T\vx_i)-y_i\right),               \\
    \frac{\D \bar{\vw}_k}{\D \bar{t}}
     & = -\frac{\beta_1}{\beta_2}\frac{1}{n}\sum_{i=1}^n\frac{\beta_1\beta_2}{\alpha}\bar{a}_k\sigma'(\bar{\vw}_j^\T\vx_i)\vx_i \left(\frac{\beta_1\beta_2}{\alpha}\sum_{k=1}^{m}\bar{a}_k\sigma(\bar{\vw}_k^\T\vx_i)-y_i\right).
\end{align*}
We introduce two scaling parameters
\begin{equation}
    \kappa := \frac{\beta_1\beta_2}{\alpha}, \quad \kappa' :=\frac{\beta_1}{\beta_2},
\end{equation}
where $\kappa$ and $\kappa'$ are called the energetic scaling parameter and the dynamical scaling parameter, respectively. Then the above dynamics can be written as
\begin{align*}
    \frac{\D \bar{a}_k}{\D \bar{t}}
     & = -\frac{1}{\kappa'}\frac{1}{n}\sum_{i=1}^n \kappa\sigma(\bar{\vw}_k^\T\vx_i) \left(\kappa\sum_{k=1}^{m}\bar{a}_k\sigma(\bar{\vw}_k^\T\vx_i)-y_i\right),      \\
    \frac{\D \bar{\vw}_k}{\D \bar{t}}
     & = -\kappa'\frac{1}{n}\sum_{i=1}^n\kappa \bar{a}_k\sigma'(\bar{\vw}_k^\T\vx_i)\vx_i \left(\kappa\sum_{k=1}^{m}\bar{a}_k\sigma(\bar{\vw}_k^\T\vx_i)-y_i\right).
\end{align*}
The above recaled dynamics can be treated as a weighted gradient flow of NN scaled by $\kappa$ equipped with the empirical risk
\begin{align}\label{eq:normalized-model}
    f^{\kappa}_{\vtheta}(\vx)
     & = \kappa\sum_{k=1}^{m}\bar{a}_k\sigma(\bar{\vw}_k^\T\vx),       \\
    R_{S, \kappa}(\vtheta)
     & =\frac{1}{2n}\sum_{i=1}^n{(f^{\kappa}_{\vtheta}(\vx_i)-y_i)}^2,
\end{align}
with the following initialization
\begin{equation}
    \bar{a}_j^0\sim N(0,1), \quad \bar{\vw}_j^0\sim N(0,\mI_d),
\end{equation}
where we can see they are of standard normal distributions. The weighted GD dynamics then can be written simply as
\begin{equation} \label{eq:qdynamics}
    \frac{\D\bar{\vq}_j}{\D \bar{t}} = -\mM_{\kappa'}\nabla_{\vq_k}R_{S,\kappa}(\bar{\vtheta}),
\end{equation}
where the mobility matrix
\begin{equation}
    \mM_{\kappa'} =
    \begin{pmatrix}
        1/\kappa' &              \\
                  & \kappa'\mI_d
    \end{pmatrix}.
\end{equation}
In the following discussion throughout this paper, we will refer to this rescaled model~\eqref{eq:normalized-model} as \emph{normalized} model and drop superscript $\kappa$ and all the ``bar''s of $a_k$, $\vw_k$, $t$ for simplicity.
Note that $\kappa$ and $\kappa'$ do not follow principle (ii) and (iii) above at infinite-width limit. They are in general functions of $m$, which attains $0$, $O(1)$, $+\infty$ at $m\to\infty$. For example, $\kappa=0$ and $\kappa'=1$ for both the NTK and mean-field model, however, they are known to have distinctive training behaviors. To account for such dynamical difference under different widely considered power-law scalings of $\alpha$, $\beta_1$ and $\beta_2$ shown in Table. \ref{tab..InitializationMethods}, we arrive at
\begin{equation}
    \gamma=\lim_{m\to\infty}-\frac{\log \kappa}{\log m}, \quad \gamma'=\lim_{m\to\infty}-\frac{\log\kappa'}{\log m},
\end{equation}
which meets all above principles as demonstrated later by theory and experiments.

\begin{remark}
    We remark that the above rescaling technique can be viewed in analogy to the nondimensionalization in physics, which is the partial or full removal of physical dimensions from an equation involving physical quantities by a suitable substitution of variables. In more general point of view, nondimensionalization can also recover characteristic properties of a system, which in our case recovers the different behaviors of training dynamics for different regimes.

    More specifically, we can view, in the original model~\eqref{eq: 2LNN}, $\vq_k=(a_k,\vw_k^{\T})^\T$, $k\in[m]$ as the generalized coordinates which have the unit of ``length'' denoted as $[\mathrm{L}]$.\@ Then in the two-layer NN~\eqref{eq: 2LNN}, $\alpha$ should have the unit of ``volume'' as a normalization factor depending on $m$ to avoid blowing up of the model. Particularly, if $\sigma$ is $\ReLU$ then we can think $\alpha$'s unit is $[\mathrm{L}]^2$ (unit of area on a plane).

    Finally, following above analysis, $\kappa=\frac{\beta_1\beta_2}{\alpha}$ and $\kappa'=\frac{\beta_1}{\beta_2}$ are two \emph{nondimensional} parameters (without unit) so as for $\gamma$ and $\gamma'$, which are suitable to serve as the coordinations of our phase diagram.
\end{remark}
\begin{remark}
    Here we list some commonly-used initialization methods and/or related works with their scaling parameters as shown in Table~\ref{tab..InitializationMethods}.
\end{remark}

\begin{table}[ht]
    \centering
    \resizebox{\textwidth}{!}{
        \begin{tabularx}{1.2\textwidth}{cccccccc}
            \toprule
            Name                                          & \multirow{2}{*}{$\alpha$}   & \multirow{2}{*}{$\beta_1$}              & \multirow{2}{*}{$\beta_2$}              & $\kappa$                                            & $\kappa'$                                     & $\gamma$                                                                                 & $\gamma'$                                                                                \\
            \scriptsize{({related works})}                &                             &                                         &                                         & ($\scriptscriptstyle\frac{\beta_1\beta_2}{\alpha}$) & ($\scriptscriptstyle\frac{\beta_1}{\beta_2}$) & ($\scriptscriptstyle\lim\limits_{m\to\infty}\frac{\log1/\kappa}{\log m}$)                & ($\scriptscriptstyle\lim\limits_{m\to\infty}\frac{\log 1/\kappa'}{\log m}$)              \\
            \midrule
            LeCun                                         & \multirow{2}{*}{$1$ }       & \multirow{2}{*}{$\sqrt{\frac{1}{m}}$}   & \multirow{2}{*}{$\sqrt{\frac{1}{d}}$}   & \multirow{2}{*}{$\sqrt{\frac{1}{md}}$}              & \multirow{2}{*}{$\sqrt{\frac{d}{m}}$}         & \multirow{2}{*}{$\frac{1}{2}$}                                                           & \multirow{2}{*}{$\frac{1}{2}$}                                                           \\
            \scriptsize{\citep{lecun2012efficient}}       &                             &                                         &                                         &                                                     &                                               &                                                                                          &                                                                                          \\
            He                                            & \multirow{2}{*}{$1$}        & \multirow{2}{*}{$\sqrt{\frac{2}{m}}$}   & \multirow{2}{*}{$\sqrt{\frac{2}{d}}$}   & \multirow{2}{*}{$\sqrt{\frac{4}{md}}$}              & \multirow{2}{*}{$\sqrt{\frac{d}{m}}$}         & \multirow{2}{*}{$\frac{1}{2}$}                                                           & \multirow{2}{*}{$\frac{1}{2}$}                                                           \\
            \scriptsize{\citep{he2015delving}}            &                             &                                         &                                         &                                                     &                                               &                                                                                          &                                                                                          \\
            Xavier                                        & \multirow{2}{*}{$1$}        & \multirow{2}{*}{$\sqrt{\frac{2}{m+1}}$} & \multirow{2}{*}{$\sqrt{\frac{2}{m+d}}$} & \multirow{2}{*}{$\sqrt{\frac{4}{(m+1)(m+d)}}$}      & \multirow{2}{*}{$\sqrt{\frac{m+d}{m+1}}$}     & \multirow{2}{*}{$1$}                                                                     & \multirow{2}{*}{$0$}                                                                     \\
            \scriptsize{\citep{glorot2010understanding}}  &                             &                                         &                                         &                                                     &                                               &                                                                                          &                                                                                          \\
            NTK                                           & \multirow{2}{*}{$\sqrt{m}$} & \multirow{2}{*}{$1$}                    & \multirow{2}{*}{$1$}                    & \multirow{2}{*}{$\sqrt{\frac{1}{m}}$}               & \multirow{2}{*}{$1$}                          & \multirow{2}{*}{$\frac{1}{2}$}                                                           & \multirow{2}{*}{$0$}                                                                     \\
            \scriptsize{\citep{jacot_neural_2018}}        &                             &                                         &                                         &                                                     &                                               &                                                                                          &                                                                                          \\
            % \citep{sirignano_mean_2020},\citep{rotskoff_parameters_2018}
            Mean-field                                    & \multirow{3}{*}{$m$}        & \multirow{3}{*}{$1$}                    & \multirow{3}{*}{$1$}                    & \multirow{3}{*}{$\frac{1}{m}$}                      & \multirow{3}{*}{$1$}                          & \multirow{3}{*}{$1$}                                                                     & \multirow{3}{*}{$0$}                                                                     \\
            \scriptsize{\citep{mei_mean_2018}}            &                             &                                         &                                         &                                                     &                                               &                                                                                          &                                                                                          \\
            \scriptsize{\citep{sirignano_mean_2020}}      &                             &                                         &                                         &                                                     &                                               &                                                                                          &                                                                                          \\
            \scriptsize{\citep{rotskoff_parameters_2018}} &                             &                                         &                                         &                                                     &                                               &                                                                                          &                                                                                          \\
            E et al.                                      & \multirow{2}{*}{$1$}        & \multirow{2}{*}{$\beta$}                & \multirow{2}{*}{$1$}                    & \multirow{2}{*}{$\beta$}                            & \multirow{2}{*}{$\beta$}                      & \multirow{2}{*}{$\scriptscriptstyle\lim\limits_{m\to\infty}\frac{\log 1/\beta}{\log m}$} & \multirow{2}{*}{$\scriptscriptstyle\lim\limits_{m\to\infty}\frac{\log 1/\beta}{\log m}$} \\
            \scriptsize{\citep{e2020comparative}}         &                             &                                         &                                         &                                                     &                                               &                                                                                          &                                                                                          \\
            \bottomrule
        \end{tabularx}
    }
    \caption{Initialization methods with their scaling parameters}
    \label{tab..InitializationMethods}
\end{table}

\subsection{Typical cases over the phase diagram}

With $\gamma$ and $\gamma'$ as coordinates, in this subsection, we illustrate through experiments the behavior of a diversity of typical cases over the phase diagram using a simple $1$-d problem of $4$ training points, which allows easy visualization.

The first row in Fig.~\ref{fig:targetfunc} shows typical learning results over different $\gamma$'s, from a relatively jagged interpolation (NTK scaling) to a smooth cubic-spline-like interpolation (mean-field scaling) and further to a linear spline interpolation. To probe into details of their parameter space representation, we notice for the ReLU activation that the parameter pair $(a_k,\vw_k)$ of each neuron can be separated into a unit orientation feature $\hat{\vw}=\vw/\norm{\vw}_{2}$ and an amplitude $A=|a|\norm{\vw}_{2}$ indicating its contribution to the output, that is, $(A,\hat{\vw})$. For the one-dimensional input, $\vw$ is two dimensional due to the incorporation of bias. Therefore, we use the angle to the $x$-axis in $[-\pi,\pi)$ to indicate the orientation of each $\hat{\vw}$. The scatter plot of $\{(A_k,\hat{\vw}_k)\}_{k=1}^{m}$ is shown in the second row in Fig.~\ref{fig:targetfunc}. Clearly, the evolution of the parameters of the examples in the first row of Fig.~\ref{fig:targetfunc} are different. For $\gamma=0.5$, the initial scatter plot is very close to the one after training. However, for $\gamma=1.75$, active neurons (i.e., neurons with significant amplitude $A$) are condensed at a few orientations, which strongly deviates from the initial scatter plot.
%Parameter $\vw$ is inside the non-linear activation function, therefore, the deviation of $\vw$ from its initialization would be critical to decide whether the NN is in the non-linear regime. 
%In the latter case, it is not clear whether the NN is in the non-linear regime since the change can be induced by either $\vw$ or $a$ or both.
\begin{figure}
    \centering
    \subfloat[$\gamma=0.5$]{
        \includegraphics[width=0.25\textwidth]{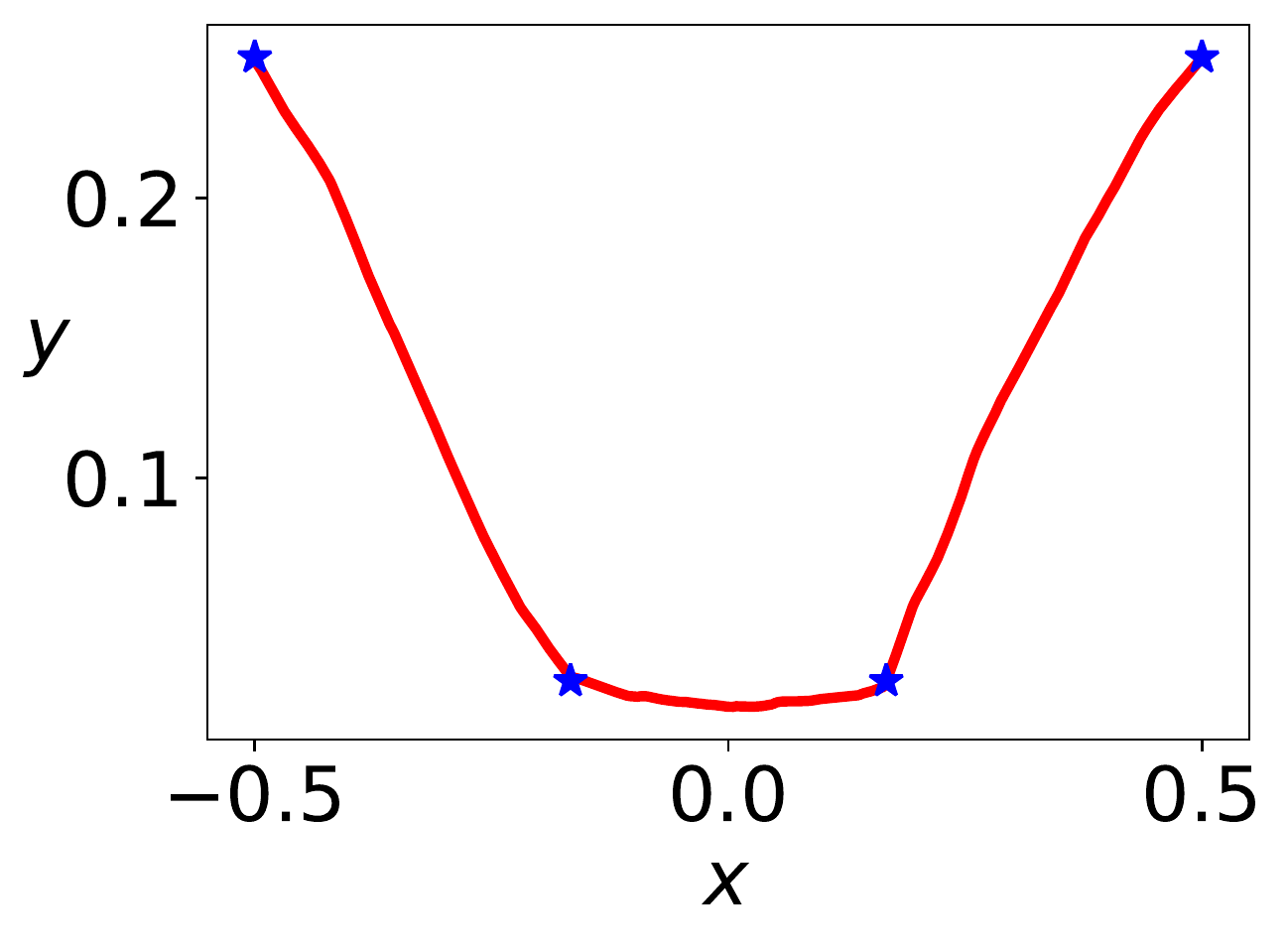}
    }
    \subfloat[$\gamma=1$]{
        \includegraphics[width=0.25\textwidth]{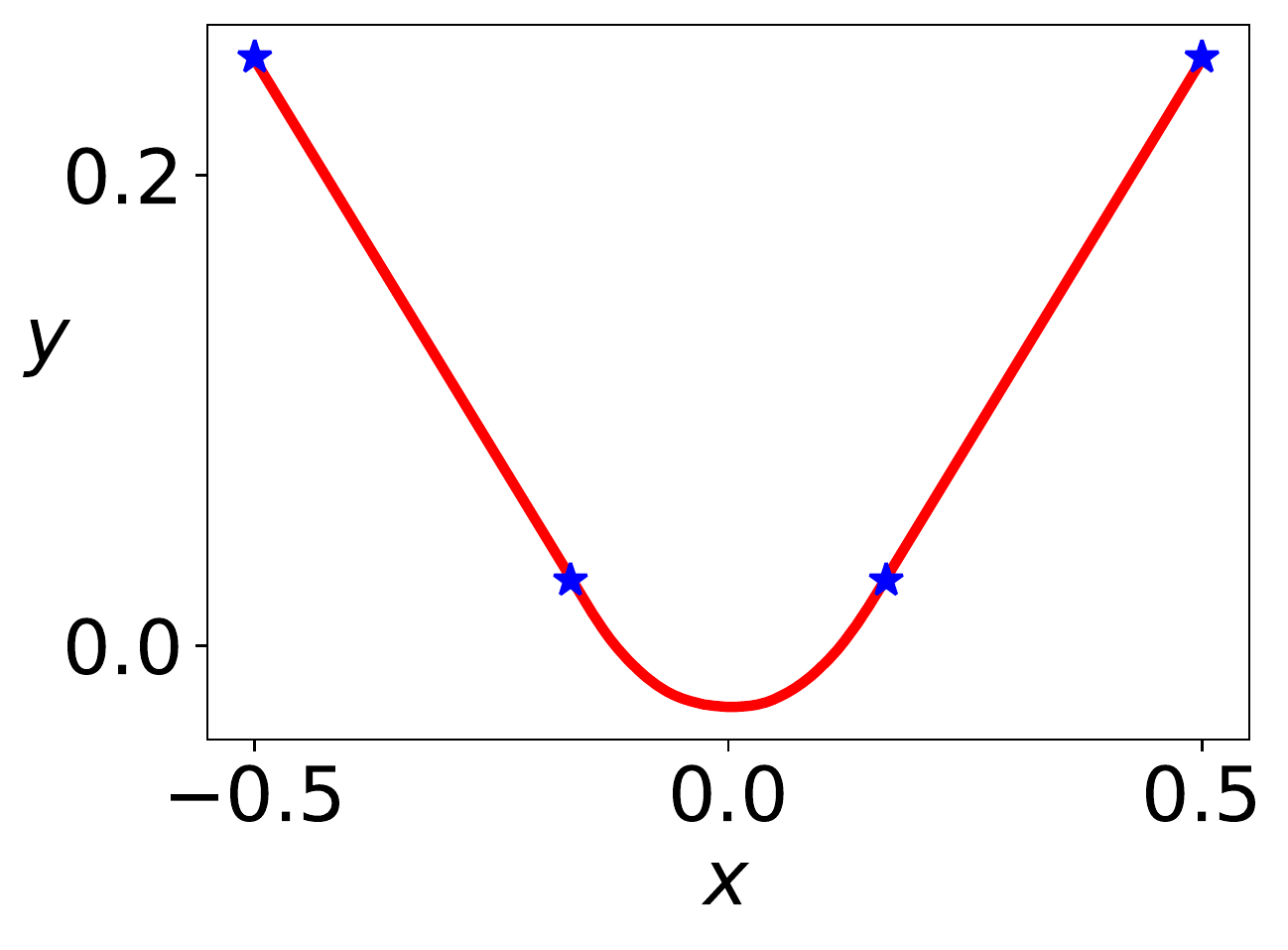}
    }
    \subfloat[$\gamma=1.75$]{
        \includegraphics[width=0.25\textwidth]{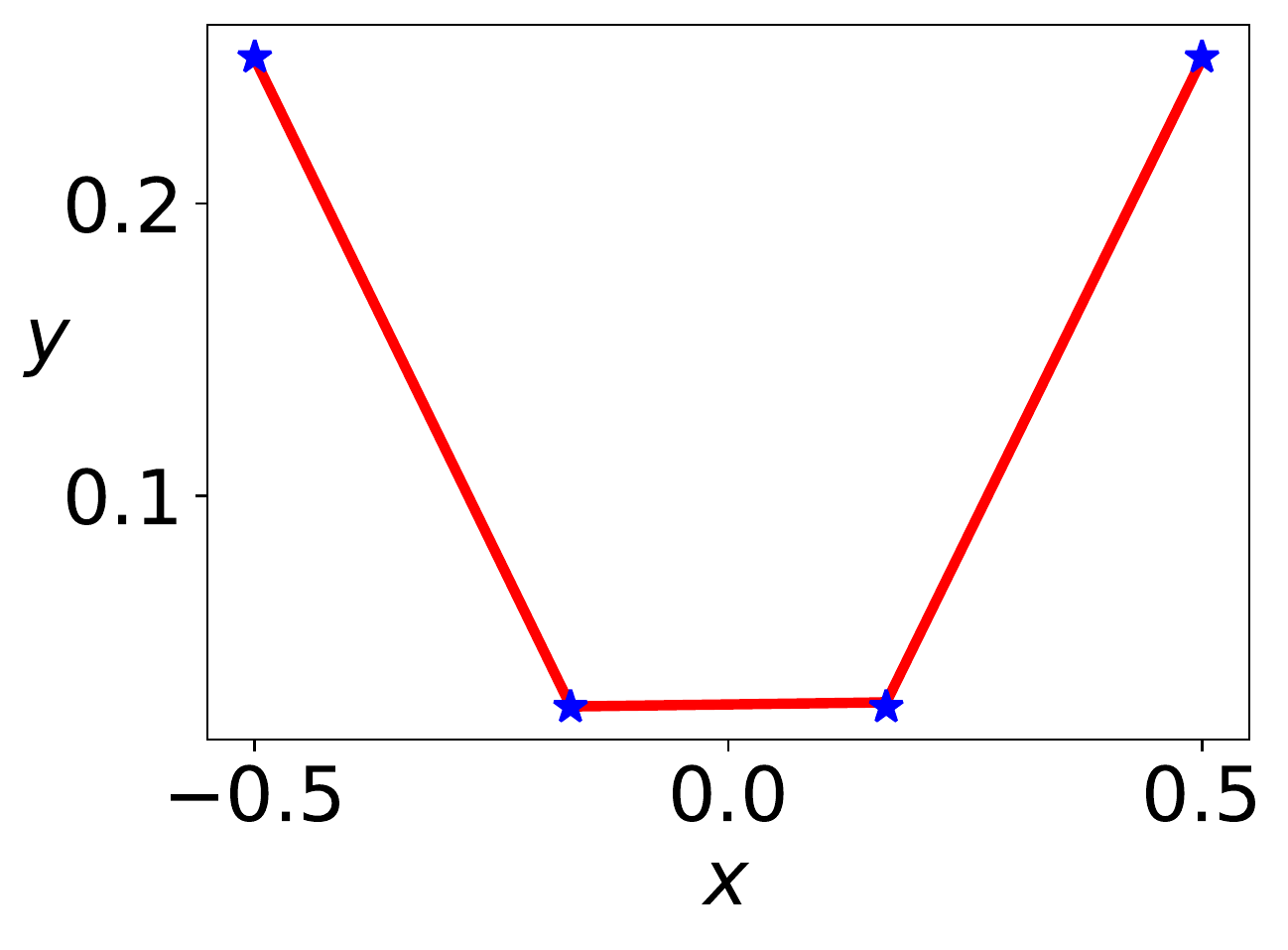}
    }
    \\
    \subfloat[$\gamma=0.5$]{
        \includegraphics[width=0.25\textwidth]{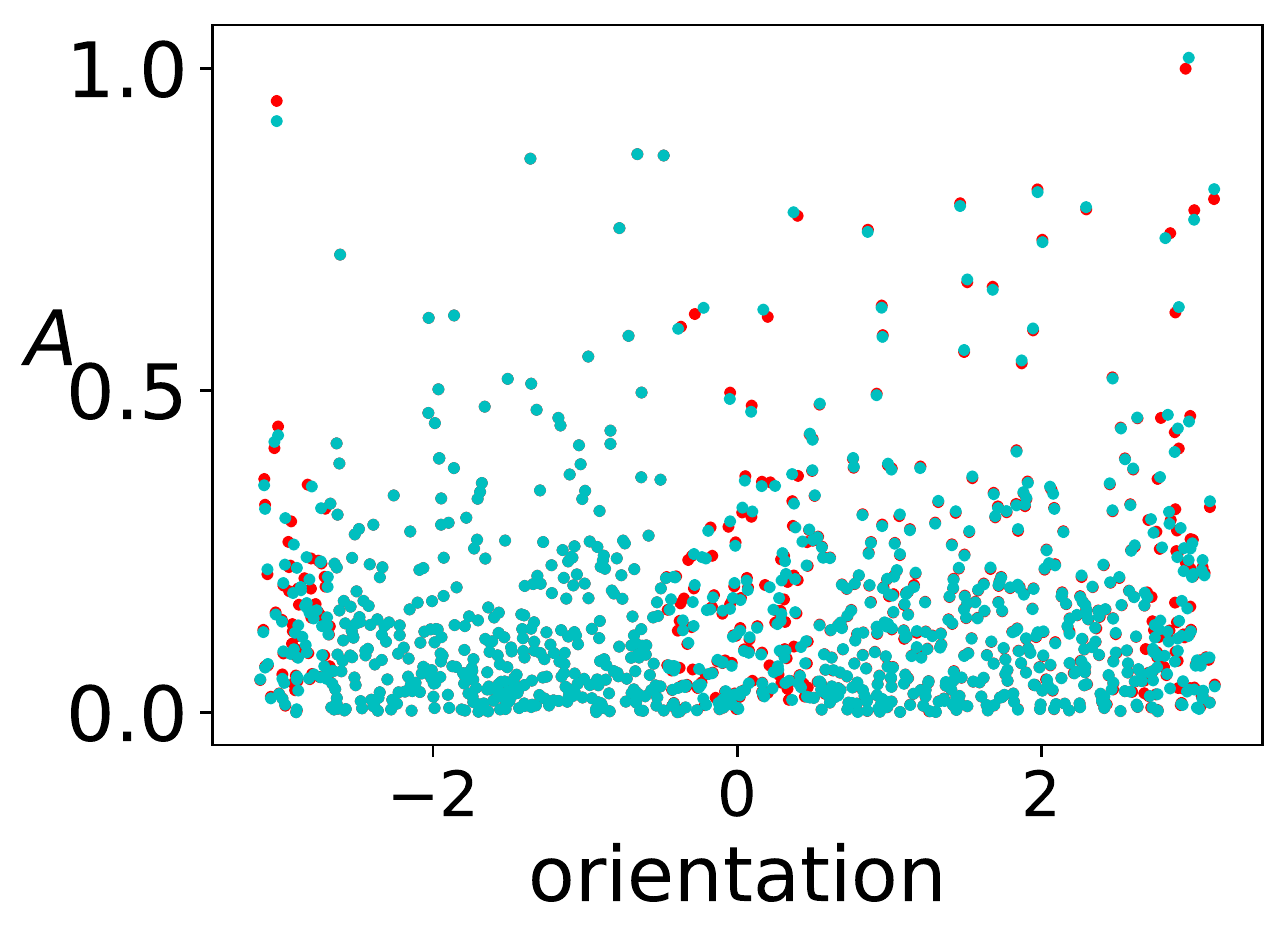}
    }
    \subfloat[$\gamma=1$]{
        \includegraphics[width=0.25\textwidth]{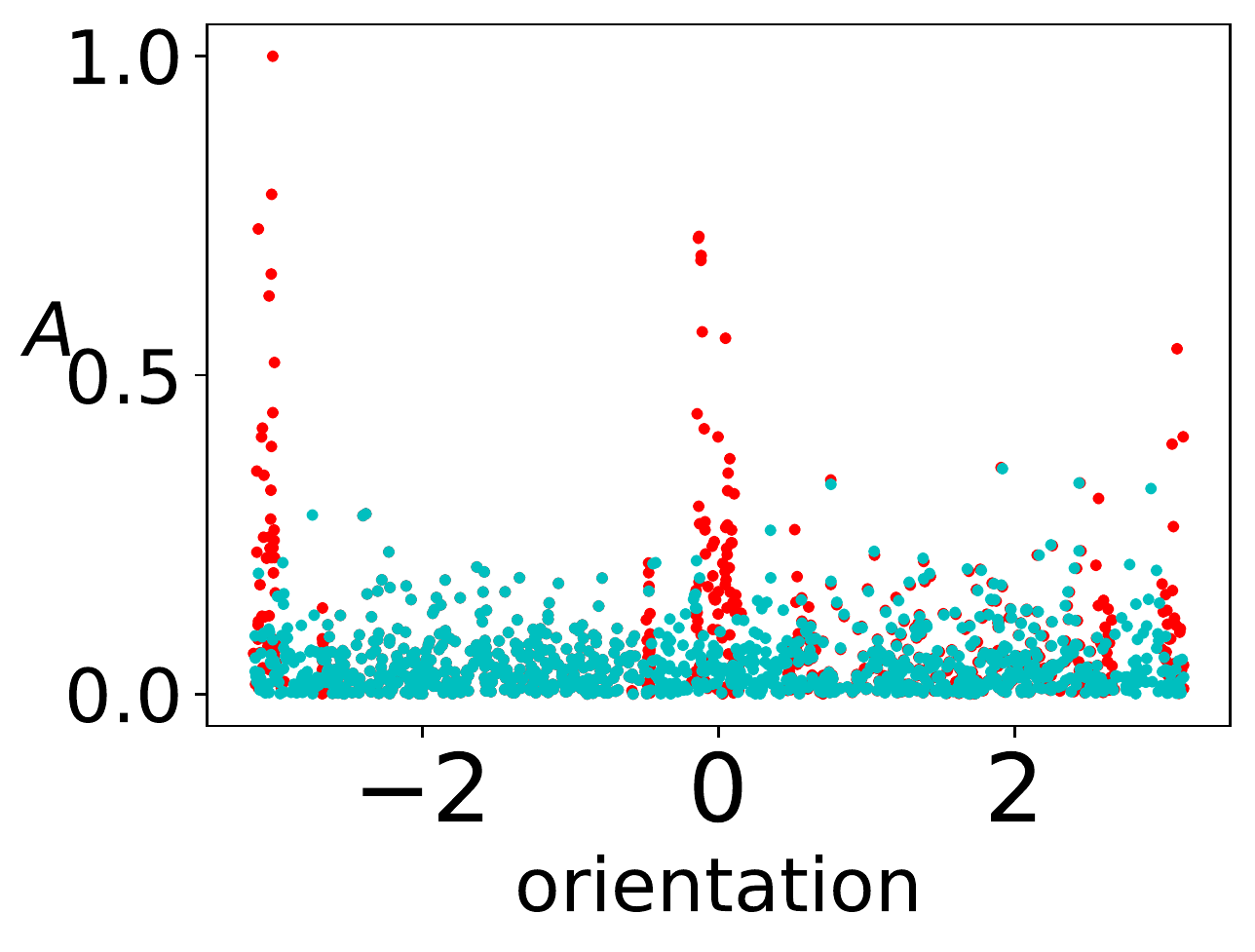}
    }
    \subfloat[$\gamma=1.75$]{
        \includegraphics[width=0.25\textwidth]{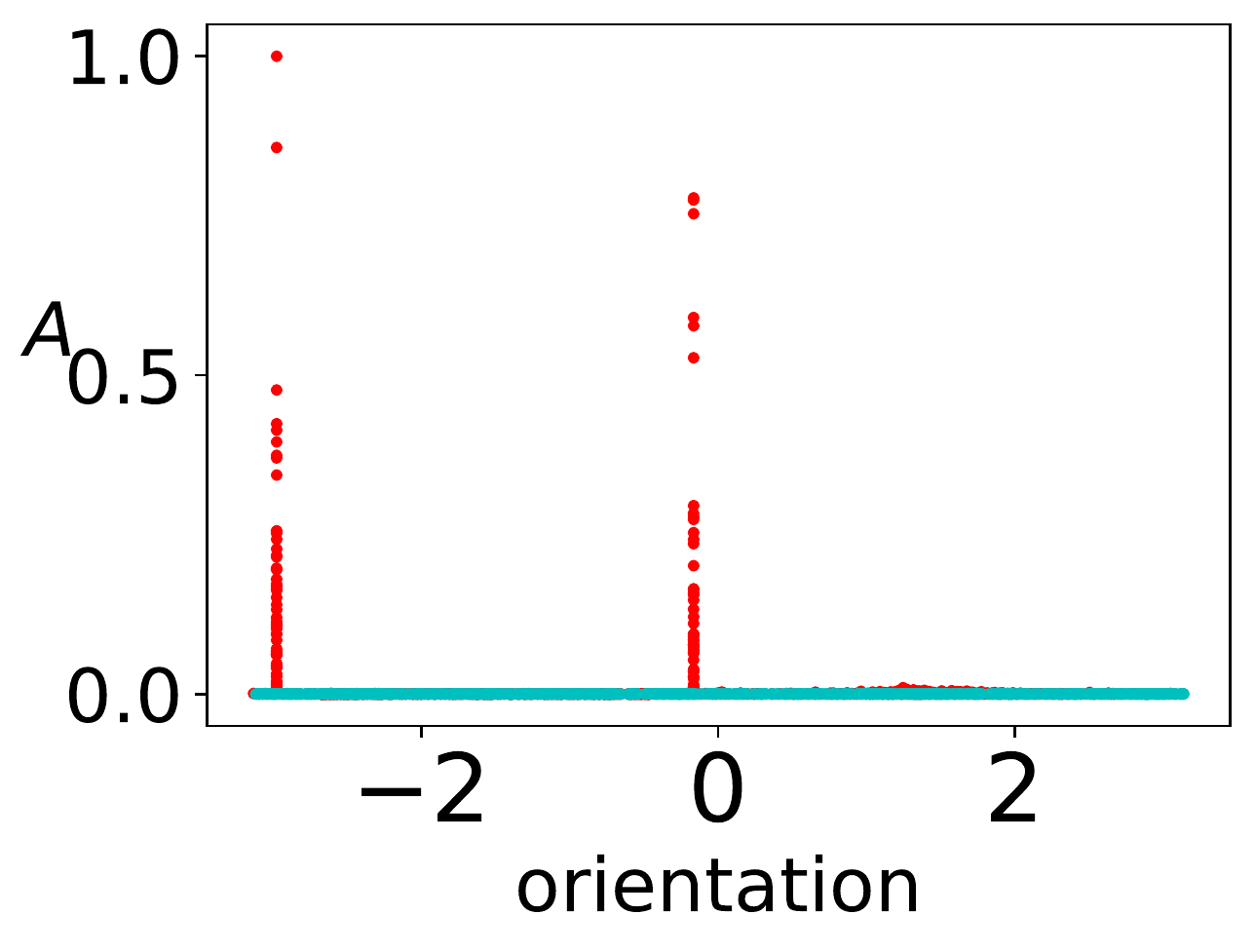}
    }
    \caption{Learning four data points by two-layer ReLU NNs with different $\gamma$'s are shown in the first row. The corresponding scatter plots of initial~(cyan) and final~(red) $\{(A_k,\hat{\vw}_k)\}_{k=1}^{m}$ are shown in the second row. $\gamma'=0$ ($\beta_1=\beta_2=1$), hidden neuron number $m=1000$.}
    \label{fig:targetfunc}
\end{figure}

% E et al.~\footnotesize{\cite{e2020comparative}} & $1$ & $m^{-\gamma}$ & $1$ & $m^{-\gamma}$ & $m^{-\gamma}$ & $(-6,+\infty)$&  $\gamma$ \\
\section{Phase diagram}
In this section, with $\gamma$ and $\gamma'$ as coordinates, we characterize at $m\to\infty$ the dynamical regimes of NNs and identify their boundaries in the phase diagram through experimental and theoretical approaches.
How to characterize and classify different types of training behaviors of NNs is an important open question.  Currently, a behavior of NN dynamics, by which gradient flow of the NN can be effectively linearized around initialization during the training, has been extensively studied both empirically and theoretically~\citep{jacot_neural_2018,lee_wide_2019,arora2019exact,e2020comparative}. We refer to the regime with this behavior as the linear regime. As shown in Fig.~\ref{fig:phase-diagram}, many works have proved that a specific point or line in the phase diagram belong to the linear regime. However, its exact range in the phase diagram remains unclear. On the other hand, NN training dynamics can also be highly nonlinear at $m\to\infty$ as widely studied for the mean-field model as a point shown in the phase diagram~\citep{mei_mean_2018,sirignano_mean_2020,rotskoff_parameters_2018}. However, whether there are other points in the phase diagram that has similar training behavior is not well understood. In addition, it is not clear if there are other regimes in the phase diagram that are nonlinear but behaves distinctively comparing to the mean-field model. In the following, we will address these problems and draw the phase diagram.

\subsection{Regime identification and separation}
The linear regime refers to the set of coordinates with which the gradient flow of $f_{\vtheta}$ at any $t$ is well approximated by gradient flow of its linearized model, i.e.,
\begin{equation}
    f^{\mathrm{lin}}_{\vtheta}=\nabla_{\vtheta}f_{\vtheta(0)}\cdot(\vtheta(t)-\vtheta(0)).
\end{equation}
Note that, the zeroth order term $f_{\vtheta(0)}$ does not appear because, without loss of generality, it is always offset to $0$ by the ASI trick to eliminate the extra generalization error induced by a random initial function as studied in~\citep{zhang_type_2019}. In general, this linear behavior only happens when $\vtheta(t)$ always stays within a small neighbourhood of $\vtheta(0)$ such that the first order Taylor expansion is a good approximation. For a two-layer NN, because its output layer is always linear w.r.t.\  output weights, this requirement of small neighbourhood is reduced to the one for the input weights, that is, $\vtheta_{\vw}(t)$ always stays within a neighbourhood of $\vtheta_{\vw}(0)$. Since the size of this neighbourhood of good linear approximation scales with $\norm{\vtheta_{\vw}(0)}_{2}$, therefore we use the following relative distance as an indicator of how far $\vtheta_{\vw}(t)$ deviates from $\vtheta_{\vw}(0)$ during the training
\begin{equation}
    \mathrm{RD}(\vtheta_{\vw}(t))=\frac{\norm{\theta_{\vw}(t)-\theta_{\vw}(0)}_{2}}{\norm{\theta_{\vw}(0)}_{2}}.
\end{equation}
Specifically, we focus on quantity $\sup\limits_{t\in[0,+\infty)}\mathrm{RD}(\vtheta_{\vw}(t))$, which is the maximum deviation of $\vtheta_{\vw}(t)$ from initialization during the training. As $m\to\infty$, if $\sup\limits_{t\in[0,+\infty)}\mathrm{RD}(\vtheta_{\vw}(t))\to 0$, then the NN training dynamics falls into the linear regime. Otherwise, if it approaches $O(1)$ or $+\infty$, then NN training dynamics is nonlinear. Note that, for the latter case, in which $\vtheta_{\vw}$ deviates infinitely far away from initialization, a very strong nonlinear dynamical behavior of condensation in feature space can be observed as illustrated in Fig. \ref{fig:targetfunc}f. We refer to the regime of $\sup\limits_{t\in[0,+\infty)}\mathrm{RD}(\vtheta_{\vw}(t))\to +\infty$ the condensed regime, which is justified latter in Sec. \ref{sec:ccr} by detailed experiments. For $\sup\limits_{t\in[0,+\infty)}\mathrm{RD}(\vtheta_{\vw}(t))\to O(1)$, NNs exhibit an intermediate level of nonlinear behavior. We refer to this regime as the critical regime.

In the following, we will separate exactly the linear regime and the condensed regime in the phase diagram through experimental and theoretical approaches. We first present an intuitive scaling analysis to provide a rationale for the boundary that separates these two regimes in the phase diagram. Then, we experimentally demonstrate the validity of this boundary in regime separation in the phase diagram for an $1$-d dataset. Finally, we establish a rigorous theory which proves the transition across this boundary for two-layer ReLU NNs at $m\to\infty$ for general datasets.

\subsubsection{Intuitive scaling analysis}
Before we jump into a detailed analysis, through an intuitive scaling analysis, we first illustrate the separation between the linear regime and the condensed regime. The capability, i.e., the magnitude of target function that can be fitted, of the two-layer ReLU NN around initialization can be roughly estimated as
\[
    C = m \beta_1 \beta_2/\alpha=m\kappa.
\]
Without loss of generality, the target function is always $O(1)$. Therefore, a necessary condition for the linear regime is that NN has the capability of fitting the target in the vicinity of initialization, i.e., $C \gtrsim O(1)$. Therefore,
\[
    \kappa \gtrsim 1/m,
\]
yielding $\gamma\leq1$ at $m\to\infty$. We further notice that, the output layer is always linear. Therefore, even when the output weight $\vtheta_{a}$ changes significantly, the dynamics can still be linearized if the input layer weight $\vtheta_{\vw}$ stays in the vicinity of its initialization. As indicated by the dynamics Eq. (\ref{eq:qdynamics}), this is possible when (i) $\kappa' \ll 1$ at initialization and (ii) the scale of $a$, say quantified by expectation $\Exp(|a|)$, satisfies $\Exp(|a|) \ll \beta_2$ throughout the training.  In this case, at the end of the training,
\begin{equation}
    C = m\beta_2 \Exp(|a|)/\alpha  \ll m\beta_2^2/\alpha=m\kappa/\kappa'.
\end{equation}
Because $C \gtrsim O(1)$, we got
\begin{equation}
    1/\kappa'\gg 1/m\kappa,
\end{equation}
which yields the condition $\gamma'>\gamma-1$ for $\gamma'>0$ at $m\to\infty$.

In contrary, if $\gamma'<\gamma-1$ and $\gamma>1$, i.e., $m\kappa\ll 1$ and $m\kappa/\kappa'\ll 1$ as $m\to\infty$, then the NN has no capability in fitting a $O(1)$ target when $\vtheta_{\vw}$ stays at the vicinity of its initialization. The capability of NN must undergo a magnificent increase to be able to fit the data, which is a feature of the condensed regime.

Above scaling analysis provides an intuitive argument about the separation of linear and condensed regimes by the boundary $\gamma=1$ for $\gamma'\leq0$ and $\gamma'=\gamma-1$ for $\gamma'>0$ in the phase diagram. To further demonstrate the criticality of this boundary, we sort to the following experimental studies for a specific case.

\subsubsection{Experimental demonstration}
To experimentally distinguish the linear and nonlinear regimes, we need to estimate 
\[
\sup\limits_{t\in[0,+\infty)}\mathrm{RD}(\vtheta_{\vw}(t)),
\]
which empirically can be approximated by $\mathrm{RD}(\vtheta_{\vw}^*)$ ($\vtheta_{\vw}^*:=\vtheta_{\vw}(\infty)$) without loss of generality. Next, because we can never run experiments at $m\to\infty$, we alternatively quantify the growth of $\mathrm{RD}(\vtheta^*_{\vw})$ as $m\to\infty$. By Fig.~\ref{fig:diffbeta} (a-c), they approximately have a power-law relation. Therefore we define
\begin{equation}
    S_{\vw}=\lim_{m\to\infty}\frac{\log \mathrm{RD}(\vtheta^*_{\vw})}{\log m},
\end{equation}
which is empirically obtained by estimating the slope in the log-log plot like in Fig.~\ref{fig:diffbeta}. As shown in Fig.~\ref{fig:diffbeta} (d), NNs with the same pair of $\gamma$ and $\gamma'$, but different $\alpha$, $\beta_1$, and $\beta_2$, have very similar $S_{w}$, which validates the effectiveness of the normalized model. In the following experiments, we only show result of one combination of $\alpha$, $\beta_1$, and $\beta_2$ for a pair of $\gamma$ and $\gamma'$.
\begin{figure}
    \centering
    \subfloat[$\gamma=0.5$]{
        \includegraphics[width=0.23\textwidth]{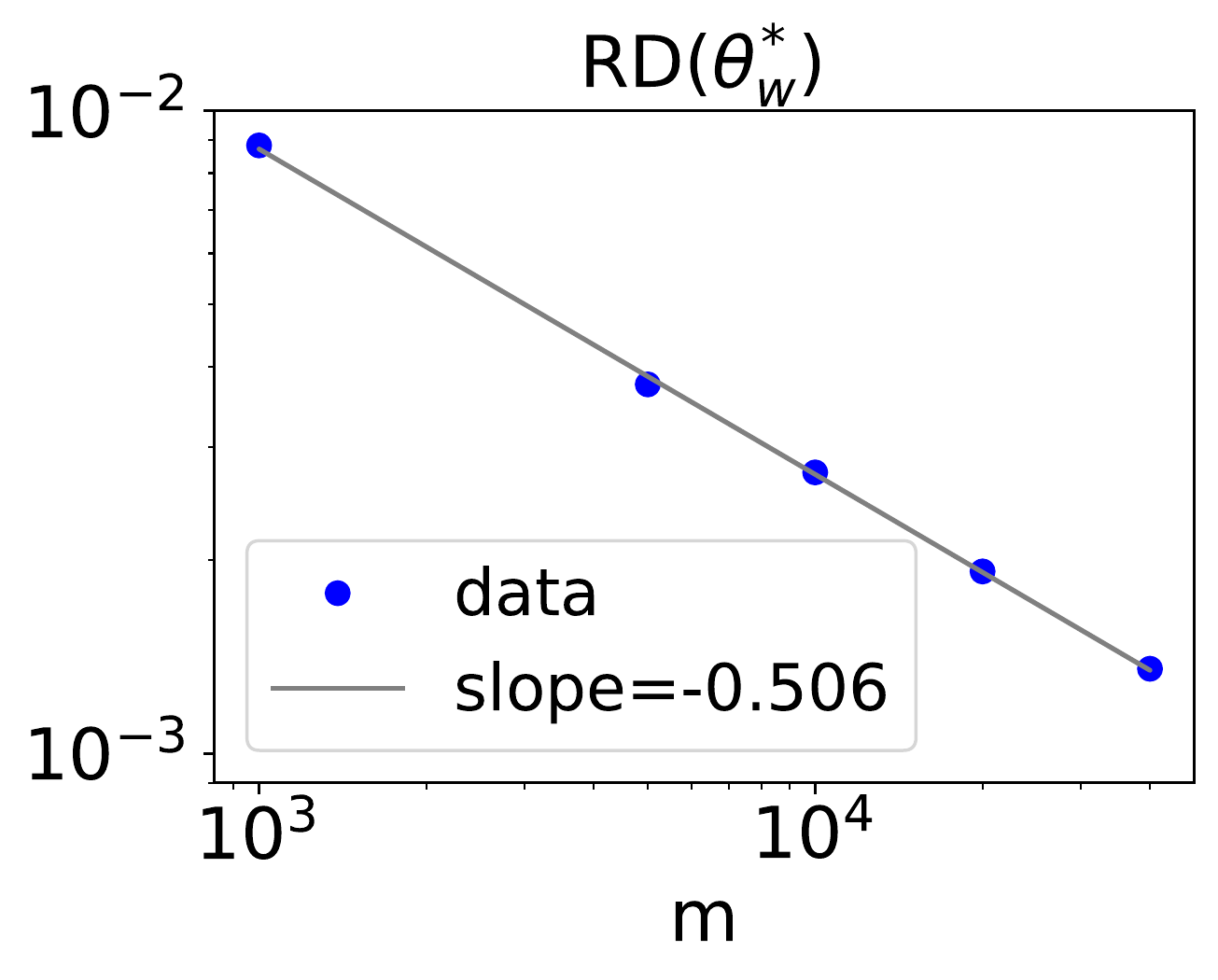}
    }
    \subfloat[$\gamma=1$]{
        \includegraphics[width=0.23\textwidth]{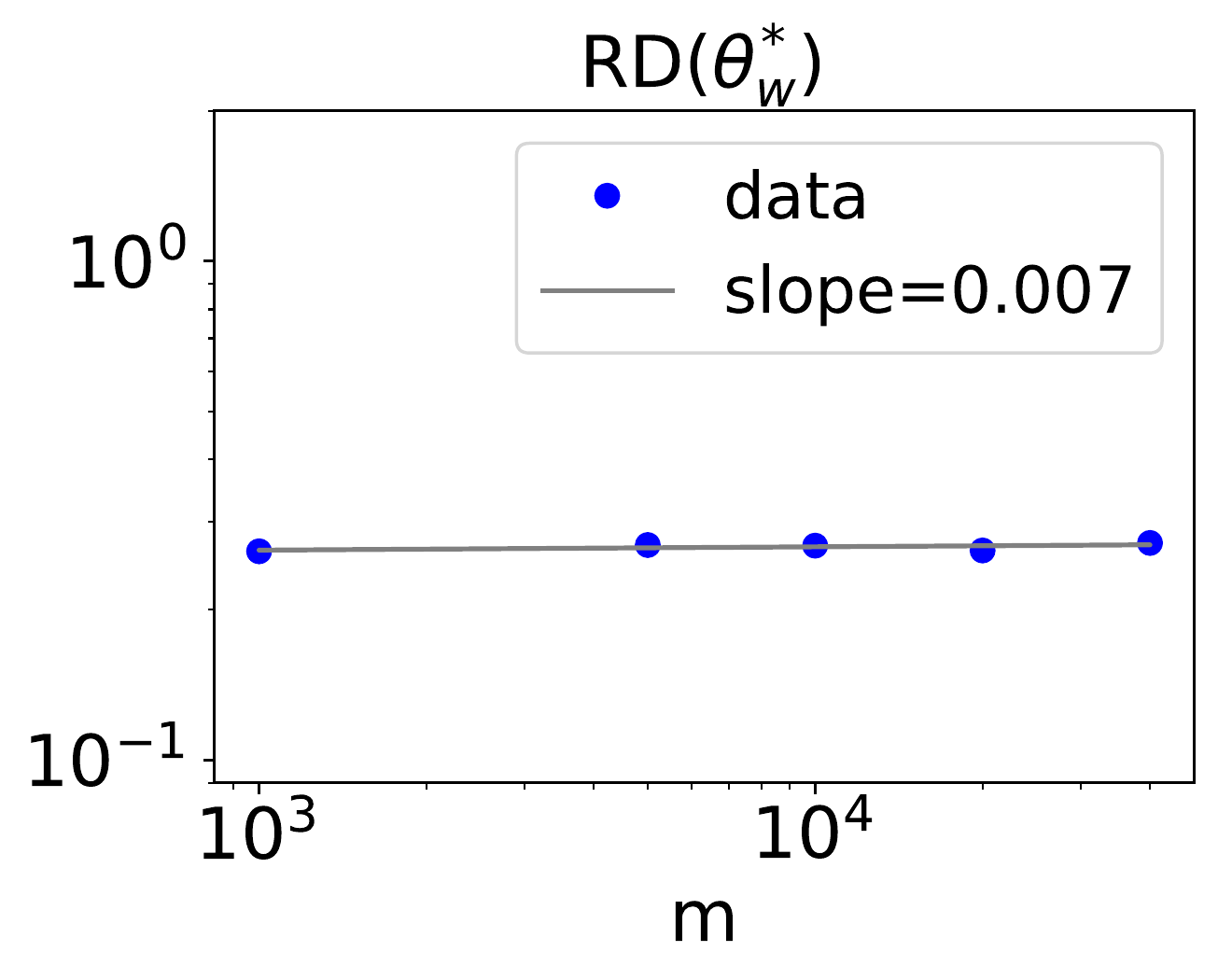}
    }
    \subfloat[$\gamma=1.75$]{
        \includegraphics[width=0.23\textwidth]{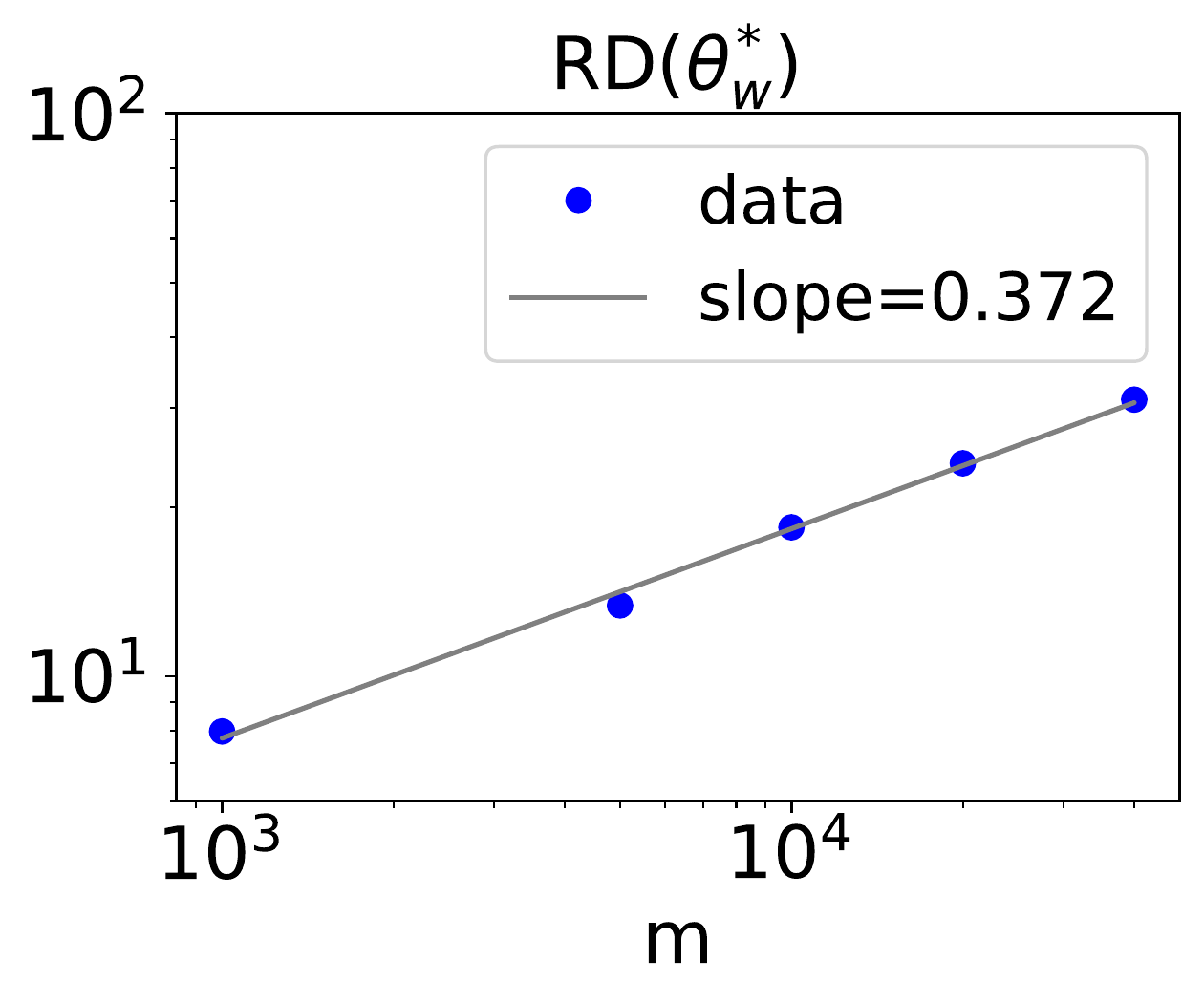}
    }
    \subfloat[$S_{w}$ vs. $\gamma$]{
        \includegraphics[width=0.23\textwidth]{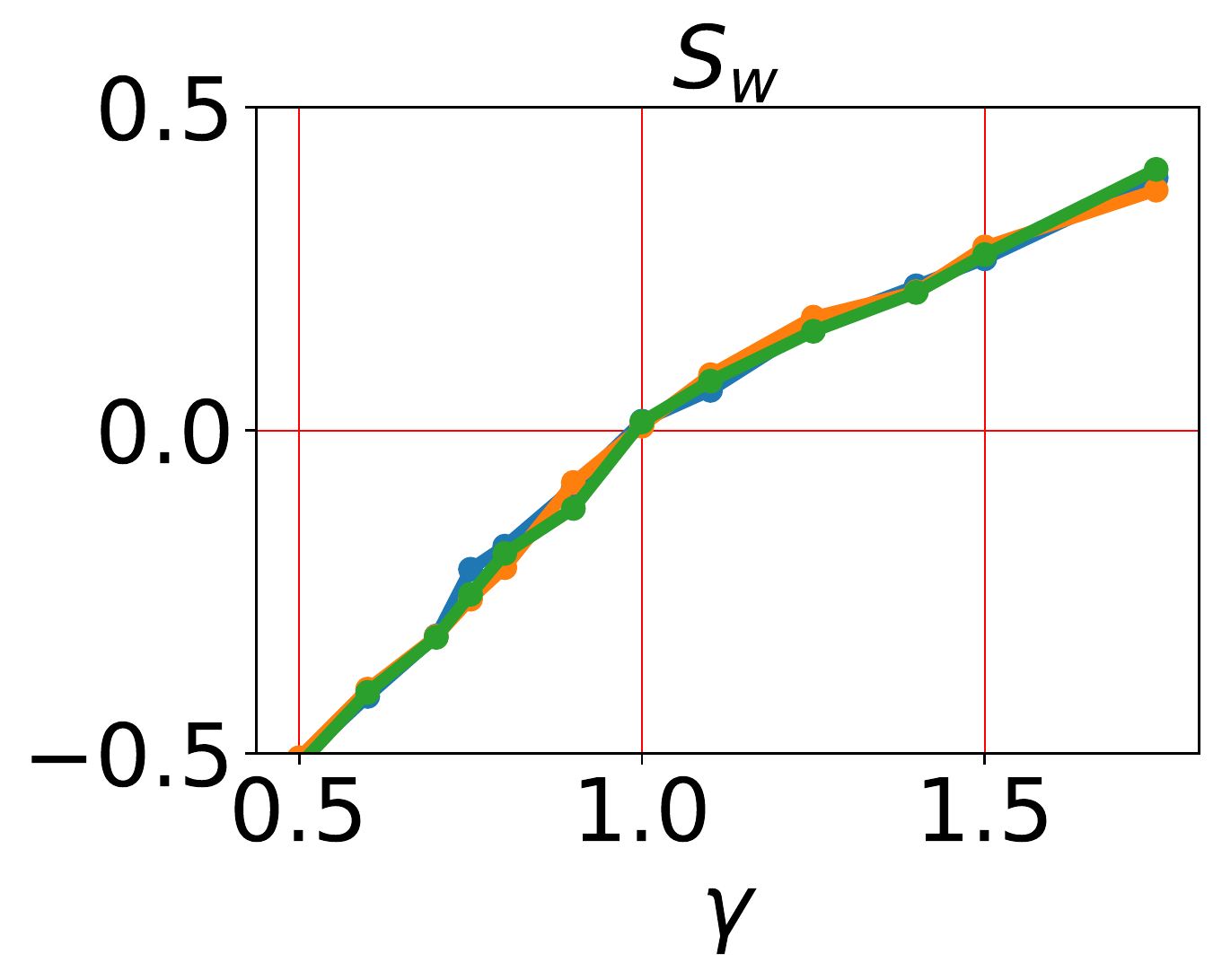}
    }
    \caption{Growth of $\mathrm{RD}(\vtheta^*_{\vw})$ w.r.t.\ $m\to\infty$ with $\gamma'=0$. For (a-c), the plot is $\mathrm{RD}(\vtheta^*_{\vw})$ vs.\ $m$ of NNs with $1000,5000,10000,20000,40000$ hidden neurons indicated by five blue dots, respectively. The gray line is a linear fit with slope indicated. For (d), the plot is $S_{\vw}$ vs. $\gamma$ for $\gamma'=0$. Each line is for a pair of $\beta_1$ and $\beta_2$: Blue: $\beta_1=1$, $\beta_2=1$; Orange: $\beta_1=m^{-1/2}$, $\beta_2=m^{-1/2}$; Blue: $\beta_1=m^{-1}$, $\beta_2=m^{-1}$.}
    \label{fig:diffbeta}
\end{figure}

Then, we visualize the phase diagram by experimentally scanning $S_w$ over the phase space. The result for the same $1$-d problem as in Fig. \ref{fig:targetfunc} is presented in Fig.~\ref{fig:wds}. In the red zone, where $S_{\vw}$ is less than zero, $\mathrm{RD}(\vtheta^*_{\vw})\to0$ as $m\to\infty$, indicating a linear regime. In contrast, in the blue zone, where $S_{\vw}$ is greater than zero, $\mathrm{RD}(\vtheta^*_{\vw})\to \infty$ as $m\to\infty$, indicating a highly nonlinear behavior. Their boundary are experimentally identified through interpolation indicated by stars in Fig.~\ref{fig:wds}, where $\mathrm{RD}(\vtheta^*_{\vw})\sim O(1)$. They are close to the boundary identified through the scaling analysis indicated by the auxiliary lines, justifying its criticality. Similarly, we use two-layer ReLU NNs to fit MNIST dataset with mean squared loss. In our experiments, the input is a $784$ dimensional vector and the output is the one-dimensional label ($0\sim9$) of the input image. As shown in Fig.~ \ref{fig:mnistwds}, the phase diagram obtained by the synthetic data also applies for such real high-dimensional dataset. 
\begin{figure}
    \centering
    \includegraphics[width=0.5\textwidth]{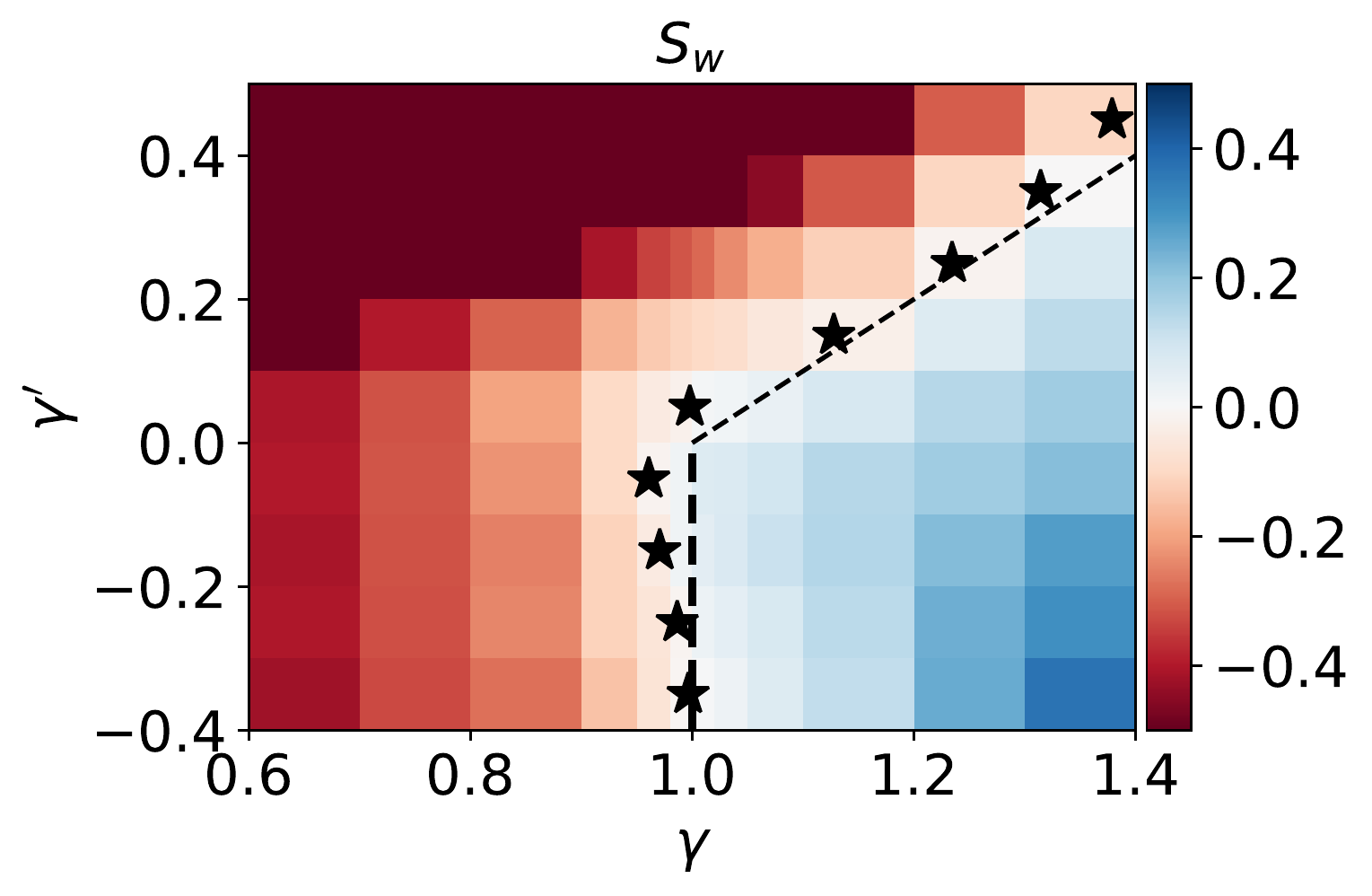}
    \caption{For synthetic data, $S_{\vw}$ estimated on two-layer ReLU NNs of $1000$, $5000$, $10000$, $20000$, $40000$ hidden neurons over $\gamma$ (ordinate) and $\gamma'$ (abscissa). The stars are zero points obtained by the linear interpolation over different $\gamma$ for each fixed $\gamma'$. Dashed lines are auxiliary lines indicating the theoretically obtained boundary.}
    \label{fig:wds}
\end{figure}

\begin{figure}
    \centering
    \includegraphics[width=0.5\textwidth]{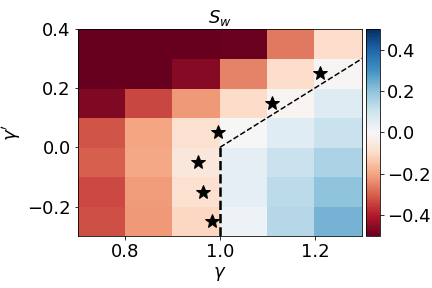}
    \caption{For MNIST data, $S_{\vw}$ estimated on two-layer ReLU NNs of $1000$, $10000$, $50000$, $250000$, $40000$ hidden neurons over $\gamma$ (ordinate) and $\gamma'$ (abscissa). The stars are zero points obtained by the linear interpolation over different $\gamma$ for each fixed $\gamma'$. Dashed lines are auxiliary lines indicating the theoretically obtained boundary.}
    \label{fig:mnistwds}
\end{figure}

\subsubsection{Theoretical results for general two-layer ReLU NNs}
The intuitive scaling analysis and the experimental demonstration result in a consistent boundary to separate the linear and condensed regimes. A question naturally arises---is there a theory that makes the intuitive scaling analysis rigorous and generalizes above empirical phase diagram for an $1$-d example to general high-dimensional data for two-layer ReLU NNs.
In the following, we address this question by providing two theorems in informal statements, which proves the criticality of $\lim\limits_{m\to+\infty}\sup\limits_{t\in[0,+\infty)}\mathrm{RD}(\vtheta_{\vw}(t))$ at above identified boundary in the phase diagram. Their rigorous statements can be found in Section \ref{sec..Theorem}.

%Specifically, we prove two theorems for a complete characterization of the behavior of $\sup\limits_{t\in[0,+\infty)}\mathrm{RD}(\vtheta_{\vw}(t))$ for the linear and non-linear regimes, respectively, and leave the theoretical study on the boundary, i.e. critical regime, for future study. \textcolor{red}{why mention "near the boundary."}.  

%The theorems are informally stated as follows while their rigorous statements can be found in Section \ref{sec..Theorem}.

\begin{customthm}{1*}
    (Informal statement of Theorem \ref{thm..LinearRegime}) If $\gamma<1$ or $\gamma'>\gamma-1$, then with a high probability over the choice of $\vtheta^0$, we have
    \begin{equation}
        \lim_{m\to+\infty}\sup\limits_{t\in[0,+\infty)}\mathrm{RD}(\vtheta_{\vw}(t))=0.
    \end{equation}
\end{customthm}
\begin{customthm}{2*}
    (Informal statement of Theorem~\ref{thm..CondensedRegime}) If $\gamma>1$ and $\gamma'<\gamma-1$, then with a high probability over the choice of $\vtheta^0$, we have
    \begin{equation}
        \lim_{m\to+\infty} \sup\limits_{t\in[0,+\infty)}\mathrm{RD}(\vtheta_{\vw}(t))=+\infty.
    \end{equation}
\end{customthm}
\begin{remark}
    $\lim\limits_{m\to+\infty}\sup\limits_{t\in[0,+\infty)}\mathrm{RD}(\vtheta_{\vw}(t))$ is like an order parameter in the analysis of phase transition in statistical mechanics, which is key to the regime separation and exhibits discontinuity at the boundary.
\end{remark}

In Theorem 1*, focusing on the linear regime, the negligible relative change of $\vw$ is essentially proved by showing the kernel of the training dynamics undergoes no significant change during the whole dynamics. However, the kernel of the training dynamics might be out of control for the condensed regime. This difficulty makes the result of Theorem 2* nontrivial. Instead of studying the kernel, more detailed information of the dynamics should be used. Indeed, we establish a neural-wise estimate, $\abs{a_k(t)}\leq \frac{1}{\kappa'}\norm{\vw_k(t)}_2+\abs{a_k^0}$, which holds for any $\kappa,\kappa'$ and any $t\geq 0$. We believe that this estimate can be extended to other network structures and general activation functions for the regimes of nonlinear dynamics.

Above two theorems complete the phase diagram of two-layer ReLU NN with distinctive dynamical regimes separated based on $\lim\limits_{m\to+\infty}\sup\limits_{t\in[0,+\infty)}\mathrm{RD}(\vtheta_{\vw}(t))$. The behavior of NN in the linear regime, e.g., exponential decay of loss, implicit regularization in terms of a RKHS norm, and etc., is very well studied. However, the critical and condense regimes is largely not understood. In the following, we make a further step to unravel a signature nonlinear behavior---condensation as $m\to\infty$ through experiments, which sheds light on future theoretical study.

\subsection{Critical and condensed regimes}\label{sec:ccr}
%Above, we clearly identified the boundary of the linear regime in the phase diagram. However, whether the rest part has similar nonlinear characteristics in their training dynamics is not known and remains theoretical challenge.  \textcolor{red}{In this subsection, we empirically explore the characteristics of the non-linear regime and find a condensation phenomenon, which sheds light on future theoretical study.}

By Fig.~\ref{fig:targetfunc} (d-f) in previous section, it can be observed that the condensation of NN representation in feature space $\{(A_k,\hat{\vw}_k)\}_{k=1}^m$ comparing to initialization is a distinctive feature for the nonlinear training dynamics of NNs. Specifically, we care about this condensation at the $m\to\infty$ limit when relative change of $\vtheta_{w}$ approaches $+\infty$. Therefore, using the same $1$-d data as in Fig.~\ref{fig:targetfunc}, we scan the learned distribution of $(A_k,\hat{\vw}_k)$ pair for $m=10^3,10^4,10^6$ over the phase diagram to experimentally find out the limiting behavior. The result is shown in Fig.~\ref{fig:cdnmap}. It is easy to observe that, right to the boundary indicated by blue boxes, the condensation becomes stronger as $m\to\infty$, implying a $\delta$-like condensation behavior at the limit. This conforms with our intuition that the farther away $\vtheta_{\vw}$ deviates from initialization, the stronger nonlinearity of NNs exhibited here in the form of condensation. Therefore, as introduced before, we refer to this regime as the \emph{condensed} regime. In the critical regime as the boundary between the linear and the condensed regimes, the level of condensation is almost fixed as $m\to\infty$, which resembles a mean-field behavior. Indeed, the well-studied mean-field model is one point in the critical regime shown in the phase diagram Fig.~\ref{fig:phase-diagram}. In general, the mechanism of condensation as well as its implicit regularization effect is not well understood, which remain as important open questions for the future research.

We also examine the condensation of NNs for MNIST dataset. For such high-dimensional data, it is impossible to directly visualize the distribution in the high-dimensional feature space like above $1$-d case. Therefore, we consider a projection approach, by which we project each $\hat{\vw}$ to a reference direction $\vp$ and plot $A_k$ vs. $I_{\hat{\vw}}=\hat{\vw}\cdot \vp$. Note that the reference direction can be arbitrary selected and does not affect our conclusion. Without loss of generality, we pick $\vp=\mathbf{1}/\sqrt{n}$. Clearly, if neurons indeed condensed at several directions in the high-dimensional feature space, then their $1$-d projection should also condense at several points. As shown in  Fig. \ref{fig:mnistcdnmap}, similar to the $1$-d case, condensation behavior can be observed in the condensed regime identified above. As the parameters move further away from the boundary in the condensed regime, condensation becomes more salient.
%Since $\vw$'s are independently sampled from a Gaussian distribution, $\hat{\vw}$'s uniformly distribute on a high-dimensional sphere. Therefore, the number of the  $\vw$   is larger when the angle between  $\vw$  and the reference direction is closer to perpendicular, i.e., $I_{\hat{\vw}}$ is closer to zero. In the scatter plot of $\{(A_k,I_{\hat{\vw}_k})\}_{k=1}^m$, the range of the $A_k$'s is larger when $I_{\hat{\vw}_k}$ is closer to $0$ due to more samples. For example, scatter plots for $\gamma<1$ in Fig. \ref{fig:mnistcdnmap}, which are close to initialization, are not uniformly distributed over $I_{\hat{\vw}_k)}$. To see a clear condensation phenomenon for high-dimensional dataset, it requires a extremely wide NN, which may be not realistic for a moderate computer. 
%However, with NNs with hidden layer of size up to $2.5\times 10^{5}$, as shown in  Fig. \ref{fig:mnistcdnmap}, in the condensation regime identified through the synthetic data, we still the condensation effect. As the parameters go further from the boundary, the condensation is more significant. }

\begin{figure}
    \centering
    \includegraphics[scale=0.8]{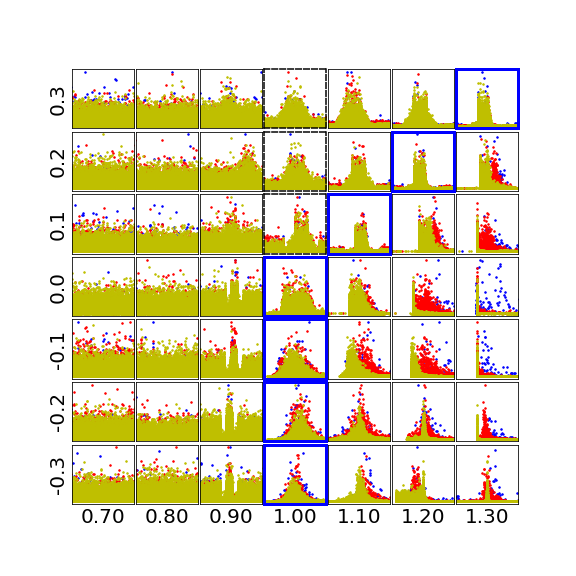}
    \caption{Condensation map for $1$-d synthetic data. Each color in each box is a scatter of $\{(A_k,\hat{\vw}_k)\}_{k=1}^m$ for a NN with corresponding $\gamma$ and $\gamma'$. The hidden neuron numbers are:  $m=10^{3}$ (blue), $10^{4}$ (red), $10^{6}$ (yellow). The abscissa coordinate is $\gamma$ and the ordinate one is $\gamma'$.}
    \label{fig:cdnmap}
\end{figure}

\begin{figure}
    \centering
    \includegraphics[scale=0.8]{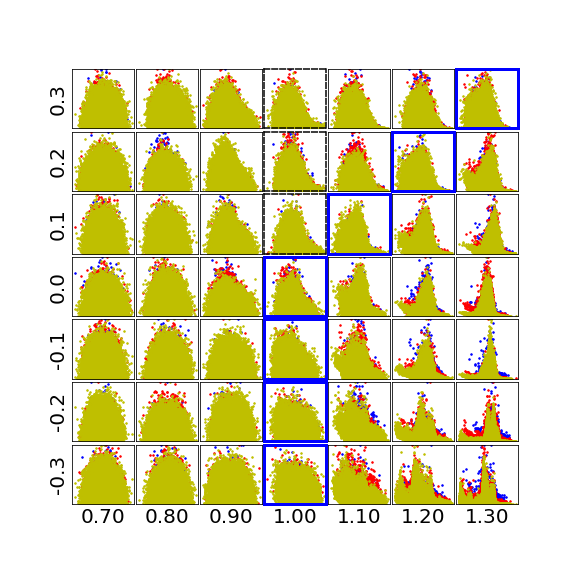}
    \caption{Condensation map for MNIST dataset. Each color in each box is a scatter of $\{(A_k,I_{\hat{\vw}}\}_{k=1}^m$ for a NN with corresponding $\gamma$ and $\gamma'$. The hidden neuron numbers are:  $m=10^{3}$ (blue), $10^{4}$ (red), $2.5\times 10^{5}$ (yellow). The abscissa coordinate is $\gamma$ and the ordinate one is $\gamma'$.}
    \label{fig:mnistcdnmap}
\end{figure}

% As the case of linear regime, we also prove a theorem for the condensed regime by using the relative change of $\vtheta_{\vw}$. It is informally stated as follows while its rigorous statement can be found in Section \ref{sec..Theorem}.
% \begin{thm*}
%     (Informal statement of Theorem~\ref{thm..CondensedRegime}) If $\gamma>1$ and $\gamma'<\gamma-1$, then with a high probability over the choice of $\vtheta^0$, we have
%     \begin{equation}
%         \sup\limits_{t\in[0,+\infty)}\mathrm{RD}(\vtheta_{\vw}(t))=\sup\limits_{t\in[0,+\infty)}\frac{\norm{\vtheta_{\vw}(t)-\vtheta_{\vw}^0}_2}{\norm{\vtheta_{\vw}^0}_2}\gg 1.
%     \end{equation}
% \end{thm*}

\section{Theoretical regime characterization}\label{sec..Theorem}
We illustrate above how our phase diagram Fig. \ref{fig:phase-diagram} is obtained through experimental and theoretical approaches. To obtain a more detailed understanding of general properties of these regimes, we present our theoretical results in detail in this section, which follows a rigorous description of our notations and definitions in the beginning. The proofs can be found in the appendix.

To start with, let us consider a two layer neural network
\begin{equation}
    f_{\vtheta}(\vx) := \frac{1}{\kappa}f^\kappa_{\vtheta}(\vx) = \sum_{k=1}^{m} a_k\sigma(\vw_k^\T\vx),
\end{equation}
with the activation function $\sigma(z) = \ReLU(z) = \max(z, 0)$. Denote the dataset
\begin{equation}
    S = {\{(\vx_i, y_i)\}}_{i=1}^n,
\end{equation}
where $\vx_i$'s are i.i.d.\ sampled from the (unknown) distribution $\fD$ over $\Omega={[0,1]}^d$ with $(\vx_i)_d=1$ and $y_i=f(\vx_i)\in[0,1]$ for all $i\in[n]$.

We denote $e_i = \kappa f_{\vtheta}(\vx_i) - y_i = \kappa f_{\vtheta}(\vx_i) - f(\vx_i)$, $i\in[n]$ and $\ve = {(e_1, e_2, \ldots, e_n)}^{\T}$. Then the empirical risk can be written as
\begin{equation}
    \RS(\vtheta):=R_{S,\kappa}(\vtheta) = \frac{1}{2n}\sum_{i=1}^n{\left(\kappa f_{\vtheta}(\vx_i)-y_i\right)}^2 = \frac{1}{2n}\ve^{\T}\ve.
\end{equation}
Its gradient descent (GD) dynamics is
\begin{equation}\label{eqn:gd_dyn}
    \dot{\vtheta} = -M_{\kappa'}\nabla_{\vtheta}\RS(\vtheta),
\end{equation}
with a more explicit form for $a_k$ and $\vw_k$ respectively
\begin{equation}
    \left \{
    \begin{aligned}
        \dot{a}_k   & = -\frac{1}{\kappa'}\nabla_{a_k}\RS(\vtheta) = -\frac{\kappa}{\kappa'n}\sum_{i=1}^{n}e_i\sigma(\vw_k^\T\vx_i), \\
        \dot{\vw}_k & =-\kappa'\nabla_{\vw_k}\RS(\vtheta) = -\frac{\kappa\kappa'}{n}\sum_{i=1}^n e_i a_i\sigma'(\vw_k^\T\vx_i)\vx_i.
    \end{aligned}
    \right.\label{eq..MainDynamics}
\end{equation}
Here $\kappa,\kappa'$ are scaling parameters proposed in Section \ref{sec..Rescaling}.
The parameters are initialized as
\begin{align}
    a^0_k     & :=a_k(0)\sim N(0,1),                                                                    \\
    \vw_k^0   & :=\vw_k(0)\sim N(0,\mI_d),                                                              \\
    \vtheta^0 & := \vtheta(0) = \mathrm{vec}(\{a_k^0, \vw_k^0\}_{k=1}^m).\label{eq..MainInitialization}
\end{align}
The kernels $k^{[a]}$ and $k^{[\vw]}$ of the GD dynamics are
\begin{equation}
    \begin{aligned}
         & k^{[a]}(\vx,\vx')=\Exp_{\vw}\sigma(\vw^\T\vx)\sigma(\vw^\T\vx'),                        \\
         & k^{[\vw]}(\vx,\vx')=\Exp_{(a,\vw)}a^2\sigma'(\vw^\T\vx)\sigma'(\vw^\T\vx')\vx\cdot\vx'. \\
    \end{aligned}
\end{equation}
The Gram matrices $\mK^{[a]}$ and $\mK^{[\vw]}$ of an infinite width two-layer network are
\begin{equation}
    \begin{aligned}
         & K^{[a]}_{ij}=k^{[a]}(\vx_i,\vx_j), \quad \mK^{[a]}=(K_{ij}^{[a]})_{n\times n},         \\
         & K^{[\vw]}_{ij}=k^{[\vw]}(\vx_i,\vx_j), \quad \mK^{[\vw]}=(K_{ij}^{[\vw]})_{n\times n}.
    \end{aligned}
\end{equation}

\noindent The Gram matrices $\mG^{[a]}$, $\mG^{[\vw]}$, and $\mG$ of a finite width two-layer network have the following expressions
\begin{equation}
    \begin{aligned}
         & G^{[a]}_{ij}(\vtheta)=\frac{1}{\kappa'm}\sum_{k=1}^m\nabla_{a_k}\kappa f_{\vtheta}(\vx_i)\cdot\nabla_{a_k}\kappa f_{\vtheta}(\vx_j)=\frac{\kappa^2}{\kappa' m}\sum_{k=1}^m\sigma(\vw_k^\T\vx_i)\sigma(\vw_k^\T\vx_j),                             \\
         & G^{[\vw]}_{ij}(\vtheta)=\frac{\kappa'}{m}\sum_{k=1}^m\nabla_{\vw_k}\kappa f_{\vtheta}(\vx_i)\cdot\nabla_{\vw_k} \kappa f_{\vtheta}(\vx_j)=\frac{\kappa^2\kappa'}{m}\sum_{k=1}^m a_k^2\sigma'(\vw_k^\T\vx_i)\sigma'(\vw_k^\T\vx_j)\vx_i\cdot\vx_j, \\
         & \mG=\mG^{[a]}+\mG^{[\vw]}.
    \end{aligned}
\end{equation}
\begin{assump}\label{assump..lambda}
    Suppose that the Gram matrices are strictly positive definite. In other words,
    \begin{equation}
        \lambda:=\min \{\lambda_a,\lambda_{\vw}\}>0,
    \end{equation}
    where
    \begin{equation}
        \lambda_a := \lambda_{\min}\left(\mK^{[a]}\right),\quad \lambda_{\vw} := \lambda_{\min}\left(\mK^{[\vw]}\right).
    \end{equation}
\end{assump}
\begin{assump}\label{assump..gammagamma'}
    Suppose that the following limits exist
    \begin{equation}
        \gamma:=\lim_{m\to\infty}-\frac{\log\kappa}{\log m},\quad        \gamma':=\lim_{m\to\infty}-\frac{\log\kappa'}{\log m}.
    \end{equation}
\end{assump}
\begin{remark}
    We expect that
    \begin{equation}
        \mG^{[a]}(\vtheta^0)  \approx\frac{\kappa^2}{\kappa'}\mK^{[a]},\quad
        \mG^{[\vw]}(\vtheta^0) \approx\kappa^2\kappa'\mK^{[\vw]},
    \end{equation}
    and these will be rigorously achieved in the following proofs. We also remark that $\lambda \leq d$, which will be used in the following proofs.
\end{remark}
\begin{remark}
    When $\gamma\leq\frac{1}{2}$, we consider NNs with non-zero initial parameters and zero initial output, which can be achieved in NNs by applying the AntiSymmetrical Initialization (ASI) trick~\citep{zhang_type_2019}.
\end{remark}

Our main results are as follows.
\begin{theorem}[linear regime]\label{thm..LinearRegime}
    Given $\delta\in(0,1)$ and the sample set $S = {\{(\vx_i, y_i)\}}_{i=1}^n\subset\Omega$ with $\vx_i$'s drawn i.i.d.\ from some unknown distribution $\fD$. Suppose that Assumption~\ref{assump..lambda} and Assumption~\ref{assump..gammagamma'} hold. ASI is used when $\gamma\leq\frac{1}{2}$. Suppose that $\gamma<1$ or $\gamma'>\gamma-1$ and the dynamics \eqref{eq..MainDynamics}--\eqref{eq..MainInitialization} is considered. Then for sufficiently large $m$, with probability at least $1-\delta$ over the choice of $\vtheta^0$, we have
    \begin{enumerate}[label=(\alph*)]
        \item (changes of $\vtheta$ and $\vtheta_{\vw}$)
              \begin{equation}
                  \sup\limits_{t\in[0,+\infty)}\norm{\vtheta_{\vw}(t)-\vtheta_{\vw}^0}_2\leq \sup\limits_{t\in[0,+\infty)}\norm{\vtheta(t)-\vtheta^0}_2\lesssim\frac{1}{\sqrt{m}\kappa}\log m.
              \end{equation}
        \item (linear convergence rate)
              \begin{equation}
                  \RS(\vtheta(t))\leq \exp\left(-\frac{2m\kappa^2\lambda t}{n}\right)\RS(\vtheta^0).
              \end{equation}
    \end{enumerate}
    Moreover, for sufficiently large $m$, with probability at least $1-\delta-2\exp\left(-\frac{C_0m(d+1)}{4C^2_{\psi,1}}\right)$ over the choice of $\vtheta^0$, we have
    \begin{enumerate}[resume*]
        \item (relative change of $\vtheta$)
              \begin{equation}
                  \sup\limits_{t\in[0,+\infty)}\frac{\norm{\vtheta(t)-\vtheta^0}_2}{\norm{\vtheta^0}_2}\lesssim\frac{1}{m\kappa}\log m.
              \end{equation}
              In particular, if $\gamma<1$, $\sup\limits_{t\in[0,+\infty)}\frac{\norm{\vtheta(t)-\vtheta^0}_2}{\norm{\vtheta^0}_2}\ll 1.$
        \item (relative change of $\vtheta_{\vw}$)
              \begin{equation}
                  \sup\limits_{t\in[0,+\infty)}\mathrm{RD}(\vtheta_{\vw}(t))=\sup\limits_{t\in[0,+\infty)}\frac{\norm{\vtheta_{\vw}(t)-\vtheta_{\vw}^0}_2}{\norm{\vtheta_{\vw}^0}_2}
                  \lesssim
                  \left\{
                  \begin{array}{cc}
                      \frac{1}{m\kappa}\log m,       & \gamma<1,         \\
                      \frac{\kappa'}{m\kappa}\log m, & \gamma'>\gamma-1.
                  \end{array}\right.
              \end{equation}
              In particular, if either $\gamma<1$ or $\gamma'>\gamma-1$, $\sup\limits_{t\in[0,+\infty)}\mathrm{RD}(\vtheta_{\vw}(t))=\sup\limits_{t\in[0,+\infty)}\frac{\norm{\vtheta_{\vw}(t)-\vtheta_{\vw}^0}_2}{\norm{\vtheta_{\vw}^0}_2}\ll1$.
    \end{enumerate}
\end{theorem}
\begin{remark}
    In the regions $\gamma<1$ or $\gamma'>\gamma-1$,
    Theorem \ref{thm..LinearRegime} shows that with a high probability over the initialization, the relative changes of $\vw_k$'s are negligible. This implies that the features change only slightly during the whole gradient descent dynamics. Therefore, in this regime and with large width $m$, one can expect the training result to be close to that of some proper linear regression model. Note that the relative change of $\vtheta$ is negligible only in the sub-region $\gamma<1$. For $\gamma\geq 1$ and $\gamma'>\gamma-1$, the relative changes of $a_k$'s can be fairly large, which may lead to unbounded relative change of $\vtheta$. The relative changes of $\vtheta$ and $a_k$'s are also empirically validated in Appendix \ref{sec:relapara}.
\end{remark}

In order to obtain the theorem that characterize the condensed regime, we need further assumption as follows,
\begin{assump}\label{assump..well-trained}
    We assume that, without loss of generality,
    \begin{equation}
        \max\limits_{i\in[n]} y_i \geq \frac{1}{2},
    \end{equation}
    and that the neural network can be well-trained to the empirical risk less than $O(\frac{1}{n})$. More quantitatively, we require that there exists a $T^*>0$ such that
    \begin{equation}
        \RS(\vtheta(T^*))\leq\frac{1}{32n}.
    \end{equation}
\end{assump}
Then we can get the following theorem
\begin{theorem}[condensed regime]\label{thm..CondensedRegime}
    The sample set $S = {\{(\vx_i, y_i)\}}_{i=1}^n\subset\Omega$ with $\vx_i$'s drawn i.i.d.\ from some unknown distribution $\fD$. Suppose that Assumption~\ref{assump..gammagamma'} and Assumption~\ref{assump..well-trained} hold. Suppose that $\gamma>1$ and $\gamma'<\gamma-1$ and the dynamics~\eqref{eq..MainDynamics}--\eqref{eq..MainInitialization} is considered. Then for sufficiently large $m$, with probability at least $1-2\exp\left(-\frac{C_0m(d+1)}{4C_{\psi,1}^2}\right)$ over the choice of $\vtheta^0$, we have
    \begin{equation}
        \sup\limits_{t\in[0,+\infty)}\mathrm{RD}(\vtheta_{\vw}(t))=\sup\limits_{t\in[0,+\infty)}\frac{\norm{\vtheta_{\vw}(t)-\vtheta_{\vw}^0}_2}{\norm{\vtheta_{\vw}^0}_2}\gg 1.
    \end{equation}
\end{theorem}
To end this section, we provide a sketch of the proofs for the main theorems. In particular, two schematic diagrams \ref{fig:sketch1} and \ref{fig:sketch2} are provided for the proofs of Theorem \ref{thm..LinearRegime} (Theorem 1*) and Theorem \ref{thm..CondensedRegime} (Theorem 2*), respectively, since they are proved in totally different ways. For Theorem \ref{thm..LinearRegime}, we first establish bounds and concentration inequalities for initial parameters. Then the lower bound for the minimal eigenvalue of initial Gram matrix is obtained, which leads to a local in time linear convergence result for the empirical risk. Finally, for sufficiently wide neural networks, we show that the previous estimate is essentially global in time. We remark that for different $(\gamma,\gamma')$'s, the details are quite different in the proofs of Theorem \ref{thm..LinearRegime}, which causes the two branches shown in Figure \ref{fig:sketch1}. For Theorem \ref{thm..CondensedRegime}, as shown in Figure \ref{fig:sketch2}, the schematic diagram of the proof is short and straightforward, thanks to a key observation of the neuron-wise estimate, i.e., Proposition \ref{prop..a-w-est}.

\begin{figure}
    \centering
    \includegraphics[scale=0.6]{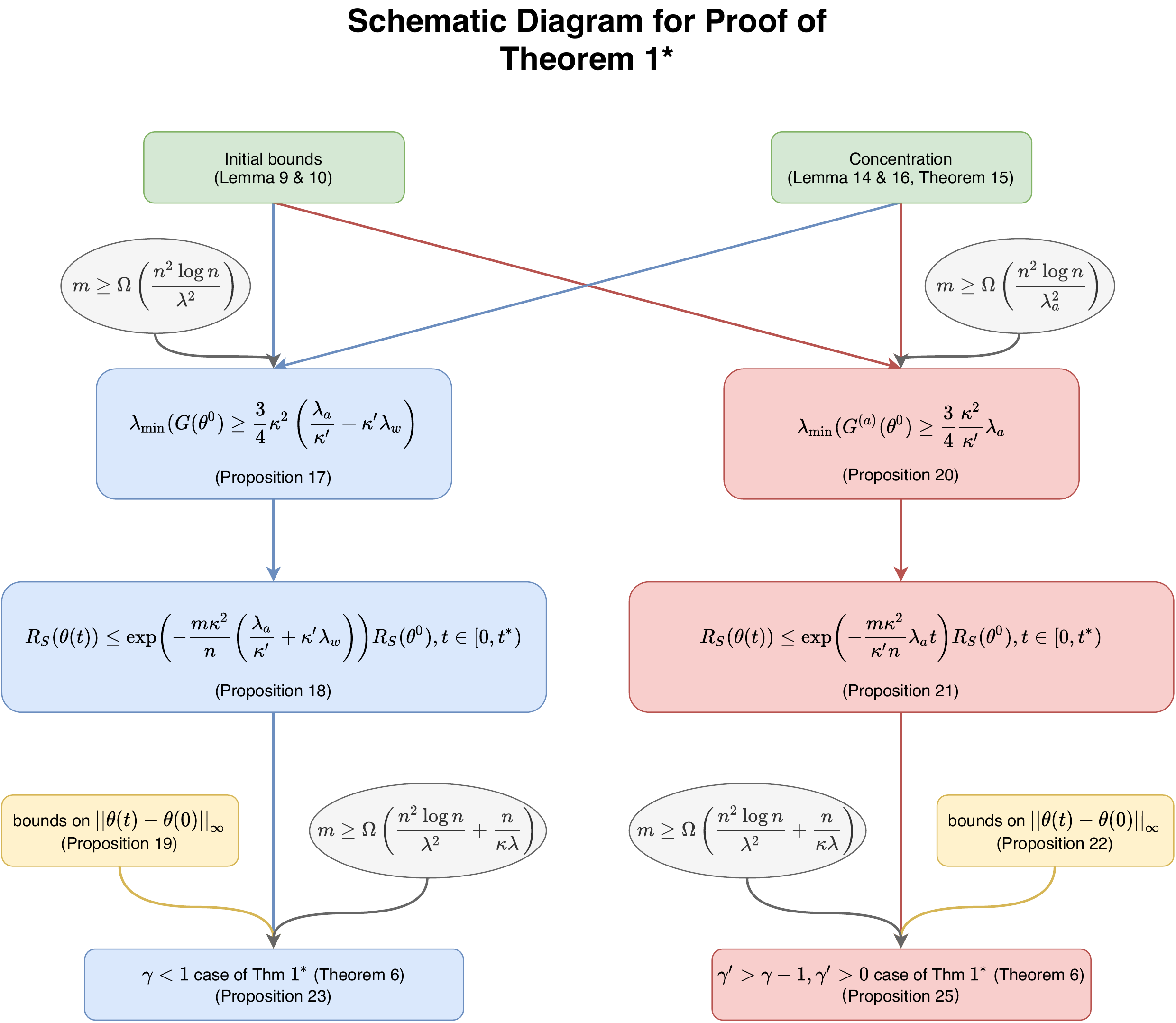}
    \caption{Sketch of proof for Theorem 1*.}
    \label{fig:sketch1}
\end{figure}

\begin{figure}
    \centering
    \includegraphics[scale=0.6]{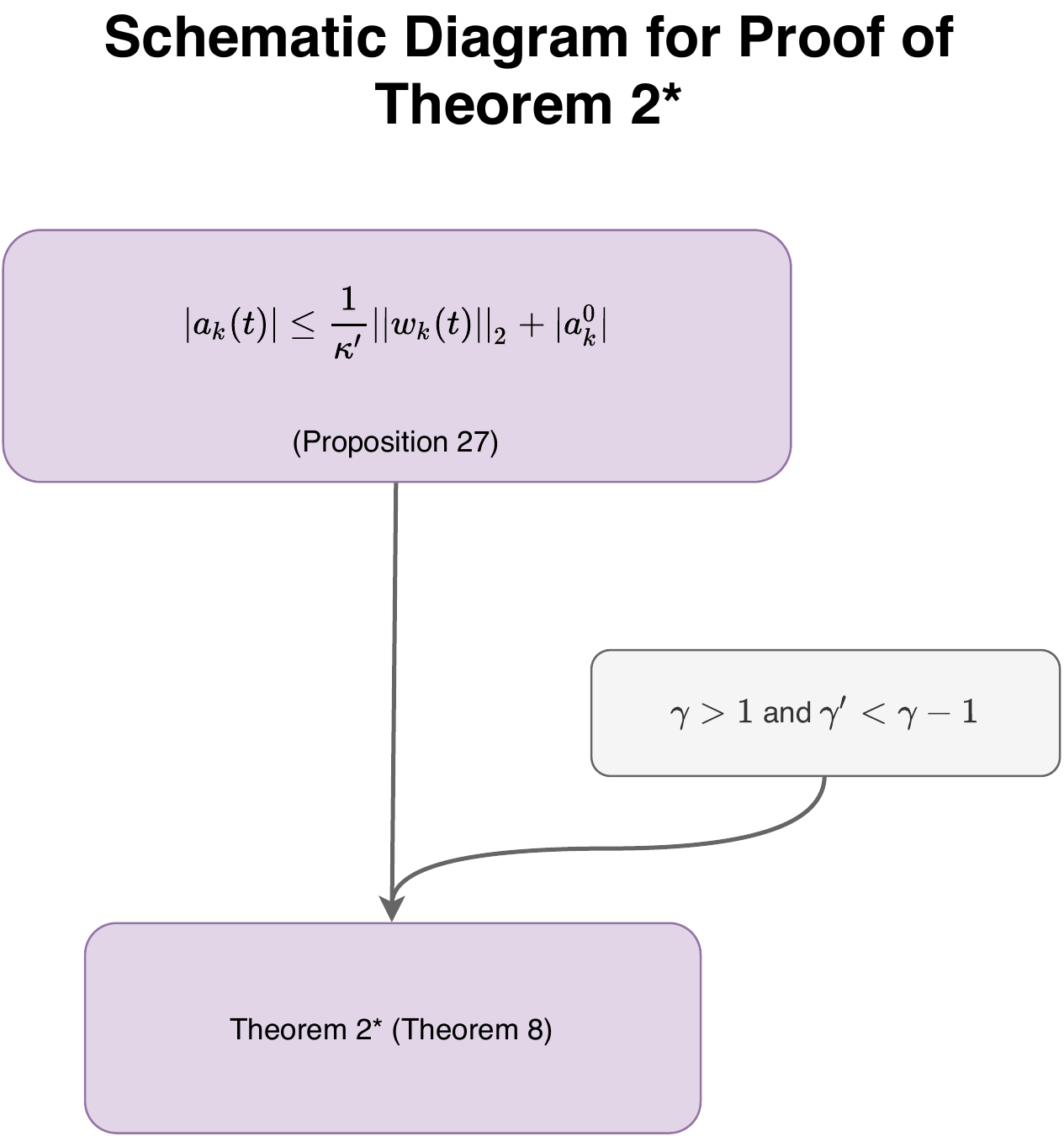}
    \caption{Sketch of proof for Theorem 2*.}
    \label{fig:sketch2}
\end{figure}

\section{Conclusions and discussion}

In this paper, we characterized the linear, critical, and condensed regimes with distinctive features and draw the phase diagram for the two-layer ReLU NN at the infinite-width limit. We experimentally demonstrate and theoretically prove the transition across the boundary (critical regime) in the phase diagram. Through experiments, we further identify the condensation as the signature behavior in the condensed regime of very strong nonlinearity.

A phase diagram serves as a map that guides the future research. In our phase diagram for two-layer ReLU NNs, the linear regimes is very well understood both theoretically and experimentally. However, the critical and condensed regimes are still largely not understood from both experimental and theoretical perspectives. The following problems for these regimes requires further studies: (i) whether the dynamics always converges to a global minimizer; (ii) what is the convergence rate; (iii) what is the mechanism of condensation; (iv) how to characterize the implicit regularization of condensation.

Our phase diagram is obtained specifically for the ReLU activation, however, our methodology and thus obtained regime characterization can be naturally extended to more general activations, which is an immediate next step of this work. In addition, how to characterize the effect of other hyperparameters, e.g., choice of optimization method, learning rate, regularization techniques, and etc., to the NN training dynamics requires future studies. Other important future problems include drawing the phase diagram for NNs of three or more layers or for convolutional networks.

In analogy to statistical mechanics, a clean regime separation may be only possible at the infinite width limit, which is not realistic in practice. Nevertheless, rich insight about a finite size system often can be derived from the analysis at the limit, which is usually much easier. Therefore, we believe it is an important task to systematically draw such phase diagrams for NNs of different structures through a combination of theoretical and experimental approaches as demonstrated in this work. These phase diagrams can be continuously refined and provides a clear pathway to open the black box of deep learning.

\acks{This work is sponsored by National Key R\&D Program of China (2019YFA0709503) (Z. X.), Shanghai Sailing Program (Z. X.).}

\appendix
\section{Technical lemmas}
This section collects some technical lemmas and propositions. For convenience, we define the two quantities
\begin{equation}
    \alpha(t):=\max\limits_{k\in[m],s\in[0,t]}\abs{a_k(s)}, \quad \omega(t):=\max\limits_{k\in[m],s\in[0,t]}\norm{\vw_k(s)}_{\infty}.
\end{equation}
\begin{lemma}[bounds of initial parameters]\label{lem..InitialParameter}
    Given $\delta\in(0,1)$ and the sample set $S = {\{(\vx_i, y_i)\}}_{i=1}^n\subset\Omega$ with $\vx_i$'s drawn i.i.d.\ from some unknown distribution $\fD$. Suppose that Assumption~\ref{assump..lambda} holds. We have with probability at least $1-\delta$ over the choice of $\vtheta^0$
    \begin{equation}\label{eqn:lem1}
        \max\limits_{k\in[m]}\left\{\abs{a_k^0},\;\norm{\vw_k^0}_{\infty}\right \}\leq \sqrt{2\log\frac{2m(d+1)}{\delta}},
    \end{equation}
\end{lemma}
\begin{proof}
    If $\rX \sim N(0, 1)$, then $\Prob(\abs{\rX} > \eps) \leq 2\E^{-\frac{1}{2}\eps^2}$ for all $\eps > 0$. Since $a^0_k\sim N(0,1)$, ${(w_k^0)}_{\alpha}\sim N(0,1)$ for $k=1, 2, \ldots, m,\; \alpha =1,\ldots,d$ and they are all independent, by setting
    \begin{equation*}
        \eps = \sqrt{2\log\frac{2m(d+1)}{\delta}},
    \end{equation*}
    one can obtain
    \begin{equation*}
        \begin{aligned}
            \Prob\left(\max\limits_{k\in[m]}\left\{\abs{a_k^0},\; \norm{\vw_k^0}_{\infty}\right \}>\eps\right)
             & = \Prob\left(\max\limits_{k\in[m],\alpha\in[d]}\left\{\abs{a_k^0},\; \abs{{(w^0_k)}_{\alpha}}\right \}>\eps\right)                                               \\
             & = \Prob\left(\bigcup\limits_{k=1}^m\left(\abs{a_k^0}>\eps\right)\bigcup\left(\bigcup\limits_{\alpha=1}^d\left(\abs{{(w_k^0)}_{\alpha}}>\eps\right)\right)\right) \\
             & \leq \sum_{k=1}^m \Prob\left(\abs{a_k^0}>\eps\right) + \sum_{k=1}^m\sum_{\alpha=1}^d \Prob\left(\abs{{(w^0_k)}_{\alpha}}>\eps\right)                             \\
             & \leq 2m \E^{-\frac{1}{2}\eps^2} + 2md \E^{-\frac{1}{2}\eps^2}                                                                                                    \\
             & = 2m(d+1)\E^{-\frac{1}{2}\eps^2}                                                                                                                                 \\
             & = \delta.
        \end{aligned}
    \end{equation*}
\end{proof}
\begin{lemma}[bound of initial empirical risk]
    Given $\delta\in(0,1)$ and the sample set $S = {\{(\vx_i, y_i)\}}_{i=1}^n\subset\Omega$ with $\vx_i$'s drawn i.i.d.\ from some unknown distribution $\fD$. Suppose that Assumption~\ref{assump..lambda} holds. We have with probability at least $1-\delta$ over the choice of $\vtheta^0$
    \begin{equation}
        \RS(\vtheta^0)\leq\frac{1}{2}\left[1 + 2d\left(\log\frac{4m(d+1)}{\delta}\right)\left(2+3\sqrt{2\log(8/\delta)}\right)\kappa\sqrt{m}\right]^2.
    \end{equation}
\end{lemma}
\begin{proof}
    Let
    \begin{equation}
        \fG = \{g_{\vx}(a,\vw) \mid g_{\vx}(a,\vw):=a\sigma(\vw^\T\vx),\vx\in\Omega \}.
    \end{equation}
    Lemma~\ref{lem..InitialParameter} implies that with probability at least $1-\delta/2$ over the choice of $\vtheta^0$, we have
    \begin{equation*}
        \abs{g_{\vx}(a^0_k,\vw^0_k)}\leq d\abs{a_k^0}\norm{\vw^0_k}\leq 2d\log\frac{4m(d+1)}{\delta}
    \end{equation*}
    Then
    \begin{equation*}
        \begin{aligned}
            \frac{1}{ m}\sup_{\vx\in\Omega}\abs{f_{\vtheta^0}(\vx)}
             & =\sup_{\vx\in\Omega}\left\lvert \frac{1}{m}\sum_{k=1}^m a_k^0\sigma(\vw_k^0\cdot\vx)-\Exp_{(a,\vw)}a\sigma(\vw^\T\vx)\right\rvert \\
             & \leq 2\Rad_{\vtheta^0}(\fG) + 6d \left(\log\frac{4m(d+1)}{\delta}\right)\sqrt{\frac{2\log(8/\delta)}{m}}.
        \end{aligned}
    \end{equation*}
    The Rademacher complexity can be estimated by
    \begin{equation*}
        \begin{aligned}
            \Rad_{\vtheta^0}(\fG)
             & =\frac{1}{m}\Exp_{\tau}\left[\sup_{\vx\in\Omega}\sum_{k=1}^m\tau_ka^0_k\sigma(\vw^0_k\cdot\vx)\right]   \\
             & \leq\frac{1}{m}\sqrt{2\log\frac{4m(d+1)}{\delta}}\Exp_{\tau}\left[\sup_{\vx\in\Omega}\sum_{k=1}^m\tau_k
            \vw_k^0\cdot\vx\right]                                                                                     \\
             & \leq\sqrt{2\log\frac{4m(d+1)}{\delta}}\sqrt{2d\log\frac{4m(d+1)}{\delta}}\frac{\sqrt{d}}{\sqrt{m}}      \\
             & =\frac{2d\log\frac{4m(d+1)}{\delta}}{\sqrt{m}}.
        \end{aligned}
    \end{equation*}
    Therefore
    \begin{equation*}
        \sup_{\vx\in\Omega}\abs{f_{\vtheta^0}(\vx)}\leq 2d\left(\log\frac{4m(d+1)}{\delta}\right)(2+3\sqrt{2\log(8/\delta)}\sqrt{m}),
    \end{equation*}
    and
    \begin{equation*}
        \begin{aligned}
            \RS(\vtheta^0)
             & \leq\frac{1}{2n}\sum_{i=1}^n{\left(1+\kappa\abs{f_{\vtheta}(\vx_i)}\right)}^2                                                    \\
             & \leq\frac{1}{2}{\left[1+2d\left(\log\frac{4m(d+1)}{\delta}\right)\left(2+3\sqrt{2\log(8/\delta)}\kappa\sqrt{m}\right)\right]}^2.
        \end{aligned}
    \end{equation*}
\end{proof}
\begin{remark}
    If $\gamma>\frac{1}{2}$, then $\kappa=o(\frac{1}{\sqrt{m}\log m})$ and $\RS(\vtheta^0)=O(1)$. One can use ASI trick~\citep{zhang_type_2019} to guarantee $\RS(\vtheta^0)\leq\frac{1}{2}$ for any $\kappa$.
\end{remark}
Next we introduce the sub-exponential norm of a random variable and the sub-exponential Bernstein's inequality.
\begin{definition}[sub-exponential norm~\citep{vershynin2018high}]
    The sub-exponential norm of a random variable $\rX$ is defined as
    \begin{equation}
        \norm{\rX}_{\psi_1} := \inf\{s>0 \mid \Exp_{\rX}[\E^{\abs{\rX}/s}]\leq 2\}.
    \end{equation}
    In particular, we denote the sub-exponential norm of a $\chi^2(d)$ random variable $\rX$ by $C_{\psi,d}:=\norm{\rX}_{\psi_1}$. Here the $\chi^2$ distribution with $d$ degrees of freedom has the probability density function
    \begin{equation*}
        f_{\rX}(z)=\frac{1}{2^{d/2}\Gamma(d/2)}z^{d/2-1}\E^{-z/2}.
    \end{equation*}
\end{definition}
\begin{remark}
    Note that
    \begin{align*}
        \Exp_{\rX\sim\chi^2(d)}\E^{\abs{\rX}/2}
         & =\int_0^{+\infty}\frac{1}{2^{d/2}\Gamma(d/2)}z^{d/2-1}\diff{z}=+\infty,                            \\
        \lim_{s\to+\infty}\Exp_{\rX\sim\chi^2(d)}\E^{\abs{\rX}/s}
         & =\lim_{s\to+\infty}\int_0^{+\infty}\frac{1}{2^{d/2}\Gamma(d/2)}z^{d/2-1}\E^{-z/2+z/s}\diff{z}=1<2.
    \end{align*}
    These imply that $2\leq C_{\psi,d}<+\infty$.
\end{remark}
\begin{lemma}\label{lem..subexponentialnorm}
    Suppose that $\vw\sim N(0,\mI_d)$, $a\sim N(0,1)$ and given $\vx_i,\vx_j \in\Omega$. Then we have
    \begin{enumerate}[(i)]
        \item if $\rX:=\sigma(\vw^\T\vx_i)\sigma(\vx\cdot\vx_j)$, then $\norm{\rX}_{\psi_1}\leq dC_{\psi,d}$.
        \item if $\rX:=a^2\sigma'(\vw^\T\vx_i)\sigma'(\vw^\T\vx_j)\vx_i\cdot\vx_j$, then $\norm{\rX}_{\psi_1}\leq dC_{\psi,d}$.
    \end{enumerate}
\end{lemma}
\begin{proof}
    Let
    $\rZ:=\norm{\vw}^2_2=\chi^2(d)$.\\
    (i)   $\abs{\rX}\leq d\norm{\vw}^2_2=d\rZ$ and
    \begin{equation*}
        \begin{aligned}
            \norm{\rX}_{\psi_1}
             & =\inf\{s>0\mid\Exp_{\rX}\exp(\abs{\rX}/s)\leq 2\}                                               \\
             & =\inf\{s>0\mid\Exp_{\vw}\exp\left(\abs{\sigma(\vw^\T\vx_i)\sigma(\vw^\T\vx_j)}/s\right)\leq 2\} \\
             & \leq\inf\{s>0\mid\Exp_{\vw}\exp(d\norm{\vw}^2_2/s)\leq 2\}                                      \\
             & =\inf\{s>0\mid\Exp_{\rZ}\exp(d\abs{\rZ}/s)\leq 2\}                                              \\
             & =d\inf\{s>0\mid\Exp_{\rZ}\exp(\abs{\rZ}/s)\leq 2\}                                              \\
             & =d\norm{\chi^2(d)}_{\psi_1}                                                                     \\
             & \leq dC_{\psi,d}.
        \end{aligned}
    \end{equation*}
    (ii) $\abs{\rX}\leq d\abs{a}^2\leq d\rZ$ and $\norm{\rX}_{\psi_1}\leq dC_{\psi,d}$.
\end{proof}
\begin{theorem}[sub-exponential Bernstein's inequality~\citep{vershynin2018high}]\label{thm:sub_exp}
    Suppose that $\rX_1,\ldots,\rX_m$ are i.i.d.\ sub-exponential random variables with $\Exp\rX_1=\mu$, then for any $s\geq 0$ we have
    \begin{equation}
        \Prob\left(\Abs{\frac{1}{m}\sum_{k=1}^m\rX_k-\mu}\geq s\right)\leq 2\exp\left(-C_0 m \min\left(\frac{s^2}{\norm{\rX_1}^2_{\psi_1}},\frac{s}{\norm{\rX_1}_{\psi_1}}\right)\right),
    \end{equation}
    where $C_0$ is an absolute constant.
\end{theorem}
\begin{proposition}[norm of initial parameters]\label{prop..InitialThetaNorm}
    Given $\delta\in(0,1)$ and the sample set $S = {\{(\vx_i, y_i)\}}_{i=1}^n\subset\Omega$ with $\vx_i$'s drawn i.i.d.\ from some unknown distribution $\fD$. Suppose that Assumption~\ref{assump..lambda} holds. We have with probability at least $1-2\exp\left(-\frac{C_0m(d+1)}{4C^2_{\psi,1}}\right)$ over the choice of $\vtheta^0$
    \begin{align}
        \sqrt{\frac{m(d+1)}{2}}
         & \leq \norm{\vtheta^0}_2\leq \sqrt{\frac{3m(d+1)}{2}},\label{eq..InitialThetaNorm}     \\
        \sqrt{\frac{md}{2}}
         & \leq \norm{\vtheta_{\vw}^0}_2\leq \sqrt{\frac{3md}{2}}.\label{eq..InitialThetavwNorm} \\
        \sqrt{\frac{m}{2}}
         & \leq \norm{\vtheta_{a}^0}_2\leq \sqrt{\frac{3m}{2}}.\label{eq..InitialThetaaNorm}
    \end{align}
\end{proposition}
\begin{proof}
    Let $\rX_1, \ldots, \rX_{m(d+1)}$ be the squares of the entries of $\vtheta^0$, which are drawn i.i.d.\ from $\chi^2(1)$. The latter is sub-exponential and $\Exp\rX_k=1$. Then by Theorem~\ref{thm:sub_exp}
    \begin{equation*}
        \Prob\left(\Abs{\frac{1}{m(d+1)}\sum_{k=1}^{m(d+1)}\rX_k-1}\geq s\right)\leq 2\exp\left(-C_0 m(d+1) \min\left(\frac{s^2}{C^2_{\psi,1}},\frac{s}{C_{\psi,1}}\right)\right).
    \end{equation*}
    Setting $s=\frac{1}{2}$, we have $\frac{s}{C_{\psi,1}}\leq \frac{1/2}{2}< 1$ and
    \begin{equation*}
        \begin{aligned}
            \Prob\left(\frac{1}{2}\leq\frac{1}{m(d+1)}\sum_{k=1}^{m(d+1)}\rX_k\leq\frac{3}{2}\right)
             & \leq 2\exp\left(-C_0 m(d+1) \min\left(\frac{1}{4C^2_{\psi,1}},\frac{1}{2C_{\psi,1}}\right)\right) \\
             & =2\exp\left(-\frac{C_0m(d+1)}{4C^2_{\psi,1}}\right).
        \end{aligned}
    \end{equation*}
    Therefore, with probability at least $1-2\exp\left(-\frac{C_0m(d+1)}{4C^2_{\psi,1}}\right)$ over the choice of $\vtheta^0$,
    \begin{equation*}
        \frac{1}{2}\leq\frac{1}{m(d+1)}\sum_{k=1}^{m(d+1)}\rX_k=\frac{1}{m(d+1)}\norm{\vtheta^0}_2^2\leq\frac{3}{2}.
    \end{equation*}
    In other words, \eqref{eq..InitialThetaNorm} holds. The proofs of \eqref{eq..InitialThetavwNorm} and \eqref{eq..InitialThetaaNorm} are similar.
\end{proof}
\begin{proposition}[minimal eigenvalue of Gram matrix $\mG$ at initial]\label{prop:lambda_min}
    Given $\delta\in(0,1)$ and the sample set $S = {\{(\vx_i, y_i)\}}_{i=1}^n\subset\Omega$ with $\vx_i$'s drawn i.i.d.\ from some unknown distribution $\fD$. Suppose that Assumption~\ref{assump..lambda} holds. If $m\geq\frac{16n^2d^2C_{\psi,d}}{C_0\lambda^2}\log\frac{4n^2}{\delta}$ then with probability at least $1-\delta$ over the choice of $\vtheta^0$, we have
    \begin{equation}
        \lambda_{\min}\left(\mG(\vtheta^0)\right)\geq\frac{3}{4}\kappa^2\left(\frac{1}{\kappa'}\lambda_a+\kappa'\lambda_{\vw}\right).
    \end{equation}
\end{proposition}
\begin{proof}
    For any $\eps > 0$, we define
    \begin{equation*}
        \begin{aligned}
            \Omega_{ij}^{[a]}   & :=\left\{\vtheta^0 \mid \left\lvert \frac{\kappa'}{\kappa^2}G^{[a]}_{ij}(\vtheta^0) - K^{[a]}_{ij}\right\rvert \leq \frac{\eps}{n} \right\},       \\
            \Omega_{ij}^{[\vw]} & := \left\{\vtheta^0 \mid \left\lvert \frac{1}{\kappa^2\kappa'}G^{[\vw]}_{ij}(\vtheta^0) - K^{[\vw]}_{ij}\right\rvert \leq \frac{\eps}{n} \right\}.
        \end{aligned}
    \end{equation*}
    By Theorem~\ref{thm:sub_exp} and Lemma~\ref{lem..subexponentialnorm}, if $\frac{\eps}{ndC_{\psi,d}}\leq 1$, then
    \begin{equation*}
        \begin{aligned}
            \Prob(\Omega^{[a]}_{ij})   & \geq 1-2\exp\left(-\frac{mC_0\eps^2}{n^2d^2C_{\psi,d}^2}\right), \\
            \Prob(\Omega^{(\vw)}_{ij}) & \geq 1-2\exp\left(-\frac{mC_0\eps^2}{n^2d^2C_{\psi,d}^2}\right),
        \end{aligned}
    \end{equation*}
    so with probability at least $\left[1-2\exp\left(-\frac{mC_0\eps^2}{n^2d^2C_{\psi,d}^2}\right)\right]^{2n^2} \geq 1-4n^2\exp\left(-\frac{mC_0\eps^2}{n^2d^2C_{\psi,d}^2}\right)$ over the choice of $\vtheta^0$, we have
    \begin{equation*}
        \begin{aligned}
             & \Norm{\frac{\kappa'}{\kappa^2}\mG^{[a]}(\vtheta^0) - \mK^{[a]}}_\mathrm{F} \leq \eps,      \\
             & \Norm{\frac{1}{\kappa^2\kappa'}\mG^{[\vw]}(\vtheta^0) - \mK^{[\vw]}}_\mathrm{F} \leq \eps,
        \end{aligned}
    \end{equation*}
    Hence by taking $\eps=\lambda/4$, that is, $\delta=4n^2\exp\left(-\frac{mC_0\lambda^2}{16n^2d^2C_{\psi,d}^2}\right)$
    \begin{equation*}
        \begin{aligned}
            \lambda_{\min}\left(\mG(\vtheta^0)\right)
             & \geq\lambda_{\min}\left(\mG^{[a]}(\vtheta^0)\right) + \lambda_{\min}\left(\mG^{[\vw]}(\vtheta^0)\right)                                   \\
             & \geq\frac{\kappa^2}{\kappa'}\lambda_a-\frac{\kappa^2}{\kappa'} \Norm{\frac{\kappa'}{\kappa^2}\mG^{[a]}(\vtheta^0) - \mK^{[a]}}_\mathrm{F} \\
             & \quad + \kappa^2\kappa'\lambda_{\vw} - \kappa^2\kappa'\Norm{\frac{1}{\kappa^2\kappa'}\mG^{[\vw]}(\vtheta^0) - \mK^{[\vw]}}_\mathrm{F}     \\
             & \geq\frac{\kappa^2}{\kappa'}(\lambda_a-\eps)+\kappa^2\kappa'(\lambda_{\vw}-\eps)                                                          \\
             & \geq\frac{3}{4}\kappa^2\left(\frac{1}{\kappa'}\lambda_a+\kappa'\lambda_{\vw}\right).
        \end{aligned}
    \end{equation*}
    We remark that for $\eps=\lambda/4$, we have $\frac{\eps}{ndC_{\psi,d}}=\frac{\lambda}{8nd}\leq \frac{1}{8}< 1$.
\end{proof}

\noindent In the following we denote
\begin{equation}
    t^* = \inf\{t \mid \vtheta(t)\notin \mathcal{N}(\vtheta^0)\},
\end{equation}
where
\begin{equation}
    \mathcal{N}(\vtheta^0) := \left\{\theta \mid \norm{\mG(\vtheta) - \mG(\vtheta^0)}_\mathrm{F}\leq \frac{1}{4}\kappa^2\left(\frac{1}{\kappa'}\lambda_a+\kappa'\lambda_{\vw}\right)\right\}.
\end{equation}
Then we have the following lemma.
\begin{proposition}[local in time exponential decay of $R_S$, $\vtheta$-lazy training]\label{prop:exp_RS}
    Given $\delta\in(0,1)$ and the sample set $S = {\{(\vx_i, y_i)\}}_{i=1}^n\subset\Omega$ with $\vx_i$'s drawn i.i.d.\ from some unknown distribution $\fD$. Suppose that Assumption~\ref{assump..lambda} holds. If $m\geq\frac{16n^2d^2C_{\psi,d}^2}{\lambda^2C_0}\log\frac{4n^2}{\delta}$, then with probability at least $1-\delta$ over the choice of $\vtheta^0$, we have for any $t\in[0, t^*)$
    \begin{equation}
        \RS(\vtheta(t)) \leq \exp\left(-\frac{m\kappa^2}{n}\left(\frac{1}{\kappa'}\lambda_a+\kappa'\lambda_{\vw}\right)\right)\RS(\vtheta^0).
    \end{equation}
\end{proposition}
\begin{proof}
    Prop.~\ref{prop:lambda_min} implies that for any $\delta\in(0,1)$, with probability at least $1-\delta$ over the choice of $\vtheta^0$ and for any $\vtheta\in\mathcal{N}(\vtheta^0)$, we have
    \begin{equation*}
        \begin{aligned}
            \lambda_{\min}\left(\mG(\vtheta)\right)
             & \geq \lambda_{\min}\left(\mG(\vtheta^0)\right) - \norm{\mG(\vtheta) - \mG(\vtheta^0)}_\mathrm{F}                                                                       \\
             & \geq \frac{3}{4}\kappa^2\left(\frac{1}{\kappa'}\lambda_a+\kappa'\lambda_{\vw}\right) - \frac{1}{4}\kappa^2\left(\frac{1}{\kappa'}\lambda_a+\kappa'\lambda_{\vw}\right) \\
             & = \frac{1}{2}\kappa^2\left(\frac{1}{\kappa'}\lambda_a+\kappa'\lambda_{\vw}\right).
        \end{aligned}
    \end{equation*}
    Note that
    \begin{equation*}
        G_{ij} = G_{ij}^{[a]}+G^{[\vw]}_{ij}=\frac{\kappa^2}{\kappa'm}\sum_{k=1}^m\nabla_{a_k}f_{\vtheta}(\vx_i)\cdot\nabla_{a_k}f_{\vtheta}(\vx_j)+\frac{\kappa^2\kappa'}{m}\sum_{k=1}^m\nabla_{\vw_k}f_{\vtheta}(\vx_i)\cdot\nabla_{\vw_k}f_{\vtheta}(\vx_j),
    \end{equation*}
    and
    \begin{equation*}
        \begin{aligned}
            \nabla_{a_k}\RS(\vtheta)   & =\frac{1}{n}\sum_{i=1}^{n}e_i\kappa\nabla_{a_k}f_{\vtheta}(\vx_i),    \\
            \nabla_{\vw_k}\RS(\vtheta) & = \frac{1}{n}\sum_{i=1}^{n}e_i\kappa\nabla_{\vw_k}f_{\vtheta}(\vx_i).
        \end{aligned}
    \end{equation*}
    Thus
    \begin{equation*}
        \frac{m}{n^2}\ve^{\T}\mG\ve=\frac{1}{\kappa'}\sum_{k=1}^m\nabla_{a_k}\RS(\vtheta)\cdot\nabla_{a_k}\RS(\vtheta)+\kappa'\sum_{k=1}^m\nabla_{\vw_k}\RS(\vtheta)\cdot\nabla_{\vw_k}\RS(\vtheta).
    \end{equation*}
    Then finally we get
    \begin{equation*}
        \begin{aligned}
            \frac{\D}{\D t}\RS(\vtheta(t))
             & =-\left(\frac{1}{\kappa'}\sum_{k=1}^m\nabla_{a_k}\RS(\vtheta)\cdot\nabla_{a_k}\RS(\vtheta)+\kappa'\sum_{k=1}^m\nabla_{\vw_k}\RS(\vtheta)\cdot\nabla_{\vw_k}\RS(\vtheta)\right), \\
             & =-\frac{m}{n^2}\ve^{\T}\mG\ve,                                                                                                                                                  \\
             & \leq -\frac{2m}{n}\lambda_{\min}\left(\mG(\vtheta(t))\right)\RS(\vtheta(t))                                                                                                     \\
             & \leq -\frac{m\kappa^2}{n}\left(\frac{1}{\kappa'}\lambda_a+\kappa'\lambda_{\vw}\right)\RS(\vtheta(t)),
        \end{aligned}
    \end{equation*}
    and an integration yields the result.
\end{proof}

\begin{proposition}[bounds on the change of parameters, $\vtheta$-lazy training]\label{prop:a_w}
    Given $\delta\in(0,1)$ and the sample set $S = {\{(\vx_i, y_i)\}}_{i=1}^n\subset\Omega$ with $\vx_i$'s drawn i.i.d.\ from some unknown distribution $\fD$. Suppose that Assumption~\ref{assump..lambda} holds. If $m\geq\max\left\{\frac{16n^2d^2C_{\psi,d}^2}{\lambda^2C_0}\log\frac{8n^2}{\delta}, \frac{4\sqrt{2d}n\sqrt{\RS(\vtheta^0)}}{\kappa\left(\lambda_a/\kappa'+\kappa'\lambda_{\vw}\right)}\right\}$, then with probability at least $1-\delta$ over the choice of $\vtheta^0$, for any $t\in[0, t^\ast)$ and any $k\in [m]$,
    \begin{align}
        \max\limits_{k\in[m]}\abs{a_k(t) - a_k(0)}
         & \leq 2\max\left\{\frac{1}{\kappa'},1\right\}\sqrt{2\log\frac{4m(d+1)}{\delta}}p, \\
        \max\limits_{k\in[m]}\norm{\vw_k(t) - \vw_k(0)}_{\infty}
         & \leq 2\max\{
        \kappa',1\}\sqrt{2\log\frac{4m(d+1)}{\delta}}p,
    \end{align}
    and
    \begin{equation}
        \max\limits_{k\in[m]}\{\abs{a_k(0)},\;\norm{\vw_k(0)}_{\infty}\} \leq \sqrt{2\log\frac{4m(d+1)}{\delta}},
    \end{equation}
    where $p := \frac{2\sqrt{2}dn\sqrt{\RS(\vtheta^0)}}{m\kappa\left(\lambda_a/\kappa'+\kappa'\lambda_{\vw}\right)}$.
\end{proposition}
\begin{proof}
    Since
    \begin{equation*}
        \alpha(t)=\max\limits_{k\in[m],s\in[0,t]}|a_k(s)|, \quad \omega(t)=\max\limits_{k\in[m],s\in[0,t]}\norm{\vw_k(s)}_{\infty},
    \end{equation*}
    we obtain
    \begin{equation*}
        \begin{aligned}
            \abs{\nabla_{a_k}\RS}^2    & =\left\lvert\frac{1}{n}\sum_{i=1}^n e_i\kappa\sigma(\vw_k^\T\vx_i)\right\rvert^2\leq 2\norm{\vw_k}^2_1\kappa^2\RS(\vtheta)\leq 2d^2(\omega(t))^2\kappa^2\RS(\vtheta),            \\
            \norm{\nabla_{\vw_k}\RS}^2 & =\left\lVert\frac{1}{n}\sum_{i=1}^n e_i\kappa a_k\sigma'(\vw_k^\T\vx_i)\vx_i\right\rVert^2_{\infty}\leq 2\abs{a_k}^2\kappa^2\RS(\vtheta)\leq 2(\alpha(t))^2\kappa^2\RS(\vtheta).
        \end{aligned}
    \end{equation*}
    By Prop.~\ref{prop:exp_RS}, we have if $m\geq \frac{16n^2d^2C_{\psi,d}^2}{\lambda^2C_0}\log\frac{8n^2}{\delta}$, then with probability at least $1 - \delta/2$ over the choice of $\vtheta^0$,
    \begin{equation*}
        \begin{aligned}
            \abs{a_k(t) - a_k(0)}
             & \leq\frac{1}{\kappa'}\int_0^t\abs{\nabla_{a_k}\RS(\vtheta(s))}\diff{s}                                                                                                                   \\
             & \leq\frac{\sqrt{2}d\kappa}{\kappa'}\int_{0}^{t} \omega(s)\sqrt{\RS(\vtheta(s))}\diff{s}                                                                                                  \\
             & \leq\frac{\sqrt{2}d\kappa}{\kappa'}\omega(t)\int_{0}^{t}\sqrt{\RS(\vtheta^0)}\exp\left(-\frac{m\kappa^2}{2n}\left(\frac{1}{\kappa'}\lambda_a+\kappa'\lambda_{\vw}\right)s\right)\diff{s} \\
             & \leq \frac{2\sqrt{2}dn\sqrt{\RS(\vtheta^0)}}{m\kappa\kappa'\left(\lambda^{[a]}_S/\kappa'+\kappa'\lambda_{\vw}\right)}\omega(t)                                                           \\
             & =\frac{p}{\kappa'}\omega(t).
        \end{aligned}
    \end{equation*}
    On the other hand,
    \begin{equation*}
        \begin{aligned}
            \norm{\vw_k(t) - \vw_k(0)}_{\infty}
             & \leq \kappa'\int_{0}^{t} \norm{\nabla_{\vw_k}\RS(\vtheta(s))}_{\infty}\diff{s}                                                                                                     \\
             & \leq \sqrt{2}\kappa\kappa' \int_{0}^t \alpha(s)\sqrt{\RS(\vtheta(s))} \diff{s}                                                                                                     \\
             & \leq \sqrt{2}\kappa\kappa' \alpha(t) \int_{0}^{t} \sqrt{\RS(\vtheta^0)}\exp\left(-\frac{m\kappa^2}{2n}\left(\frac{1}{\kappa'}\lambda_a+\kappa'\lambda_{\vw}\right)s\right)\diff{s} \\
             & \leq \frac{2\sqrt{2}n\sqrt{\RS(\vtheta^0)}\kappa'}{m\kappa\left(\lambda^{[a]}_S/\kappa'+\kappa'\lambda_{\vw}\right)}\alpha(t)                                                      \\
             & \leq p\kappa'\alpha(t).
        \end{aligned}
    \end{equation*}
    Thus
    \begin{equation*}
        \begin{aligned}
            \alpha(t) & \leq\alpha(0)+p\omega(t)\frac{1}{\kappa'}, \\
            \omega(t) & \leq\omega(0)+p\alpha(t)\kappa'.
        \end{aligned}
    \end{equation*}
    By Lemma~\ref{lem..InitialParameter}, we have with probability at least $1 - \delta/2$ over the choice of $\vtheta^0$,
    \begin{equation*}
        \max\limits_{k\in[m]}\{\abs{a_k(0)},\;\norm{\vw_k(0)}_{\infty}\}\leq \sqrt{2\log\frac{4m(d+1)}{\delta}}.
    \end{equation*}
    If
    \begin{equation*}
        m \geq \frac{4\sqrt{2}dn\sqrt{\RS(\vtheta^0)}}{\kappa\left(\lambda_a/\kappa'+\kappa'\lambda_{\vw}\right)},
    \end{equation*}
    then we have
    \begin{equation*}
        p=\frac{2\sqrt{2}dn\sqrt{\RS(\vtheta^0)}}{m\kappa\left(\lambda_a/\kappa'+\kappa'\lambda_{\vw}\right)}\leq \frac{1}{2}.
    \end{equation*}
    Thus
    \begin{align*}
        \alpha(t) & \leq\alpha(0)+\frac{p}{\kappa'}\omega(0)+p^2\alpha(t),          \\
        \alpha(t) & \leq\frac{4}{3}\alpha(0)+\frac{2}{3}\frac{1}{\kappa'}\omega(0).
    \end{align*}
    Therefore
    \begin{equation*}
        \alpha(t)\leq 2\max\left\{1,\frac{1}{\kappa'}\right\}\sqrt{2\log\frac{4m(d+1)}{\delta}}.
    \end{equation*}
    Similarly, one can obtain the estimate of $\omega(t)$ as
    \begin{equation*}
        \omega(t)\leq 2\max\{1,\kappa'\}\sqrt{2\log\frac{4m(d+1)}{\delta}}.
    \end{equation*}
    Finally we have for any $t\in[0, t^*)$ with probability at least $1-\delta$ over the choice of $\vtheta^0$,
    \begin{equation*}
        \begin{aligned}
            \max\limits_{k\in[m]}|a_k(t) - a_k(0)|
             & \leq 2\max\left\{\frac{1}{\kappa'},1\right\}\sqrt{2\log\frac{4m(d+1)}{\delta}}p, \\
            \max\limits_{k\in[m]}\norm{\vw_k(t) - \vw_k(0)}_{\infty}
             & \leq 2\max\{
            \kappa',1\}\sqrt{2\log\frac{4m(d+1)}{\delta}}p,
        \end{aligned}
    \end{equation*}
    which completes the proof.
\end{proof}

To show our main results with $\gamma'>\gamma-1$, we further define
\begin{equation}
    t^*_a=\inf\{t\mid \vtheta(t)\in\fN_a(\vtheta^0)\}, \quad t^*_{\vw}=\inf\{t\mid \vtheta(0)\in\fN_{\vw}(\vtheta^0)\},
\end{equation}
where
\begin{align}
    \fN_a(\vtheta^0)     & :=\left\{\vtheta\mid\norm{\mG^{[a]}(\vtheta)-\mG^{[a]}(\vtheta^0)}_\mathrm{F}\leq\frac{1}{4}\frac{\kappa^2}{\kappa'}\lambda_a\right\}, \\
    \fN_{\vw}(\vtheta^0) & :=\left\{\vtheta\mid\norm{\mG^{[\vw]}(\vtheta)-\mG^{[\vw]}(\vtheta^0)}_\mathrm{F}\leq\frac{1}{4}\kappa^2\kappa'\lambda_{\vw}\right\}.
\end{align}

\begin{proposition}[minimal eigenvalue of Gram matrix $\mG^{[a]}$ at initial]\label{prop:exp_a}
    Given $\delta\in(0,1)$ and the sample set $S = {\{(\vx_i, y_i)\}}_{i=1}^n\subset\Omega$ with $\vx_i$'s drawn i.i.d.\ from some unknown distribution $\fD$. Suppose that Assumption~\ref{assump..lambda} holds. If $m\geq\frac{16n^2d^2C_{\psi,d}^2}{C_0\lambda_a^2}\log\frac{2n^2}{\delta}$, then we have with probability at least $1-\delta$ over the choice of $\vtheta^0$,
    \begin{equation}
        \lambda_{\min}\left(\mG^{[a]}(\vtheta^0)\right)\geq\frac{3}{4}\frac{\kappa^2}{\kappa'}\lambda_a.
    \end{equation}
\end{proposition}
\begin{proof}
    For any $\eps>0$ define
    \begin{equation}
        \Omega_{ij}^{[a]}:=\left\{\vtheta^0\mid\left|\frac{\kappa'}{\kappa^2}G_{ij}^{[a]}(\vtheta^0)-K_{ij}^{[a]}\right|\leq\frac{\eps}{n}\right\}.
    \end{equation}
    By Theorem~\ref{thm:sub_exp} and Lemma~\ref{lem..subexponentialnorm}, if $\frac{\eps}{ndC_{\psi,d}}\leq 1$ then
    \begin{equation*}
        \Prob(\Omega_{ij}^{[a]}) \geq 1-2\exp\left(-\frac{mC_0\eps^2}{n^2d^2C^2_{\psi,d}}\right),
    \end{equation*}
    with probability at least
    \begin{equation*}
        \left[1-2\exp\left(-\frac{mC_0\eps^2}{n^2 d^2 C^2_{\psi,d}}\right)\right]^{n^2}\geq 1-2n^2\exp\left(-\frac{mC_0\eps^2}{n^2 d^2 C^2_{\psi,d}}\right)
    \end{equation*}
    over the choice of $\vtheta^0$, we have
    \begin{equation*}
        \Norm{\frac{\kappa'}{\kappa^2}\mG^{[a]}(\vtheta^0)-\mK^{[a]}}_\mathrm{F}\leq\eps.
    \end{equation*}
    Taking $\eps=\lambda_a/4$, i.e., $\delta=2n^2\exp\left(-\frac{mC_0\lambda_a^2}{16n^2d^2C^2_{\psi,d}}\right)$, we obtain the estimate
    \begin{equation*}
        \begin{aligned}
            \lambda_{\min}(\mG^{[a]}(\vtheta^0)
             & \geq \frac{\kappa^2}{\kappa'}\lambda_a-\frac{\kappa^2}{\kappa'}\Norm{\frac{\kappa'}{\kappa^2}\mG^{[a]}(\vtheta^0)-\mK^{[a]}}_\mathrm{F} \\
             & \geq \frac{\kappa^2}{\kappa'}\left(\lambda_a-\eps\right)                                                                                \\
             & \geq\frac{3}{4}\frac{\kappa^2}{\kappa'}\lambda_a.
        \end{aligned}
    \end{equation*}
\end{proof}
\begin{proposition}[local in time exponential decay of $R_S$, $\vw$-lazy training]
    Given $\delta\in(0,1)$ and the sample set $S = {\{(\vx_i, y_i)\}}_{i=1}^n\subset\Omega$ with $\vx_i$'s drawn i.i.d.\ from some unknown distribution $\fD$. Suppose that Assumption~\ref{assump..lambda} holds. If $m\geq\frac{16n^2d^2C_{\psi,d}^2}{C_0\lambda_a^2}\log\frac{2n^2}{\delta}$, then with probability at least $1-\delta$ over the choice of $\vtheta^0$, for $t\in[0,t^*_a)$,
    \begin{equation}
        \RS(\vtheta(t))\leq\exp\left(-\frac{m\kappa^2}{\kappa'n}\lambda_at\right)\RS(\vtheta^0).
    \end{equation}
\end{proposition}
\begin{proof}
    By Prop.~\ref{prop:exp_a}, for any $\delta\in(0,1)$ with probability $1-\delta$ over the choice of $\vtheta^0$ and for any $\vtheta\in\fN_a(\vtheta^0)$,
    \begin{equation*}
        \begin{aligned}
            \lambda_{\min}(\mG^{[a]}(\vtheta))
             & \geq\lambda_{\min}(\mG^{[a]}(\vtheta^0))-\norm{\mG(\vtheta)-\mG(\vtheta^0)}_\mathrm{F}         \\
             & \geq \frac{3}{4}\frac{\kappa^2}{\kappa'}\lambda_a-\frac{1}{4}\frac{\kappa^2}{\kappa'}\lambda_a \\
             & = \frac{1}{2}\frac{\kappa^2}{\kappa'}\lambda_a.
        \end{aligned}
    \end{equation*}
    Therefore
    \begin{equation*}
        \begin{aligned}
            \frac{\D}{\D t}\RS(\vtheta(t))
             & = -\frac{m}{n^2}\ve^{\T}\mG\ve                                                    \\
             & \leq -\frac{m}{n^2}\ve^{\T}\mG^{[a]}\ve                                           \\
             & \leq -\frac{2m}{n}\lambda_{\min}\left(\mG^{[a]}(\vtheta(t))\right)\RS(\vtheta(t)) \\
             & \leq -\frac{m\kappa^2}{\kappa'n}\lambda_a\RS(\vtheta(t)).
        \end{aligned}
    \end{equation*}
    This leads to the linear convergence rate.
\end{proof}
\begin{proposition}[bounds on the change of parameters, $\vw$-lazy training]
    Given $\delta\in(0,1)$ and the sample set $S = {\{(\vx_i, y_i)\}}_{i=1}^n\subset\Omega$ with $\vx_i$'s drawn i.i.d.\ from some unknown distribution $\fD$. Suppose that Assumption~\ref{assump..lambda} holds. If $m\geq\frac{16n^2d^2C^2_{\psi,d}}{C_0\lambda_a^2}\log\frac{4n^2}{\delta}$ and $\frac{m\kappa}{\kappa'}\geq\frac{4\sqrt{2d}n\sqrt{\RS(\vtheta^0)}}{\lambda_a}$ and $\kappa'\leq 1$, then with probability at least $1-\delta$ over the choice of $\vtheta^0$ and for any $t\in[0,t^*_a)$, $k\in[m]$,
    \begin{align}
        \max\limits_{k\in[m]}\abs{a_k(t) - a_k(0)}
         & \leq 2\frac{1}{\kappa'}\sqrt{2\log\frac{4m(d+1)}{\delta}}p_a, \\
        \max\limits_{k\in[m]}\norm{\vw_k(t) - \vw_k(0)}_{\infty}
         & \leq 2\sqrt{2\log\frac{4m(d+1)}{\delta}}p_a,
    \end{align}
    and
    \begin{equation}
        \max\limits_{k\in[m]}\{\abs{a_k(0)},\;\norm{\vw_k(0)}_{\infty}\} \leq \sqrt{2\log\frac{4m(d+1)}{\delta}},
    \end{equation}
    where $p_a=\frac{2\sqrt{2}dn\sqrt{\RS(\vtheta^0)}}{m\kappa\lambda_a/\kappa'}$.
\end{proposition}
\begin{proof}
    Since
    \begin{equation*}
        \alpha(t)=\max\limits_{k\in[m],s\in[0,t]}|a_k(s)|, \quad \omega(t)=\max\limits_{k\in[m],s\in[0,t]}\norm{\vw_k(s)}_{\infty},
    \end{equation*}
    then
    \begin{equation*}
        \begin{aligned}
            \abs{\nabla_{a_k}\RS}^2    & =\left\lvert\frac{1}{n}\sum_{i=1}^n e_i\kappa\sigma(\vw_k^\T\vx_i)\right\rvert^2\leq 2\norm{\vw_k}^2_1\kappa^2\RS(\vtheta)\leq 2d^2(\omega(t))^2\kappa^2\RS(\vtheta),            \\
            \norm{\nabla_{\vw_k}\RS}^2 & =\left\lVert\frac{1}{n}\sum_{i=1}^n e_i\kappa a_k\sigma'(\vw_k^\T\vx_i)\vx_i\right\rVert^2_{\infty}\leq 2\abs{a_k}^2\kappa^2\RS(\vtheta)\leq 2(\alpha(t))^2\kappa^2\RS(\vtheta).
        \end{aligned}
    \end{equation*}
    By Prop.~\ref{prop:exp_RS}, we have if $m\geq \frac{16n^2d^2C_{\psi,d}^2}{\lambda^2C_0}\log\frac{8n^2}{\delta}$, then with probability at least $1 - \delta/2$ over the choice of $\vtheta^0$,
    \begin{equation*}
        \begin{aligned}
            \abs{a_k(t) - a_k(0)}
             & \leq\frac{1}{\kappa'}\int_0^t\abs{\nabla_{a_k}\RS(\vtheta(s))}\diff{s}                                                                                                                   \\
             & \leq\frac{\sqrt{2}d\kappa}{\kappa'}\int_{0}^{t} \omega(s)\sqrt{\RS(\vtheta(s))}\diff{s}                                                                                                  \\
             & \leq\frac{\sqrt{2}d\kappa}{\kappa'}\omega(t)\int_{0}^{t}\sqrt{\RS(\vtheta^0)}\exp\left(-\frac{m\kappa^2}{2n}\left(\frac{1}{\kappa'}\lambda_a+\kappa'\lambda_{\vw}\right)s\right)\diff{s} \\
             & \leq \frac{2\sqrt{2}dn\sqrt{\RS(\vtheta^0)}}{m\kappa\kappa'\left(\lambda^{[a]}_S/\kappa'+\kappa'\lambda_{\vw}\right)}\omega(t)                                                           \\
             & =\frac{p_a}{\kappa'}\omega(t).
        \end{aligned}
    \end{equation*}
    On the other hand,
    \begin{equation*}
        \begin{aligned}
            \norm{\vw_k(t) - \vw_k(0)}_{\infty}
             & \leq \kappa'\int_{0}^{t} \norm{\nabla_{\vw_k}\RS(\vtheta(s))}_{\infty}\diff{s}                                                                                                     \\
             & \leq \sqrt{2}\kappa\kappa' \int_{0}^t \alpha(s)\sqrt{\RS(\vtheta(s))} \diff{s}                                                                                                     \\
             & \leq \sqrt{2}\kappa\kappa' \alpha(t) \int_{0}^{t} \sqrt{\RS(\vtheta^0)}\exp\left(-\frac{m\kappa^2}{2n}\left(\frac{1}{\kappa'}\lambda_a+\kappa'\lambda_{\vw}\right)s\right)\diff{s} \\
             & \leq \frac{2\sqrt{2}n\sqrt{\RS(\vtheta^0)}\kappa'}{m\kappa\left(\lambda^{[a]}_S/\kappa'+\kappa'\lambda_{\vw}\right)}\alpha(t)                                                      \\
             & \leq p_a\kappa'\alpha(t).
        \end{aligned}
    \end{equation*}
    Thus
    \begin{equation*}
        \begin{aligned}
            \alpha(t) & \leq\alpha(0)+p_a\omega(t)\frac{1}{\kappa'}, \\
            \omega(t) & \leq\omega(0)+p_a\alpha(t)\kappa'.
        \end{aligned}
    \end{equation*}
    By Lemma~\ref{lem..InitialParameter}, we have with probability at least $1 - \delta/2$ over the choice of $\vtheta^0$,
    \begin{equation*}
        \max\limits_{k\in[m]}\{\abs{a_k(0)},\;\norm{\vw_k(0)}_{\infty}\}\leq\sqrt{2\log\frac{4m(d+1)}{\delta}}.
    \end{equation*}
    If
    \begin{equation*}
        m \geq \frac{4\sqrt{2}dn\sqrt{\RS(\vtheta^0)}}{\kappa\left(\lambda_a/\kappa'+\kappa'\lambda_{\vw}\right)},
    \end{equation*}
    then
    \begin{equation*}
        p_a=\frac{2\sqrt{2}dn\sqrt{\RS(\vtheta^0)}}{m\kappa\lambda_a/\kappa'}\leq \frac{1}{2}.
    \end{equation*}
    Thus
    \begin{align*}
        \alpha(t) & \leq\alpha(0)+\frac{p_a}{\kappa'}\omega(0)+p_a^2\alpha(t),      \\
        \alpha(t) & \leq\frac{4}{3}\alpha(0)+\frac{2}{3}\frac{1}{\kappa'}\omega(0),
    \end{align*}
    Therefore
    \begin{equation*}
        \alpha(t)\leq 2\frac{1}{\kappa'}\sqrt{2\log\frac{4m(d+1)}{\delta}}.
    \end{equation*}
    Similarly, one can obtain the estimate of $\omega(t)$ as
    \begin{equation*}
        \omega(t)\leq 2\sqrt{2\log\frac{4m(d+1)}{\delta}}.
    \end{equation*}
    Finally, with probability at least $1-\delta$ over the choice of $\vtheta^0$ and for any $t\in[0, t^*_a)$, we have
    \begin{equation*}
        \begin{aligned}
            \max\limits_{k\in[m]}|a_k(t) - a_k(0)|
             & \leq 2\frac{1}{\kappa'}\sqrt{2\log\frac{4m(d+1)}{\delta}}p_a, \\
            \max\limits_{k\in[m]}\norm{\vw_k(t) - \vw_k(0)}_{\infty}
             & \leq 2\sqrt{2\log\frac{4m(d+1)}{\delta}}p_a,
        \end{aligned}
    \end{equation*}
    which completes the proof.
\end{proof}
\section{Proof of Theorem \ref{thm..LinearRegime}}
We further divide the linear regime into two part: $\gamma<1$ where the training dynamics is $\vtheta$-lazy and $\gamma'>\gamma-1$ where the training dynamics is $\vw$-lazy. Theorem \ref{thm..LinearRegime} is hence covered by Proposition \ref{prop..ThetaLazyRegime} and Proposition \ref{prop..WLazyRegime} whose proofs are given in this section.
\begin{proposition}[$\vtheta$-lazy training]\label{prop..ThetaLazyRegime}
    Given $\delta\in(0,1)$ and the sample set $S = {\{(\vx_i, y_i)\}}_{i=1}^n\subset\Omega$ with $\vx_i$'s drawn i.i.d.\ from some unknown distribution $\fD$. Suppose that Assumption~\ref{assump..lambda} and Assumption~\ref{assump..gammagamma'} hold. ASI is used if $\gamma\leq\frac{1}{2}$. Suppose that $\gamma<1$ and the dynamics \eqref{eq..MainDynamics}--\eqref{eq..MainInitialization} is considered. Then for sufficiently large $m$, with probability at least $1-\delta$ over the choice of $\vtheta^0$, we have
    \begin{enumerate}[(a)]
        \item $\sup\limits_{t\in[0,+\infty)}\norm{\vtheta(t)-\vtheta^0}_2\lesssim\frac{1}{\sqrt{m}\kappa}\log m$.
        \item $\RS(\vtheta(t))\leq \exp\left(-\frac{2m\kappa^2\lambda t}{n}\right)\RS(\vtheta^0)$. \\
              Moreover, we have with probability at least $1-\delta-2\exp\left(-\frac{C_0m(d+1)}{4C^2_{\psi,1}}\right)$.
        \item $\sup\limits_{t\in[0,+\infty)}\frac{\norm{\vtheta(t)-\vtheta^0}_2}{\norm{\vtheta^0}_2}\lesssim\frac{1}{m\kappa}\log m$.
    \end{enumerate}
\end{proposition}
\begin{proof}
    Let $t\in[0,t^*)$, $p=\frac{2\sqrt{2}dn\sqrt{\RS(\vtheta^0)}}{m\kappa\left(\lambda_a/\kappa'+\kappa'\lambda_{\vw}\right)}$ and $\xi=\sqrt{2\log\frac{8m(d+1)}{\delta}}$.
    \begin{enumerate}[(a)]
        \item From Proposition~\ref{prop:a_w} we have with probability at least $1-\delta/2$ over the choice of $\vtheta^0$,
              \begin{equation*}
                  \begin{aligned}
                      \sup\limits_{t\in[0,t^*]}\norm{\vtheta(t)-\vtheta^0}_2
                       & \leq \left[m(d+1)\left(2\sqrt{2\log\frac{8m(d+1)}{\delta}}\frac{\sqrt{2}dn\sqrt{\RS(\vtheta^0)}}{m\kappa\left(\lambda_a/\kappa'+\kappa'\lambda_{\vw}\right)}\right)^2\right]^{\frac{1}{2}}      \\
                       & =\max\left\{\kappa',\frac{1}{\kappa'}\right\}\sqrt{m(d+1)}2\sqrt{2\log\frac{8m(d+1)}{\delta}}\frac{\sqrt{2}dn\sqrt{\RS(\vtheta^0)}}{m\kappa\left(\lambda_a/\kappa'+\kappa'\lambda_{\vw}\right)} \\
                       & \leq\max\left\{\kappa',\frac{1}{\kappa'}\right\}\frac{4\sqrt{d+1}dn\sqrt{\log\frac{8m(d+1)}{\delta}}\sqrt{\RS(\vtheta^0)}}{\sqrt{m}\kappa\left(\lambda_a/\kappa'+\kappa'\lambda_{\vw}\right)}   \\
                       & \leq\frac{4\sqrt{d+1}dn\sqrt{\log\frac{8m(d+1)}{\delta}}\sqrt{\RS(\vtheta^0)}}{\sqrt{m}\kappa\lambda}                                                                                           \\
                       & \lesssim \frac{1}{\sqrt{m}\kappa}\log m,
                  \end{aligned}
              \end{equation*}
              where we use the fact
              \begin{equation*}
                  \frac{\max\left\{\kappa',\frac{1}{\kappa'}\right\}}{\lambda_a/\kappa'+\kappa'\lambda_{\vw}}\leq\max\left\{\frac{\kappa'}{\kappa'\lambda_{\vw}},\frac{1/\kappa'}{\lambda_a/\kappa'}\right\}\leq\frac{1}{\lambda}.
              \end{equation*}
        \item The linear convergence rate is essentially proved by Prop.~\ref{prop:exp_RS} with $t^* = +\infty$. We divide the proof into the following three steps. In particular, $t^*=+\infty$ is proved in the step (iii).
              \begin{enumerate}[(i)]
                  \item Let
                        \begin{equation*}
                            g^{[a]}_{ij}(\vw) := \sigma(\vw^\T\vx_i)\sigma(\vw^\T\vx_j),
                        \end{equation*}
                        then
                        \begin{equation*}
                            \Abs{G_{ij}^{[a]}(\vtheta(t)) - G_{ij}^{[a]}(\vtheta(0))} \leq \frac{\kappa^2}{m\kappa'}\sum_{k=1}^m \Abs{g^{[a]}_{ij}(\vw_k(t)) - g^{[a]}_{ij}(\vw_k(0))}.
                        \end{equation*}
                        By mean value theorem, for somce $c\in(0,1)$,
                        \begin{equation*}
                            \Abs{g^{[a]}_{ij}(\vw_k(t)) - g^{[a]}_{ij}(\vw_k(0))} \leq \norm{\nabla g^{[a]}_{ij}\left(c\vw_k(t) + (1-c)\vw_k(0)\right)}_{\infty}\norm{\vw_k(t) - \vw_k(0)}_1,
                        \end{equation*}
                        where
                        \begin{equation*}
                            \nabla g_{ij}^{[a]}(\vw)=\sigma'(\vw\cdot\xi_i)\sigma(\vw^\T\vx_j)\vx_i+\sigma(\vw^\T\vx_i)\sigma'(\vw^\T\vx_j)\vx_j,
                        \end{equation*}
                        and
                        \begin{equation*}
                            \norm{\nabla g_{ij}^{[a]}(\vw)}_{\infty}\leq 2\norm{\vw}_1.
                        \end{equation*}
                        From Proposition~\ref{prop:a_w} we have with probability at least $1-\delta/2$ over the choice of $\vtheta^0$,
                        \begin{align*}
                            \norm{\vw_k(t)-\vw_k(0)}_{\infty} & \leq p\alpha(t)\kappa'\leq 2\max\{\kappa',1\}\xi p, \\
                            \norm{\vw_k(t)-\vw_k(0)}_1        & \leq 2d\max\{\kappa',1\}\xi p.
                        \end{align*}
                        Thus
                        \begin{equation*}
                            \begin{aligned}
                                \norm{c\vw_k(t) + (1-c)\vw_k(0)}_1
                                 & \leq d\left(\norm{\vw_k(0)}_{\infty} + \norm{\vw_k(t) - \vw_k(0)}_{\infty}\right) \\
                                 & \leq d\left(\xi+2\max\{\kappa',1\}\xi p\right)                                    \\
                                 & \leq 2d\xi\max\{\kappa',1\}.
                            \end{aligned}
                        \end{equation*}
                        Then
                        \begin{equation*}
                            \abs{G_{ij}^{[a]}(\vtheta(t)) - G_{ij}^{[a]}(\vtheta(0))} \leq 8d^2\kappa^2\xi^2\max\left\{\kappa',\frac{1}{\kappa'}\right\}p,
                        \end{equation*}
                        and
                        \begin{equation*}
                            \begin{aligned}
                                \norm{\mG^{[a]}(\vtheta(t)) - \mG^{[a]}(\vtheta(0))}_\mathrm{F}
                                 & \leq 8d^2n\kappa^2\max\left\{\kappa',\frac{1}{\kappa'}\right\}\left(2\log\frac{8m(d+1)}{\delta}\right)\frac{2\sqrt{2}dn\sqrt{\RS(\vtheta^0)}}{m\kappa\left(\lambda_a/\kappa'+\kappa'\lambda_{\vw}\right)} \\
                                 & \leq\kappa\max\left\{\kappa',\frac{1}{\kappa'}\right\}\frac{32\sqrt{2}d^3n^2\left(\log\frac{8m(d+1)}{\delta}\right)\sqrt{\RS(\vtheta^0)}}{m\left(\lambda_a/\kappa'+\kappa'\lambda_{\vw}\right)}.
                            \end{aligned}
                        \end{equation*}
                        If we choose
                        \begin{equation}\label{cond:m1}
                            m\kappa\geq\frac{128\sqrt{2}d^3n^2\left(\log\frac{8m(d+1)}{\delta}\right)\sqrt{\RS(\vtheta^0)}}{\lambda^2},
                        \end{equation}
                        then noticing that
                        \begin{equation*}
                            \begin{aligned}
                                \frac{1}{\lambda^2}
                                 & \geq\frac{1}{\left(4\left(\frac{1}{27}\lambda_a(\lambda_{\vw})^3\right)^{1/4}\right) ^2} \\
                                 & \geq\frac{1}{\left(\lambda_a/(\kappa')^{3/2}+\sqrt{\kappa'}\lambda_{\vw}\right)^2}       \\
                                 & =\frac{\kappa'}{\left(\lambda_a/\kappa'+\kappa'\lambda_{\vw}\right)^2}
                            \end{aligned}
                        \end{equation*}
                        and
                        \begin{equation*}
                            \begin{aligned}
                                \frac{1}{\lambda^2}
                                 & \geq\frac{1}{\left(4\left(\frac{1}{27}(\lambda_a)^3\lambda_{\vw}\right)^{1/4}\right) ^2} \\
                                 & \geq \frac{1}{\left(\lambda_a/\sqrt{\kappa'}+(\kappa')^{3/2}\lambda_{\vw}\right)^2}      \\
                                 & =\frac{1}{\left(\lambda_a/\kappa'+\kappa'\lambda_{\vw}\right)^2\kappa'},
                            \end{aligned}
                        \end{equation*}
                        we have
                        \begin{equation*}
                            m\kappa\geq \max\left\{\kappa',\frac{1}{\kappa'}\right\}\frac{256\sqrt{2}d^3n^2\left(\log\frac{8m(d+1)}{\delta}\right)\sqrt{\RS(\vtheta^0)}}{\left(\lambda_a/\kappa'+\kappa'\lambda_{\vw}\right)^2}.
                        \end{equation*}
                        Therefore \begin{equation}\label{thm-proof:step1}
                            \norm{\mG^{[a]}(\vtheta(t)) - \mG^{[a]}(\vtheta(0))}_\mathrm{F}\leq\frac{1}{8}\kappa^2\left(\frac{1}{\kappa'}\lambda_a+\kappa'\lambda_{\vw}\right).
                        \end{equation}
                  \item Define
                        \begin{equation*}
                            \begin{aligned}
                                D_{i,k}=\{\omega_k(0)\mid \norm{\vw_k(t)-\vw_k(0)}_{\infty} & \leq 2\xi\max\{\kappa',1\} p,                                    \\
                                                                                            & \sigma'(\vw_k(t^*)\cdot\vx_i)\not=\sigma'(\vw_k(0)\cdot\vx_i)\}.
                            \end{aligned}
                        \end{equation*}
                        If $\abs{\vw_k(0)\cdot\vx_i}>4d\max\{\kappa',1\}\xi p$, then
                        \begin{equation*}
                            \abs{\vw_k(t)\cdot\vx_i-\vw_k(0)\cdot\vx_i}\leq\norm{\vx_i}_1\norm{\vw_k(t)-\vw_k(0)}_{\infty}\leq2d\sqrt{2\log\frac{8m(d+1)}{\delta}} p,
                        \end{equation*}
                        thus $\vw_k(t)\cdot\vx_i$ and $\vw_k(0)\cdot\vx_i$ have the same sign which means $D_{i,k}$ is empty. Recall that $\vx_i\in[0, 1]^d$ with $(\vx_i)_d=1$, then $\norm{\vx_i}_2\geq 1$. Let $\hat{\vx}_i=\frac{\vx_i}{\norm{\vx_i}_2}$ then $\abs{\vw_k(0)\cdot\vx_i}\geq\abs{\vx_k(0)\cdot\hat{\vx}_i}$ and
                        \begin{equation*}
                            \begin{aligned}
                                \Prob(D_{i,k})
                                 & \leq\Prob(|\vw_k(0)\cdot\vx_i|\leq 4d\max\{\kappa',1\}\xi p)                         \\
                                 & \leq\Prob(|\vw_k(0)\cdot\hat{\vx}_i|\leq 4d\max\{\kappa',1\}\xi p)                   \\
                                 & = \Prob(|\vw_k(0)\cdot(1,0,0,\ldots,0)^{\T}|\leq 4d\max\{\kappa',1\}\xi p)           \\
                                 & = \Prob(|(\vw_k(0))_1|\leq 4d\max\{\kappa',1\}\xi p)                                 \\
                                 & = 2\int_0^{4d\max\{\kappa',1\}\xi p}\frac{1}{\sqrt{2\pi}}\E^{-\frac{x^2}{2}}\diff{x} \\
                                 & \leq \frac{8}{\sqrt{2\pi}}d\max\{\kappa',1\}\xi p                                    \\
                                 & \leq 4d\max\{\kappa',1\}\xi p.
                            \end{aligned}
                        \end{equation*}
                        Then
                        \begin{equation*}
                            \begin{aligned}
                                \abs{G^{[\vw]}_{ij}(\vtheta(t))-G^{[\vw]}_{ij}(\vtheta(0))}
                                 & \leq \frac{\kappa^2\kappa'\abs{\vx_i\cdot\vx_j}}{m}\sum_{k=1}^m\Big|a_k^2(t^*)\sigma'(\vw_k(t)\cdot\vx_i)\sigma'(\vw_k(t)\cdot\vx_j) \\
                                 & \quad\quad\quad\quad\quad\quad\quad\quad -a_k^2(0)\sigma'(\vw_k(0)\cdot\vx_i)\sigma'(\vw_k(0)\cdot\vx_j)\Big|                        \\
                                 & \leq \frac{\kappa^2\kappa'd}{m}\sum_{k=1}^m\left[a_k^2(t)\abs{D_{k,i,j}}+|a_k^2(t)-a_k^2(0|\right],
                            \end{aligned}
                        \end{equation*}
                        where
                        \begin{equation*}
                            D_{k,i,j}:=\sigma'(\vw_k(t)\cdot\vx_i)\sigma'(\vw_k(t)\cdot\vx_j)-\sigma'(\vw_k(0)\cdot\vx_i)\sigma'(\vw_k(0)\cdot\vx_j).
                        \end{equation*}
                        Thus
                        \begin{equation*}
                            \Exp|D_{k,i,j}|\leq\Prob(D_{k,i}\cup D_{k,j})\leq 8d\max\{\kappa',1\}\xi p.
                        \end{equation*}
                        At the same time
                        \begin{equation*}
                            \begin{aligned}
                                \abs{a_k^2(t)-a_k^2(0)}
                                 & \leq\abs{a_k(t)-a_k(0)}^2+2\abs{a_k(0)}\abs{a_k(t)-a_k(0)}                                                                     \\
                                 & \leq \left(2\max\left\{\frac{1}{\kappa'},1\right\}\xi p\right)^2+2\xi\left(2\max\left\{\frac{1}{\kappa'},1\right\}\xi p\right) \\
                                 & \leq 6\xi^2\max\left\{\frac{1}{\kappa'^2},1\right\}p,
                            \end{aligned}
                        \end{equation*}
                        so
                        \begin{equation*}
                            \begin{aligned}
                                a_k^2(t)\leq \abs{a_k^2(t)-a_k^2(0)}+a_k^2(0)
                                 & \leq \left(2\max\left\{\frac{1}{\kappa'},1\right\}\xi p\right)^2+2\xi\left(2\max\left\{\frac{1}{\kappa'},1\right\}\xi p\right) + \xi^2 \\
                                 & \leq 4\max\left\{\frac{1}{\kappa'^2},1\right\}\xi^2.
                            \end{aligned}
                        \end{equation*}
                        Then
                        \begin{equation*}
                            \begin{aligned}
                                \Exp\sum_{i,j=1}^n & \Abs{G_{ij}^{[\vw]}(\vtheta(t)) - G_{ij}^{[\vw]}(\vtheta(0))}\\
                                 & \leq\sum_{i,j=1}^n\frac{\kappa^2\kappa'  d}{m}\sum_{k=1}^m\left(4\max\left\{\frac{1}{\kappa'^2},1\right\}\xi^2\Exp|D_{k,i,j}|+6\max\left\{\frac{1}{\kappa'^2},1\right\}\xi^2 p\right)           \\
                                 & \leq\sum_{i,j=1}^n\frac{\kappa^2\kappa' d}{m}\sum_{k=1}^m\left(4\max\left\{\frac{1}{\kappa'^2},1\right\}\xi^2 8d\max\{\kappa',1\}\xi p+6\max\left\{\frac{1}{\kappa'^2},1\right\}\xi^2 p\right) \\
                                 & \leq \kappa^2\kappa' d n^2\left(32d\xi\max\{\kappa',\frac{1}{\kappa'^2} \}   +6\max\left\{\frac{1}{\kappa'^2},1\right\}\right)\xi^2 p                                                          \\
                                 & \leq 40\kappa^2 d^2n^2\left(2\log\frac{8m(d+1)}{\delta}\right)^{3/2}\max\{\kappa'^2,\frac{1}{\kappa'}\} p.
                            \end{aligned}
                        \end{equation*}
                        By Markov's inequality, with probability at least $1-\delta/2$ over the choice of $\vtheta^0$, we have
                        \begin{equation*}
                            \begin{aligned}
                                &\norm{G^{[\vw]}(\vtheta(t))-G^{[\vw]}(\vtheta(0))}_\mathrm{F} \\
                                 & \leq \sum_{i,j=1}^n\Big|G_{ij}^{[\vw]}(\vtheta(t))-G^{[\vw]}_{ij}(\vtheta(0))\Big|                                                                                                                                                            \\
                                 & \leq \max\left\{\kappa'^2,\frac{1}{\kappa'}\right\}\frac{40\kappa^2 d^2n^2\left(2\log\frac{8m(d+1)}{\delta}\right)^{3/2} p}{\delta/2}                                                                                                         \\
                                 & \leq \max\left\{\kappa'^2,\frac{1}{\kappa'}\right\}\frac{80\kappa^2 d^2n^2 2\sqrt{2}}{\delta}\left(\log\frac{8m(d+1)}{\delta}\right)^{3/2}\frac{2\sqrt{2}dn\sqrt{\RS(\vtheta^0)}}{m\kappa\left(\lambda_a/\kappa'+\kappa'\lambda_{\vw}\right)} \\
                                 & \leq\kappa\max\left\{\kappa'^2,\frac{1}{\kappa'}\right\}\frac{640 d^3n^3\left(\log\frac{8m(d+1)}{\delta}\right)^{3/2}\sqrt{\RS(\vtheta^0)}\delta^{-1}}{m\left(\lambda_a/\kappa'+\kappa'\lambda_{\vw}\right)}.
                            \end{aligned}
                        \end{equation*}
                        If
                        \begin{equation*}
                            m\geq\frac{5120\delta^{-1}d^3n^3\left(\log\frac{8m(d+1)}{\delta}\right)^{3/2}\sqrt{\RS(\vtheta^0)}}{\lambda^2},
                        \end{equation*}
                        then noticing that
                        \begin{equation*}
                            \frac{1}{\lambda^2} \geq\frac{\kappa'^2}{\left(\kappa'\lambda_{\vw}\right)^2} \geq\frac{\kappa'^2}{\left(\lambda_a/\kappa'+\kappa'\lambda_{\vw}\right)^2}
                        \end{equation*}
                        and
                        \begin{equation*}
                            \begin{aligned}
                                \frac{1}{\lambda^2}
                                 & \geq\frac{1}{\left(4\left(\frac{1}{27}(\lambda_a)^3\lambda_{\vw}\right)^{1/4}\right) ^2} \\
                                 & \geq\frac{1}{\left(\lambda_a/\sqrt{\kappa'}+(\kappa')^{3/2}\lambda_{\vw}\right)^2}       \\
                                 & =\frac{1}{\left(\lambda_a/\kappa'+\kappa'\lambda_{\vw}\right)^2\kappa'},
                            \end{aligned}
                        \end{equation*}
                        we have \begin{equation}\label{thm-proof:step2}
                            \norm{G^{[\vw]}(\vtheta(t))-G^{[\vw]}(\vtheta(0))}_\mathrm{F}\leq\frac{1}{8}\kappa^2\left(\frac{1}{\kappa'}\lambda_a+\kappa'\lambda_{\vw}\right).
                        \end{equation}
                  \item For $t\in[0,t^*)$,
                        \begin{equation*}
                            \RS(\vtheta(t))\leq\exp\left(-\frac{m\kappa^2}{n}\left(\frac{1}{\kappa'}\lambda_a+\kappa'\lambda_{\vw}\right)t\right)\RS(\vtheta^0)\leq\exp\left(-\frac{2m\kappa^2\lambda}{n}\right)\RS(\vtheta^0).
                        \end{equation*}
                        Suppose that $t^*<+\infty$ then one can take the limit $t\to t^*$ in~\eqref{thm-proof:step1} and~\eqref{thm-proof:step2}. This will lead to a contradiction with the definition of $t^*$. Therefore $t^*=+\infty$.
              \end{enumerate}
        \item By Proposition \ref{prop..InitialThetaNorm}, we have with probability at least $1-2\exp\left(-\frac{C_0m(d+1)}{4C^2_{\psi,1}}\right)$ over the choice of $\vtheta^0$,
              \begin{equation*}
                  \norm{\vtheta^0}\geq \sqrt{\frac{m(d+1)}{2}},
              \end{equation*}
              Therefore, with probability at least $1-\delta-2\exp\left(-\frac{C_0m(d+1)}{4C^2_{\psi,1}}\right)$ over the choice of $\vtheta^0$, we have
              \begin{equation*}
                  \begin{aligned}
                      \sup\limits_{t\in[0,+\infty)}\frac{\norm{\vtheta(t)-\vtheta^0}_2}{\norm{\vtheta^0}_2}
                       & \leq \sqrt{\frac{2}{m(d+1)}}\sup\limits_{t\in[0,+\infty)}\norm{\vtheta(t)-\vtheta^0}_2                                          \\
                       & \leq  \sqrt{\frac{2}{m(d+1)}} \frac{4\sqrt{d+1}dn\sqrt{\log\frac{8m(d+1)}{\delta}}\sqrt{\RS(\vtheta^0)}}{\sqrt{m}\kappa\lambda} \\
                       & \leq \frac{1}{m\kappa}\frac{4\sqrt{2}dn\sqrt{\log\frac{8m(d+1)}{\delta}}\sqrt{\RS(\vtheta^0)}}{\lambda}                         \\
                       & \lesssim \frac{1}{m\kappa}\log m.
                  \end{aligned}
              \end{equation*}
    \end{enumerate}
\end{proof}
\begin{remark}
    The proof indicates more quantitative conditions on $m$ and $\kappa$ for Proposition~\ref{prop..ThetaLazyRegime} to hold:
    \begin{equation}
        m\geq \frac{16n^2d^2C_{\psi,d}^2}{\lambda^2C_0}\log\frac{16n^2}{\delta},
    \end{equation}
    and
    \begin{equation}
        \begin{aligned}
            m\kappa\geq\max\Bigg\{\frac{2\sqrt{2d}n\sqrt{\RS(\vtheta^0)}}{\lambda}, & \;\frac{128\sqrt{2}d^3n^2\left(\log\frac{8m(d+1)}{\delta}\sqrt{\RS(\vtheta^0)}\right)}{\lambda^2},                \\
                                                                                    & \frac{5120\delta^{-1}d^3n^3\left(\log\frac{8m(d+1)}{\delta}\right)^{3/2}\sqrt{\RS(\vtheta^0)}}{\lambda^2}\Bigg\}.
        \end{aligned}
    \end{equation}
\end{remark}

\begin{proposition}[$\vw$-lazy training]\label{prop..WLazyRegime}
    Given $\delta\in(0,1)$ and the sample set $S = {\{(\vx_i, y_i)\}}_{i=1}^n\subset\Omega$ with $\vx_i$'s drawn i.i.d.\ from some unknown distribution $\fD$. Suppose that Assumption~\ref{assump..lambda} and Assumption~\ref{assump..gammagamma'} hold.  Suppose that $\gamma'>\gamma-1$, $\gamma'>0$, and the dynamics \eqref{eq..MainDynamics}--\eqref{eq..MainInitialization} is considered. Then for sufficiently large $m$, with probability at least $1-\delta$ over the choice of $\vtheta^0$, we have
    \begin{enumerate}[(a)]
        \item \begin{equation*}
                  \begin{aligned}
                      \sup\limits_{t\in[0,+\infty)}\norm{\vtheta_{\vw}(t)-\vtheta_{\vw}^0}_2
                      \leq \sup\limits_{t\in[0,+\infty)}\norm{\vtheta(t)-\vtheta^0}_2
                      \lesssim\frac{1}{\sqrt{m}\kappa}\log m.
                  \end{aligned}
              \end{equation*}
        \item $\RS(\vtheta(t))\leq\exp\left(-\frac{m\kappa^2\lambda_a t}{\kappa'n}\right)\RS(\vtheta^0)$. \\
              Moreover we have with probability at least $1-\delta-2\exp\left(-\frac{C_0m(d+1)}{4C_{\psi,1}^2}\right)$.
        \item \begin{equation*}
                  \begin{aligned}
                      \sup\limits_{t\in[0,+\infty)}\frac{\norm{\vtheta(t)-\vtheta^0}_2}{\norm{\vtheta^0}_2}
                       & \lesssim\frac{1}{m\kappa}\log m, \quad (\text{not }\ll 1), \\
                      \sup\limits_{t\in[0,+\infty)}\frac{\norm{\vtheta_{\vw}(t)-\vtheta_{\vw}^0}_2}{\norm{\vtheta_{\vw}}_2}
                       & \lesssim\frac{\kappa'}{m\kappa}\log m, \quad (\ll 1).
                  \end{aligned}
              \end{equation*}
    \end{enumerate}
\end{proposition}
\begin{proof}
    Let $t\in[0, t^*_a)$, $p_a=\frac{2\sqrt{2}dn\sqrt{\RS(\vtheta^0)}}{m\kappa\lambda_a/\kappa'}$, and $\xi=\sqrt{2\log\frac{8m(d+1)}{\delta}}$.
    \begin{enumerate}[(a)]
        \item From Proposition~\ref{prop:a_w} we have with probability at least $1-\delta/2$ over the choice of $\vtheta^0$
              \begin{equation*}
                  \begin{aligned}
                      \sup\limits_{t\in[0,t^*_a)}\norm{\vtheta_{\vw}(t)-\vtheta_{\vw}^0}_2
                       & \leq \sup\limits_{t\in[0,t^*_a)}\norm{\vtheta(t)-\vtheta^0}_2                                                                                                                                 \\
                       & \leq \left[\left(\frac{m}{\kappa'^2}+md\right)\left(2\sqrt{2\log\frac{8m(d+1)}{\delta}}p_a\right)^2\right]^{\frac{1}{2}}                                                                      \\
                       & =\sqrt{m(d+1)}2\frac{1}{\kappa'}\sqrt{2\log\frac{8m(d+1)}{\delta}}\frac{\sqrt{2}dn\sqrt{\RS(\vtheta^0)}}{m\kappa\lambda_a/\kappa'}                                                            \\
                       & \leq\max\left\{\kappa',\frac{1}{\kappa'}\right\}\frac{4\sqrt{d+1}dn\sqrt{\log\frac{8m(d+1)}{\delta}}\sqrt{\RS(\vtheta^0)}}{\sqrt{m}\kappa\left(\lambda_a/\kappa'+\kappa'\lambda_{\vw}\right)} \\
                       & \leq\frac{8\sqrt{d+1}dn\sqrt{\log\frac{8m(d+1)}{\delta}}\sqrt{\RS(\vtheta^0)}}{\sqrt{m}\kappa\lambda_a}.
                  \end{aligned}
              \end{equation*}
              %   \begin{equation*}
              %       \begin{aligned}
              %           \sup\limits_{t\in[0,t^*_a)}\norm{\vtheta_{\vw}(t)-\vtheta_{\vw}^0}_2
              %           & = \left[md\left(2\sqrt{2\log\frac{8m(d+1)}{\delta}}p_a\right)^2\right]^{1/2}                                    \\
              %           & = \sqrt{md}2\sqrt{2\log\frac{8m(d+1)}{\delta}}\frac{2\sqrt{2}dn\sqrt{\RS(\vtheta^0)}}{m\kappa\lambda_a/\kappa'} \\
              %           & = \frac{\kappa'}{\sqrt{m}\kappa}\frac{8\sqrt{d}dn\sqrt{\log\frac{8m(d+1)}{\delta}}\sqrt{\RS(\vtheta^0)}}{\lambda_a}          \\
              %           & \lesssim \frac{\kappa'}{\sqrt{m}\kappa}\log m.
              %       \end{aligned}
              %   \end{equation*}

        \item We divide this proof into the following two steps.
              \begin{enumerate}[(i)]
                  \item Let
                        \begin{equation*}
                            G^{[a]}_{ij}(\vw) := \sigma(\vw^\T\vx_i)\sigma(\vw^\T\vx_j),
                        \end{equation*}
                        then
                        \begin{equation*}
                            \abs{G_{ij}^{[a]}(\vtheta(t)) - G_{ij}^{[a]}(\vtheta(0))} \leq \frac{\kappa^2}{m\kappa'}\sum_{k=1}^m \abs{g^{[a]}_{ij}(\vw_k(t)) - g^{[a]}_{ij}(\vw_k(0))}.
                        \end{equation*}
                        By the mean value theorem, for somce $c\in(0,1)$,
                        \begin{equation*}
                            \abs{g^{[a]}_{ij}(\vw_k(t)) - g^{[a]}_{ij}(\vw_k(0))} \leq \norm{\nabla g_{ij}\left(c\vw_k(t) + (1-c)\vw_k(0)\right)}_{\infty}\norm{\vw_k(t) - \vw_k(0)}_1,
                        \end{equation*}
                        where
                        \begin{equation*}
                            \nabla g_{ij}^{[a]}(\vw)=\sigma'(\vw\cdot\xi_i)\sigma(\vw^\T\vx_j)\vx_i+\sigma(\vw^\T\vx_i)\sigma'(\vw^\T\vx_j)\vx_j,
                        \end{equation*}
                        and
                        \begin{equation*}
                            \norm{\nabla g_{ij}^{[a]}(\vw)}_{\infty}\leq 2\norm{\vw}_1.
                        \end{equation*}
                        From Proposition~\ref{prop:a_w} we have with probability at least $1-\delta/2$ over the choice of $\vtheta^0$,
                        \begin{align*}
                            \norm{\vw_k(t)-\vw_k(0)}_{\infty} & \leq p_a\alpha(t)\kappa'\leq 2\xi p_a, \\
                            \norm{\vw_k(t)-\vw_k(0)}_1        & \leq 2d\xi p_a.
                        \end{align*}
                        Thus
                        \begin{equation*}
                            \begin{aligned}
                                \norm{c\vw_k(t) + (1-c)\vw_k(0)}_1
                                 & \leq d\left(\norm{\vw_k(0)}_{\infty} + \norm{\vw_k(t) - \vw_k(0)}_{\infty}\right) \\
                                 & \leq d\left(\xi+2\xi p_a\right)                                                   \\
                                 & \leq 2d\xi.
                            \end{aligned}
                        \end{equation*}
                        Then
                        \begin{equation*}
                            \abs{G_{ij}^{[a]}(\vtheta(t)) - G_{ij}^{[a]}(\vtheta(0))} \leq 8d^2\frac{\kappa^2}{\kappa'}\xi^2p_a,
                        \end{equation*}
                        and
                        \begin{equation*}
                            \begin{aligned}
                                \norm{\mG^{[a]}(\vtheta(t)) - \mG^{[a]}(\vtheta(0))}_\mathrm{F}
                                 & \leq 16d^2n\left(\log\frac{8m(d+1)}{\delta}\right)\frac{\kappa^2}{\kappa'}p_a                              \\
                                 & \leq\frac{32\sqrt{2}d^3n^2\left(\log\frac{8m(d+1)}{\delta}\right)\sqrt{\RS(\vtheta^0)}\kappa}{m\lambda_a}.
                            \end{aligned}
                        \end{equation*}
                        If
                        \begin{equation*}
                            \frac{m\kappa}{\kappa'}\geq\frac{256\sqrt{2}d^3n^2\left(\log\frac{8m(d+1)}{\delta})\right)\sqrt{\RS(\vtheta^0)}}{\lambda_a^2},
                        \end{equation*}
                        then we have
                        \begin{equation}\label{thm-proof:w-step1}
                            \norm{\mG^{[a]}(\vtheta(t)) - \mG^{[a]}(\vtheta(0))}_\mathrm{F}\leq\frac{1}{8}\frac{\kappa^2}{\kappa'}\lambda_a.
                        \end{equation}
                  \item For $t\in[0,t^*_a)$ by Prop.~\ref{prop:exp_RS},
                        \begin{equation*}
                            \RS(\vtheta(t))\leq\exp\left(-\frac{m\kappa^2\lambda_at}{n\kappa'}\right)\RS(\vtheta^0).
                        \end{equation*}
                        Suppose that $t^*_a<+\infty$ then one can take the limit $t\to t^*_a$ in~\eqref{thm-proof:w-step1}. This will lead to a contradiction with the definition of $t^*_a$. Therefore $t^*_a=+\infty$.
              \end{enumerate}
        \item By Proposition \ref{prop..InitialThetaNorm}, we have with probability at least $1-2\exp(-\frac{C_0m(d+1)}{4C_{\psi,1}^2})$ over the choice of $\vtheta^0$,
              \begin{equation*}
                  \norm{\vtheta^0}^2_2\geq\frac{d+1}{2}m.
              \end{equation*}
              So with probability at least $1-\delta-2\exp\left(-\frac{C_0m(d+1)}{4C^2_{\psi,1}}\right)$ over the choice of $\vtheta^0$, we have
              \begin{equation*}
                  \begin{aligned}
                      \sup\limits_{t\in[0,+\infty)}\frac{\norm{\vtheta(t)-\vtheta^0}_2}{\norm{\vtheta^0}_2}
                       & \leq \sqrt{\frac{2}{m(d+1)}}\sup\limits_{t\in[0,+\infty)}\norm{\vtheta(t)-\vtheta^0}_2                                          \\
                       & \leq  \sqrt{\frac{2}{m(d+1)}} \frac{8\sqrt{d+1}dn\sqrt{\log\frac{8m(d+1)}{\delta}}\sqrt{\RS(\vtheta^0)}}{\sqrt{m}\kappa\lambda} \\
                       & \leq \frac{1}{m\kappa}\frac{8\sqrt{2}dn\sqrt{\log\frac{8m(d+1)}{\delta}}\sqrt{\RS(\vtheta^0)}}{\lambda_a}                       \\
                       & \lesssim \frac{1}{m\kappa}\log m.
                  \end{aligned}
              \end{equation*}
              Similarly, by Proposition \ref{prop..InitialThetaNorm}, we have with probability at least $1-2\exp(-\frac{C_0m(d+1)}{4C_{\psi,1}^2})$ over the choice of $\vtheta^0$,
              \begin{equation*}
                  \norm{\vtheta_{\vw}^0}^2_2\geq\frac{d}{2}m.
              \end{equation*}
              So with probability at least $1-\delta-2\exp\left(-\frac{C_0md}{4C^2_{\psi,1}}\right)$ over the choice of $\vtheta^0$, we have
              \begin{equation*}
                  \begin{aligned}
                      \sup\limits_{t\in[0,+\infty)}\frac{\norm{\vtheta_{\vw}(t)-\vtheta_{\vw}^0}_2}{\norm{\vtheta_{\vw}^0}_2}
                       & \leq \sqrt{\frac{2}{dm}}\sup\limits_{t\in[0,+\infty)}\norm{\vtheta_{\vw}(t)-\vtheta_{\vw}^0}_2                                           \\
                       & \leq  \sqrt{\frac{2}{dm}}\frac{\kappa'}{\sqrt{m}\kappa}\frac{8\sqrt{d}dn\sqrt{\log\frac{8m(d+1)}{\delta}}\sqrt{\RS(\vtheta^0)}}{\lambda} \\
                       & \lesssim \frac{\kappa'}{m\kappa}\log m.
                  \end{aligned}
              \end{equation*}
    \end{enumerate}
\end{proof}
We remark that in fact we can prove a similar result about the change of parameter $a_k$'s. We state this result as follows without proof.
\begin{proposition}[$a$-lazy training]\label{prop..ALazyRegime}
    Given $\delta\in(0,1)$ and the sample set $S = {\{(\vx_i, y_i)\}}_{i=1}^n\subset\Omega$ with $\vx_i$'s drawn i.i.d.\ from some unknown distribution $\fD$. Suppose that Assumption~\ref{assump..lambda} and Assumption~\ref{assump..gammagamma'} hold.  Suppose that $\gamma'<\gamma-1$, $\gamma'<0$, and the dynamics \eqref{eq..MainDynamics}--\eqref{eq..MainInitialization} is considered. Then for sufficiently large $m$, with probability at least $1-\delta$ over the choice of $\vtheta^0$, we have
    \begin{enumerate}[(a)]
        \item \begin{equation*}
                  \begin{aligned}
                      \sup\limits_{t\in[0,+\infty)}\norm{\vtheta_{a}(t)-\vtheta_{a}^0}_2
                      \leq \sup\limits_{t\in[0,+\infty)}\norm{\vtheta(t)-\vtheta^0}_2
                      \lesssim\frac{1}{\sqrt{m}\kappa}\log m.
                  \end{aligned}
              \end{equation*}
        \item $\RS(\vtheta(t))\leq\exp\left(-\frac{m\kappa^2\kappa'\lambda_{\vw} t}{n}\right)\RS(\vtheta^0)$. \\
              Moreover we have with probability at least $1-\delta-2\exp\left(-\frac{C_0m(d+1)}{4C_{\psi,1}^2}\right)$ over the choice of $\vtheta^0$, we have
        \item \begin{equation*}
                  \begin{aligned}
                      \sup\limits_{t\in[0,+\infty)}\frac{\norm{\vtheta(t)-\vtheta^0}_2}{\norm{\vtheta^0}_2}
                       & \lesssim\frac{1}{m\kappa}\log m, \quad (\text{not }\ll 1), \\
                      \sup\limits_{t\in[0,+\infty)}\frac{\norm{\vtheta_{a}(t)-\vtheta_{a}^0}_2}{\norm{\vtheta_{a}}_2}
                       & \lesssim\frac{1}{m\kappa\kappa'}\log m, \quad (\ll 1).
                  \end{aligned}
              \end{equation*}
    \end{enumerate}
\end{proposition}

\section{Proof of Theorem~\ref{thm..CondensedRegime}}
In order to characterize the condensed regime, we need a crucial proposition that ravels a natrual relation between $a_k(t)$ and $\vw_k(t)$ during the GD training dynamics.
\begin{proposition}\label{prop..a-w-est}
    Consider the GD training dynamics~\eqref{eq..MainDynamics}--\eqref{eq..MainInitialization}, then we have
    \begin{equation}
        \abs{a_k(t)}\leq \frac{1}{\kappa'}\norm{\vw_k(t)}_2+\abs{a_k^0},
    \end{equation}
    which holds for any $t\geq 0$ and $k\in[m]$.
\end{proposition}
\begin{proof}
    Multiplying equations in~\eqref{eq..MainDynamics} by $\kappa'a_k$ and $\frac{\vw_k}{\kappa'}$ respectively, we obtain
    \begin{equation}\label{eq..energy}
        \left \{
        \begin{aligned}
            \kappa'\dot{a}_k a_k              & = -\frac{\kappa}{n}\sum_{i=1}^{n}e_i a_k\sigma(\vw_k^\T\vx_i),              \\
            \frac{1}{\kappa'}\dot{\vw}_k\vw_k & = -\frac{\kappa}{n}\sum_{i=1}^n e_i a_k\sigma'(\vw_k^\T\vx_i)\vw_k^\T\vx_i.
        \end{aligned}
        \right.
    \end{equation}
    Notice that for ReLU activation $\sigma(z)=z\sigma'(z)$, $z\in\sR$.
    by comparing the right hand side of~\eqref{eq..energy}, one can obtain
    \begin{equation*}
        \kappa'^2\frac{\D}{\D t}\abs{a_k}^2 = \frac{\D}{\D t}\norm{\vw_k}^2_2.
    \end{equation*}
    Integrating this from $0$ to $t$ leads to
    \begin{equation*}
        \kappa'^2\left(\abs{a_k(t)}^2-\abs{a_k^0}^2\right) = \norm{\vw_k(t)}^2_2-\norm{\vw_k^0}^2_2,
    \end{equation*}
    which then can be written as
    \begin{equation}
        \abs{a_k(t)}^2 = \frac{1}{\kappa'^2}\left(\norm{\vw_k(t)}^2_2-\norm{\vw_k^0}^2_2\right) + \abs{a_k^0}^2.
    \end{equation}
    Finally we have
    \begin{equation*}
        \abs{a_k(t)}\leq\sqrt{\frac{1}{\kappa'^2}\norm{\vw_k(t)}^2_2+\abs{a_k^0}^2}
        \leq
        \frac{1}{\kappa'}\norm{\vw_k(t)}_2+\abs{a_k^0}.
    \end{equation*}
\end{proof}
\begin{proof}[Proof of Theorem~\ref{thm..CondensedRegime}]
    By Assumption \ref{assump..well-trained}, there exits a $T^*>0$ such that
    \begin{equation*}
        \RS(\vtheta(T^*))\leq \frac{1}{32n}.
    \end{equation*}
    Without loss of generality, we assume $f(\vx_1)\geq \frac{1}{2}$.
    Therefore
    \begin{equation*}
        \frac{1}{2n}e_1(T^*)^2\leq \frac{1}{2n}\ve(T^*)^\T\ve(T^*)=\RS(\vtheta(T^*))\leq\frac{1}{32n},
    \end{equation*}
    which means
    \begin{equation*}
        \abs{e_1(T^*)}\leq \frac{1}{4}.
    \end{equation*}
    Recalling the definition that $e_1=\kappa f_{\vtheta}(\vx_1) - f(\vx_1)$, we have
    \begin{equation*}
        \kappa\sum_{k=1}^m a_k(T^*)\sigma(\vw_k(T^*)^\T\vx_1) \geq f(\vx_1) - \frac{1}{4} \geq \frac{1}{4}.
    \end{equation*}
    So
    \begin{equation*}
        \begin{aligned}
            \frac{1}{4\kappa}
             & \leq\sum_{k=1}^m a_k(T^*)\sigma(\vw_k(T^*)^\T\vx_1)                                                                                \\
             & \leq \sqrt{d}\sum_{k=1}^m\abs{a_k(T^*)}\norm{\vw_k(T^*)}_2                                                                         \\
             & \leq  \sqrt{d}\sum_{k=1}^m\left(\frac{1}{\kappa'}\norm{\vw_k(T^*)}_2+\abs{a_k^0}\right)\norm{\vw_k(T^*)}_2                         \\
             & = \sqrt{d}\left(\frac{1}{\kappa'}\sum_{k=1}^m\norm{\vw_k(T^*)}_2^2+\sum_{k=1}^m\abs{a_k^0}\norm{\vw_k(T^*)}_2\right)               \\
             & \leq
            \sqrt{d}\left(\frac{1}{\kappa'}\norm{\vtheta_{\vw}(T^*)}_2^2+\frac{1}{4}\norm{\vtheta_{a}^0}_2^2+\norm{\vtheta_{\vw}(T^*)}_2^2\right) \\
             & \leq
            2\sqrt{d}\max\left\{\frac{1}{\kappa'},1\right\}\norm{\vtheta_{\vw}(T^*)}_2^2+\frac{\sqrt{d}}{4}\norm{\vtheta_{a}^0}_2^2,
        \end{aligned}
    \end{equation*}
    where we have used Proposition~\ref{prop..a-w-est}.
    By Proposition \ref{prop..InitialThetaNorm}, we have with probability at least $1-2\exp(-\frac{C_0m(d+1)}{4C_{\psi,1}^2})$ over the choice of $\vtheta^0$,
    \begin{align*}
        \norm{\vtheta^0_{a}}
         & \leq \sqrt{\frac{3}{2}m},  \\
        \norm{\vtheta^0_{\vw}}
         & \leq \sqrt{\frac{3}{2}dm}.
    \end{align*}
    If $m\kappa\leq \frac{1}{3\sqrt{d}}$, then $\frac{3\sqrt{d}m}{8}\leq \frac{1}{8\kappa}$ and
    \begin{align*}
        \frac{1}{8\kappa}
        \leq \frac{1}{4\kappa}-\frac{3\sqrt{d}m}{8}
         & \leq \frac{1}{4\kappa}-\frac{\sqrt{d}}{4}\norm{\vtheta_{a}^0}_2^2                  \\
         & \leq 2\sqrt{d}\max\left\{\frac{1}{\kappa'},1\right\}\norm{\vtheta_{\vw}(T^*)}_2^2.
    \end{align*}
    Thus
    \begin{equation*}
        \frac{\min\{1,\kappa'\}}{16\sqrt{d}\kappa}\leq \norm{\vtheta_{\vw}(T^*)}_2^2.
    \end{equation*}
    Therefore
    \begin{equation*}
        \begin{aligned}
            \sup\limits_{t\in[0,+\infty)}\frac{\norm{\vtheta_{\vw}(t)-\vtheta_{\vw}^0}_2}{\norm{\vtheta^0_{\vw}}_2}
             & \geq \frac{\norm{\vtheta_{\vw}(T^*)-\vtheta_{\vw}^0}_2}{\norm{\vtheta^0}_2} \\
             & \geq \sqrt{\frac{\frac{\min\{1,\kappa'\}}{16\sqrt{d}\kappa}
            }{\frac{3}{2}dm}} - 1                                                          \\
             & \gtrsim \sqrt{\frac{\min\{1,\kappa'\}}{\kappa m}}.
        \end{aligned}
    \end{equation*}
    If $\gamma'<\gamma-1$, then
    \begin{equation*}
        \frac{\min\{1,\kappa'\}}{\kappa m}\gg 1,
    \end{equation*}
    which completes the proof.
\end{proof}
\begin{remark}
    Suppose that Assumption \ref{assump..lambda} and \ref{assump..gammagamma'} hold. If $\gamma>1$ and $\gamma'<1-\gamma$, then Theorem \ref{thm..CondensedRegime} can hold without taking Assumption \ref{assump..well-trained}. Actually, for any $\delta\in(0,1)$, Proposition \ref{prop..ALazyRegime} guarantees the Assumption \ref{assump..well-trained} with probability at least $1-\delta$ over the choice of $\vtheta^0$, when $m$ is sufficiently large. Therefore, under Assumptions \ref{assump..lambda} and \ref{assump..gammagamma'}, if $m$ is sufficiently large, then we have with probability at least $1-\delta$ over the choice of $\vtheta^0$, the relative change $\sup\limits_{t\in[0,+\infty)}\mathrm{RD}(\vtheta_{\vw}(t))\gg 1$.
\end{remark}

\section{Relative deviation of parameters}\label{sec:relapara}
For completion, we can also similarly define the slope of the relative deviation for $\vtheta$ and $a$ denoted by $S_{\vtheta}$ and $S_{a}$, respectively. As shown in
Fig.~\ref{fig:theta_a_slope} (a), the boundary for the $\vtheta$ is $\gamma=1$, regardless of $\gamma'$, that is, all parameters are close to their initialization after training. For output weight $a$, as shown in Fig.~\ref{fig:theta_a_slope}(b), the boundary consists of two rays, one is  $\gamma=1$ and
$\gamma'\geq 0$, the other is $\gamma+\gamma'=1$ and $\gamma'\leq 0$. This verifies that, in the area between  $\gamma=1$ and $\gamma-\gamma'=1$ of  $\gamma'\leq 0$, the change of scatter plot from a Gaussian initialization is induced by the change of $a$.

\begin{figure}
    \begin{centering}
        \subfloat[]{\begin{centering}
                \includegraphics[scale=0.4]{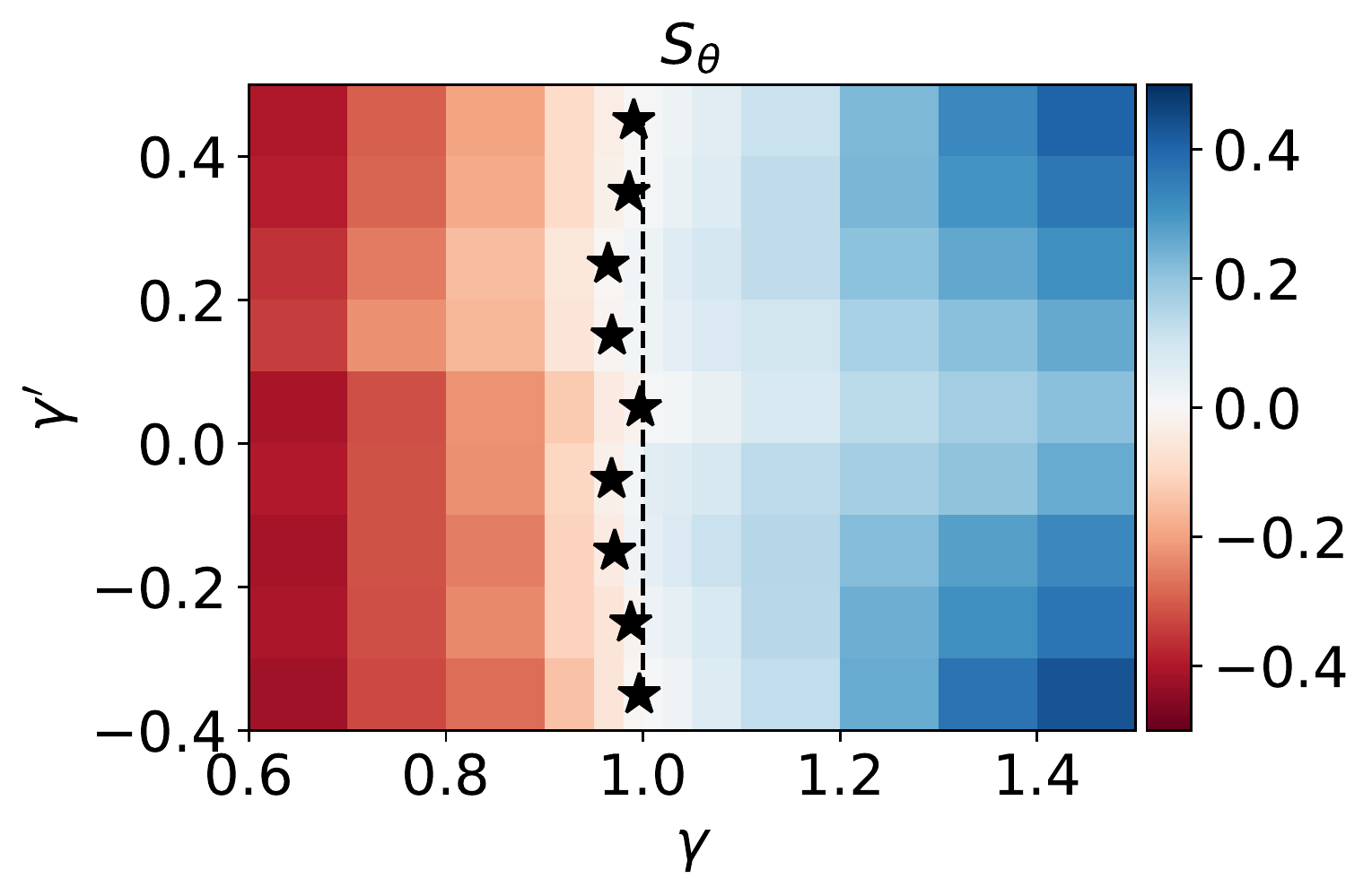}
                \par\end{centering}
        }\subfloat[]{\begin{centering}
                \includegraphics[scale=0.4]{pic/systemexplarg3/scalestudy3/rescale_w_slope.pdf}
                \par\end{centering}
        }
        \par\end{centering}
    \caption{$S_{\vtheta}$ in (a) and $S_{\vw}$ in (b) estimated on NNs of $1000,5000,10000,20000,40000$ neurons over $\gamma$ (ordinate) and $\gamma'$ (abscissa). The stars are zero points obtained by the linear interpolation over different $\gamma$ for each fixed $\gamma'$. Dashed lines are auxiliary lines.   \label{fig:theta_a_slope} }
\end{figure}

\bibliography{dl}
\end{document}